%% file: main.tex
\documentclass[10pt,twocolumn,letterpaper]{article}

\usepackage{cvpr}              

\input{preamble}

\definecolor{cvprblue}{rgb}{0.21,0.49,0.74}
\usepackage[pagebackref,breaklinks,colorlinks,citecolor=cvprblue]{hyperref}

\usepackage[capitalize]{cleveref}

\def\paperID{7926} 
\def\confName{CVPR}
\def\confYear{2024}

\title{Quantifying Task Priority for Multi-Task Optimization}

\author{Wooseong Jeong\\
KAIST\\
{\tt\small stk14570@kaist.ac.kr}
\and
Kuk-Jin Yoon\\
KAIST\\
{\tt\small kjyoon@kaist.ac.kr}\\
}

\begin{document}
\maketitle
\input{sec/0_abstract} 
\input{sec/1_intro}
\input{sec/2_related}
\input{sec/3_prelim}

\input{sec/4_method}
\input{sec/5_exper}
\input{sec/6_exper_result}
\input{sec/7_ablation}
\input{sec/8_conclusion}

\vspace{1pt}
\noindent
\textbf{Acknowledgements} This research was supported by National Research Foundation of Korea (NRF) grant funded by the Korea government (MSIT) (NRF2022R1A2B5B03002636) and the Challengeable Future Defense Technology Research, Development Program through the Agency For Defense Development (ADD) funded by the Defense Acquisition Program Administration (DAPA) in 2024 (No.912768601), and the Technology Innovation Program (1415187329, 20024355, Development of autonomous driving connectivity technology based on sensor-infrastructure cooperation) funded By the Ministry of Trade, Industry Energy(MOTIE, Korea).

{
    \small
    \bibliographystyle{ieeenat_fullname}
    \bibliography{main}
}

\input{sec/X_suppl}


\end{document}

%% file: preamble.tex
\usepackage[dvipsnames]{xcolor}

\usepackage{amsmath,amsfonts,bm}
\usepackage{algorithm}
\usepackage{algorithmic}
\usepackage[linesnumbered,ruled,vlined,algo2e]{algorithm2e}

\SetKwInput{kwOutput}{Output}

\SetCommentSty{mycommfont}

\usepackage{booktabs}
\usepackage{multirow}

\usepackage{tabularx}

\usepackage{color, colortbl}
\definecolor{_gray}{RGB}{200,200,200}

\usepackage{float}

\usepackage{amsthm}
\newtheorem{theorem}{Theorem}
\newtheorem{definition}{Definition}
\usepackage{thm-restate}

\usepackage{indentfirst}
\input{math_commands.tex}

\usepackage{multicol}

%% file: math_commands.tex
\usepackage{amsmath,amsfonts,bm}

\def\eqref#1{equation~\ref{#1}}

\def\1{\bm{1}}

\DeclareMathAlphabet{\mathsfit}{\encodingdefault}{\sfdefault}{m}{sl}
\SetMathAlphabet{\mathsfit}{bold}{\encodingdefault}{\sfdefault}{bx}{n}

\DeclareMathOperator*{\argmax}{arg\,max}
\DeclareMathOperator*{\argmin}{arg\,min}

%% file: sec/0_abstract.tex
\begin{abstract}
The goal of multi-task learning is to learn diverse tasks within a single unified network. As each task has its own unique objective function, conflicts emerge during training, resulting in negative transfer among them. Earlier research identified these conflicting gradients in shared parameters between tasks and attempted to realign them in the same direction. However, we prove that such optimization strategies lead to sub-optimal Pareto solutions due to their inability to accurately determine the individual contributions of each parameter across various tasks. In this paper, we propose the concept of task priority to evaluate parameter contributions across different tasks. To learn task priority, we identify the type of connections related to links between parameters influenced by task-specific losses during backpropagation. The strength of connections is gauged by the magnitude of parameters to determine task priority. Based on these, we present a new method named connection strength-based optimization for multi-task learning which consists of two phases. The first phase learns the task priority within the network, while the second phase modifies the gradients while upholding this priority. This ultimately leads to finding new Pareto optimal solutions for multiple tasks. Through extensive experiments, we show that our approach greatly enhances multi-task performance in comparison to earlier gradient manipulation methods.
\end{abstract}

%% file: sec/1_intro.tex
\section{Introduction}
\label{sec:intro}
Multi-task learning (MTL) is a learning paradigm that handles multiple different tasks in a single model \cite{caruana1997multitask}. Compared to learning tasks individually, MTL can effectively reduce the number of parameters, leading to less memory usage and computation with a higher convergence rate. Furthermore, it leverages multiple tasks as an inductive bias, enabling the learning of generalized features while reducing overfitting. Complex systems such as robot vision and autonomous driving require the ability to perform multiple tasks within a single system. Thus, MTL can be a first step in finding general architecture for computer vision.

A primary goal of MTL is minimizing \emph{negative transfer} \citep{crawshaw2020multi} and finding \emph{Pareto-optimal solutions} \citep{RN36} for multiple tasks. Negative transfer is a phenomenon where the learning of one task adversely affects the performance of other tasks. Since each task has its own objective, this can potentially result in a trade-off among tasks. A condition in which enhancing one task is not possible without detriment to another is called \emph{Pareto optimality}. A commonly understood cause of this trade-off is \emph{conflicting gradients} \citep{RN20} that arise during the optimization process. When the gradients of two tasks move in opposing directions, the task with larger magnitudes dominates the other, disrupting the search for Pareto-optimal solutions. The situation becomes more complex due to imbalances in loss scales across tasks. The way we weigh task losses is crucial for multi-task performance. When there is a significant disparity in the magnitudes of losses, the task with a larger loss would dominate the entire network. Hence, the optimal strategy for MTL should efficiently handle conflicting gradients across different loss scales.

Previous studies address negative transfer by manipulating gradients or balancing tasks' losses. Solutions for handling conflicting gradients are explored in \citep{RN36, RN20, RN18, senushkin2023independent}. These approaches aim to align conflicting gradients towards a cohesive direction within a shared network space. However, these techniques are not effective at preventing negative transfer, as they don't pinpoint which shared parameters are crucial for the tasks. This results in sub-optimal Pareto solutions for MTL, leading to pool multi-task performance. Balancing task losses is a strategy that can be applied independently from gradient manipulation methods. It includes scaling the loss according to homoscedastic uncertainty \citep{RN23}, or dynamically finding loss weights by considering the rate at which the loss decreases \citep{RN26}.

In this paper, we propose the concept of \emph{task priority} to address negative transfer in MTL and suggest \emph{connection strength} as a quantifiable measure for this purpose. The task priority is defined over shared parameters by comparing the influence of each task's gradient on the overall multi-task loss. This reveals the relative importance of shared parameters to various tasks. To learn and conserve the task priority throughout the optimization process, we propose the concept of \emph{task-specific connections} and their \emph{strength} in the context of MTL. A \emph{task-specific connection} denotes the link between shared and task-specific parameters during the backpropagation of each task-specific loss. The strength of this connection can be quantified by measuring the scale of the parameters involved. Based on the types of connections and their respective strengths, we apply two distinct optimization phases. The goal of the first phase is to find new Pareto-optimal solutions for multiple tasks by learning task priorities through the use of specific connection types. The second phase aims to maintain the task priorities learned from varying loss scales by quantifying the strength of these connections. Our method outperforms previous optimization techniques that relied on gradient manipulation, consistently discovering new Pareto optimal solutions for various tasks, thereby improving multi-task performance.

Our contributions are summarized as follows:
\begin{itemize}
    \item We propose the concept of task priority within a shared network to assess the relative importance of parameters across different tasks and to uncover the limitation inherent in traditional multi-task optimization.
    \item We reinterpret connection strength within the context of MTL to quantify task priority. Based on this reinterpretation, we propose a new multi-task optimization approach called connection strength-based optimization to learn and preserve task priorities.
    \item To demonstrate the robustness of our method, we perform extensive experiments.
    Our results consistently reveal substantial enhancements in multi-task performance when compared to prior research.
\end{itemize}

%% file: sec/2_related.tex
\section{Related Work}
\label{sec:related}
\textbf{Optimization for MTL} aims to mitigate negative transfer between tasks. Some of them \citep{RN19, RN36, RN20, RN18, liu2021towards, navon2022multi, senushkin2023independent} directly modify gradients to address task conflicts. MGDA \citep{RN19, RN36} views MTL as a multi-objective problem and minimizes the norm point in the convex hull to find a Pareto optimal set. PCGrad \citep{RN20} introduces the concept of conflicting gradients and employs gradient projection to handle them. CAGrad \citep{RN18} minimizes the multiple loss functions and regularizes the trajectory by leveraging the worst local improvement of individual tasks. Aligned-MTL \citep{senushkin2023independent} stabilize optimization by aligning the principal components of the gradient matrix. Recon \citep{guangyuan2022recon} uses an approach similar to Neural Architecture Search (NAS) to address conflicting gradients. Some approaches use normalized gradients \citep{RN24} to prevent spillover of tasks or assign stochasticity on the network's parameter based on the level of consistency in the sign of gradients \citep{RN21}. RotoGrad \citep{RN22} rotates the feature space of the network to narrow the gap between tasks. Unlike earlier methods that guided gradients towards an intermediate direction (as illustrated in \cref{fig:overview}(a)), our approach identifies task priority in shared parameters to update gradients, leading to finding new Pareto-optimal solutions.

\textbf{Scaling task-specific loss} largely influences multi-task performance since the task with a significant loss would dominate the whole training process and cause severe task interference. To address the task unbalancing problem in the training, some approaches re-weight the multi-task loss by measuring homoscedastic uncertainty \citep{RN23}, prioritizing tasks based on task difficulty \citep{RN25}, or balancing multi-task loss dynamically by considering the descending rate of loss \citep{RN26}. We perform extensive experiments involving different loss-scaling methods to demonstrate the robustness of our approach across various loss-weighting scenarios.

\textbf{MTL architectures} can be classified depending on the extent of network sharing across tasks. The shared trunk consists of a shared encoder followed by an individual decoder for each task \citep{RN51, RN52, RN49, RN50}. Multi-modal distillation methods \citep{RN9, RN29, RN32, pap} have been proposed, which can be used at the end of the shared trunk for distillation to propagate task information effectively. On the other hand, cross-talk architecture uses separate networks for each task and allows parallel information flow between layers \citep{RN43}. Our optimization approach can be applied to any model to mitigate task conflicts and enhance multi-task performance.

\begin{figure*}[t]
\vspace{-8pt}
\begin{center}
\includegraphics[width=0.99\linewidth]{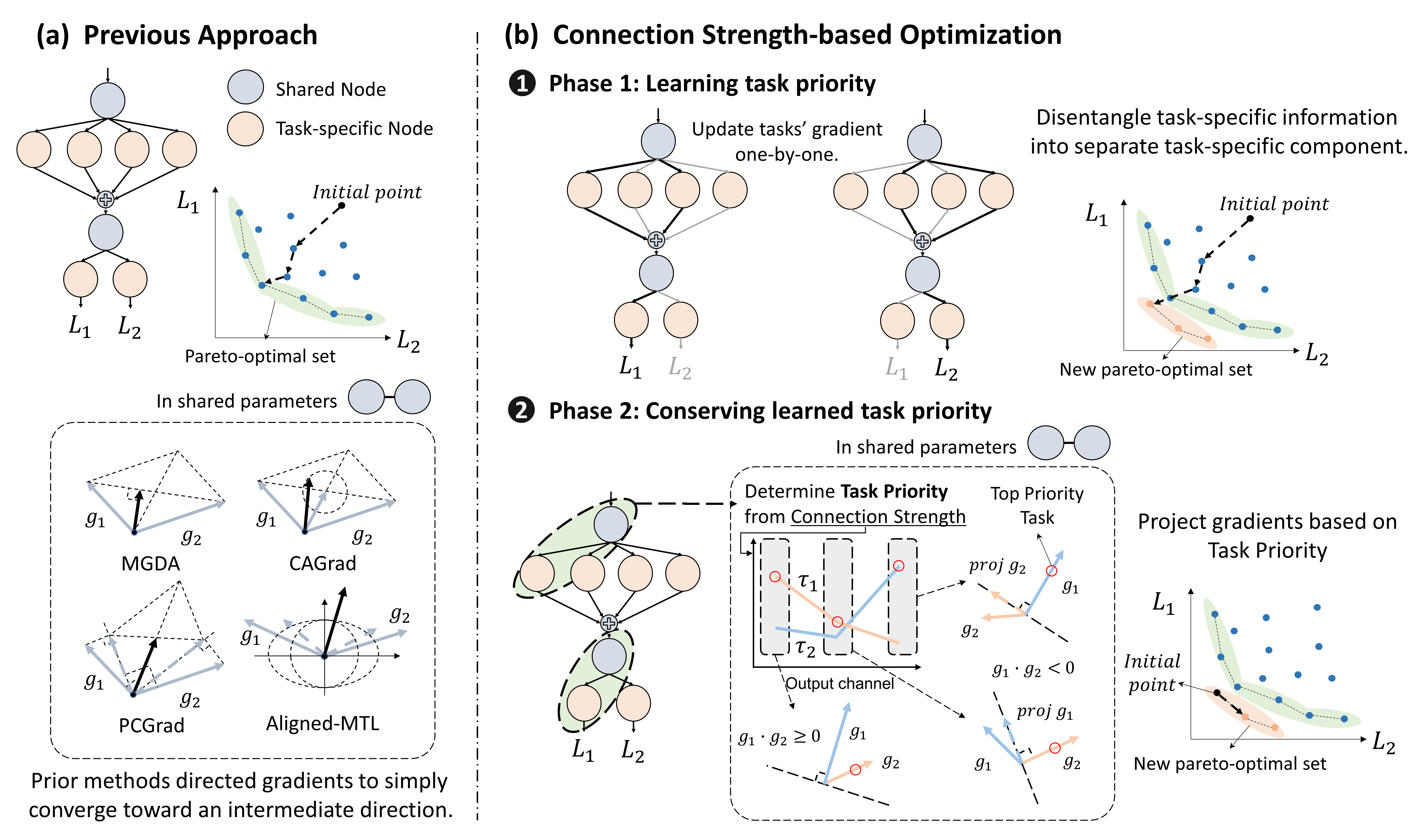}
\end{center}
\vspace{-10pt}
\caption{
    Overview of our connection strength-based optimization. (a) Previous methods \citep{RN36, RN20, RN18, senushkin2023independent} modify gradients in shared parameters to converge toward an intermediate direction without considering the task priority, which leads to sub-optimal Pareto solutions. (b) Our method divides the optimization process into two distinct phases. In Phase 1, task priority is learned through task-specific connections, leading to the identification of a new Pareto optimal solution. In Phase 2, task priority is gauged using the connection strength between shared and task-specific nodes. Subsequently, gradients in shared parameters are aligned with the direction of the highest-priority task's gradients. This phase ensures that priorities established in Phase 1 are maintained, thus reducing potential negative transfer.}
   \label{fig:overview}
\vspace{-5pt}
\end{figure*}

%% file: sec/3_prelim.tex
\section{Preliminaries}
\label{preliminaries}
\subsection{Problem Definition for Multi-task Learning}
\label{preliminaries:problem_definition}
In multi-task learning (MTL), the network learns a set of tasks $\mathcal{T}=\{\tau_1, \tau_2, ..., \tau_{\mathcal{K}}\}$ jointly, where $\mathcal{K}$ is the number of tasks. Each task $\tau_i$ has its own loss function $\mathcal{L}_{i}(\Theta)$ where $\Theta$ is the parameter of the network. The network parameter $\Theta$ can be classified into $\Theta = \{\Theta_{s}, \Theta_{1}, \Theta_{2},...,\Theta_{\mathcal{K}}\}$ where $\Theta_{s}$ is shared parameter across all tasks and $\Theta_i$ is task-specific parameters devoted to task $\tau_i$. Then, the objective function of multi-task learning is to minimize the weighted sum of all tasks' losses:
\begin{equation}
  \Theta^* = \argmin_{\Theta}\sum^{\mathcal{K}}_{i=1} w_i\mathcal{L}_i(\Theta_s, \Theta_i)
\end{equation}
The performance in multi-task scenarios is affected by the weighting $w_i$ of the task-specific loss $\mathcal{L}_i$.

\subsection{Prior Approach for Multi-Task Optimization}
From an optimization perspective, MTL seeks Pareto optimal solutions for multiple tasks.
\begin{definition}[Pareto optimality]
For a given network parameter $\Theta$, if we get $\Theta_{new}$ such that $\mathcal{L}_{i}(\Theta)>\mathcal{L}_{i}(\Theta_{new})$ holds for any task $\tau_i$, while ensuring that $\mathcal{L}_{j}(\Theta)\geq\mathcal{L}_{j}(\Theta_{new})$ is satisfied for all other tasks $\tau_j$ ($j\neq i$), then the situation is termed a Pareto improvement. In this context, $\Theta_{new}$ is said to dominate $\Theta$. A parameter $\Theta^{*}$ is Pareto-optimal if no further Pareto improvements are possible. A set of Pareto optimal solutions is called a Pareto frontier.
\end{definition}
Earlier research \citep{RN36, RN18, senushkin2023independent} interprets MTL in the context of multi-objective optimization, aiming for Pareto optimality. They present a theoretical analysis that demonstrates the convergence of optimization towards Pareto stationary points. Nevertheless, their analysis is constrained when applied to real-world scenarios due to its assumption of convex loss functions, which conflicts with the non-convex nature of neural networks. Also, their demonstration of optimization converging to Pareto stationary points doesn't necessarily guarantee reaching Pareto-optimal points, as the former are necessary but not sufficient conditions for Pareto optimality. We delineate their limitations theoretically by introducing the concept of task priority and empirically validate them by analyzing training loss and multi-task performance. On the other hand, Yu \etal. \citep{RN20} emphasize the conflicting gradients.
\begin{definition}[Conflicting gradients]
Conflicting gradients are defined in the shared space of the network. Denote the gradient of task $\tau_i$ with respect to the shared parameters $\Theta_s$ as $g_i=\nabla_{\Theta_s}\mathcal{L}_i(\Theta_s, \Theta_i)$. And $g_i$ and $g_j$ are gradients of a pair of tasks $\tau_i$ and $\tau_j$ where $i \neq j$. If $g_i \cdot g_j \leq 0$, then the two gradients are called conflicting gradients.
\end{definition}
Previous approaches \citep{RN36,RN20,RN18,senushkin2023independent} address the issue of conflicting gradients in shared parameters $\Theta_s$ by aligning the gradients in a consistent direction as shown in \cref{fig:overview}(a). Nonetheless, they face challenges in minimizing negative transfer, as they cannot discern which parameters in $\Theta_s$ are most important to tasks. We refer to the relative importance of a task in the shared parameter as task priority. Previous studies aligned gradients without taking into account task priority, inadvertently resulting in negative transfer and reduced multi-task performance. In contrast, we introduce the notion of connection strength to determine task priority in the shared space and propose new gradient update rules based on this priority.

%% file: sec/4_method.tex
\section{Method}
\label{method}
In this section, we introduce the concept of task priority to minimize negative transfer between tasks. To measure task priority, we establish connections in the network and assess their strength. Following that, we propose a novel optimization method for MTL termed connection strength-based optimization. Our approach breaks down the optimization process into two phases as shown in \cref{fig:overview}(b). In Phase 1, we focus on instructing the network to catch task-specific details by learning task priority. In Phase 2, task priority within the shared parameters is determined and project gradients to preserve the priority.

\subsection{Motivation: Task priority}
\label{method:motivation}
We propose a straightforward theoretical analysis of our approach, using the notation given in \cref{preliminaries}. Before diving deeper, we first introduce the definition of task priority.
\begin{definition}[Task priority]
\label{def:task_priority}
Assume that the task losses $\mathcal{L}_i$ for $i=1,2,...,\mathcal{K}$ are differentiable. Consider $\mathcal{X}^t$ as the input data at time $t$. We initiate with shared parameters $\Theta_s^{t}$ and task-specific parameters $\Theta_i^{t}$ with sufficiently small learning rate $\eta>0$. A subset of shared parameters at time $t$ is denoted as $\theta^t$, such that $\theta^t \subset \Theta_s^t$. For any task $\tau_i \in \mathcal{T}$, the task's gradient for $\theta^t$ is as follows:
\begin{equation}
\begin{split}
     \mathrm{g}_i = \nabla_{\theta^{t}} \mathcal{L}_{i}(\mathcal{X}^t,\tilde{\Theta}_{s}^t, \theta^t, \Theta_i^t)
    \end{split}
\end{equation}
where $\tilde{\Theta}_{s}^t$ represents the parameters that are part of $\Theta_s^t$ but not in $\theta^t$. For two distinct tasks $\tau_m, \tau_n \in \mathcal{T}$, if $\tau_m$ holds priority over $\tau_n$ in $\theta^t$, then the following inequality holds:
\begin{equation}
\begin{split}
     \sum_{i=1}^{\mathcal{K}}w_i\mathcal{L}_{i}(\tilde{\Theta}_{s}^t, \theta^t-\eta g_m, \Theta_i^t)
     \leq
     \sum_{i=1}^{\mathcal{K}}w_i\mathcal{L}_{i}(\tilde{\Theta}_{s}^t, \theta^t-\eta g_n, \Theta_i^t)
    \end{split}
\end{equation}
\end{definition}
Our motivation is to divide shared parameters $\Theta_s$ into subsets $\{\theta_{s,1}, \theta_{s,2},...,\theta_{s,\mathcal{K}}\}$ based on task priority. Specifically, $\theta_{s,i}$ represents a set of parameters that have a greater influence on task $\tau_i$ compared to other tasks. From the task priority, we can derive the following theorem.

\begin{restatable}[]{theorem}{theom}
\label{theorem1}
Updating gradients based on task priority for shared parameters $\Theta_s$ (update $g_i$ for each $\theta_{s,i}$) results in a smaller multi-task loss $\sum_{i=1}^{\mathcal{K}} w_i \mathcal{L}_i$ compared to updating the weighted summation of task-specific gradients $\sum_{i=1}^{\mathcal{K}} \nabla w_i \mathcal{L}_i$ without considering task priority.
\end{restatable}
The theorem suggests that by identifying the task priority within the shared parameter $\Theta_s$, we can further expand the known Pareto frontier compared to neglecting that priority.  A detailed proof and theoretical analysis are provided in \cref{Append:theoretical_analysis}. However, identifying task priority in real-world scenarios is highly computationally demanding. Because it requires evaluating priorities for each subset of the parameter $\Theta_s$ through pairwise comparisons among multiple tasks. Instead, we prioritize tasks based on connection strength for practical purposes.

\subsection{Type and Strength of Connection}
\label{method:connection}

If we think of each input and output of the network's component as a node, we can depict the computation flow by establishing connections between them, and then evaluate the strength of these connections to measure their interconnectedness. The idea of connection strength initially emerged in the field of network compression by pruning connections in expansive CNNs ~\citep{RN48}. This notion stems from the intuition that larger parameters have a greater influence on the model's output. Numerous studies \citep{comp1, comp2, comp3, comp4, comp5, comp6, comp7} have reinforced this hypothesis. In our study, we re-interpret this intuition for MTL to determine task priority in shared parameters of the network.

Before we dive in, we divide network connections based on the type of task. Conventionally, connection in a network refers to the connectivity between nodes, quantified by the magnitude of parameters. However, we regrouped the network connection based on which task's loss influences on the connection in backpropagation.
\vspace{-1pt}
\begin{definition}[Task-specific connection]
\label{def:connection}
The connection of task $\tau_i$ includes a set of parameters and their interconnections, specifically those involved in the backpropagation process related to the loss function $\mathcal{L}_i$ for task $\tau_i$.
\end{definition}
In the context of MTL, where each task has its own distinct objective function, diverse connections are formed during the backpropagation. Such connections are determined by the specific loss associated with each task, leading us to term them \emph{task-specific connections}. A set of shared and task-specific parameters, $\Theta_s$ and $\Theta_i$, establishes a unique connection.
\vspace{-1pt}
The connection strength can be measured by the scale of parameters, mirroring the conventional notion. In this instance, we employ task-specific batch normalization to determine the task priority of the output channel of the shared convolutional layer. To establish connection strength, we initiate with a convolutional layer where the input is represented as $x \in \textbf{R}^{N_{I}\times H \times W}$ and the weight is denoted by $W \in \textbf{R}^{N_{O}\times N_{I} \times K \times K}$. Here, $N_{I}$ stands for the number of input channels, $N_{O}$ for the number of output channels, and $K$ indicates the kernel size. Suppose we have output channel set $\mathcal{C}^{out}=\{c^{out}_{p}\}_{p=1}^{N_{O}}$ and input channel set $\mathcal{C}^{in}=\{c^{in}_{q}\}_{q=1}^{N_{I}}$. For any given pair of output and input channels $c^{out}_{p} \in \mathcal{C}^{out}$, $c^{in}_{q} \in \mathcal{C}^{in}$, the connection strength $s_{p,q}$ is defined as:
\begin{equation}
  s_{p, q} = \frac{1}{K^2}\sum_{m=0}^{K-1} \sum_{n=0}^{K-1} {W(c^{out}_p,c^{in}_q,m,n)}^2
  \label{eq:eq1}
\end{equation}
The variables $m$ and $n$ correspond to the indices of the convolutional kernel.
We explore the convolutional layer followed by task-specific batch normalization, which plays a key role in determining task priority for each output channel.
We revisit the equation for batch normalization with input $y$ and output $z$ of batch normalization \citep{ioffe2015batch}:
\begin{equation}
    z=\frac{\gamma}{\sqrt{Var[y]+\epsilon}}\cdot y+(\beta-\frac{\gamma E[y]}{\sqrt{Var[y]+\epsilon}})
    \label{eq:eq2}
\end{equation}
The coefficient of $y$ has a direct correlation with the kernel's relevance to the task since it directly modulates the output $y$. Therefore, for task $\tau_i$, we re-conceptualize the connection strength at the intersection of the convolutional layer and task-specific batch normalization in the following way:
\begin{equation}
\begin{split}
S_{p}^{\tau_{i}} = & \frac{\gamma_{\tau_{i},p}^2}{Var[y]_p+\epsilon}\cdot \sum_{q=1}^{N_I} s_{p, q}
    \label{eq:eq3}
    \end{split}
\end{equation}
where $\gamma_{\tau_{i},p}$ is a scale factor of the task-specific batch normalization.
$S_p^{\tau_i}$ measures the contribution of each output channel $c^{out}_p$ to the output of task $\tau_i$. 
However, it is not possible to directly compare $S_p^{\tau_i}$ across tasks because the tasks exhibit different output scales.
Hence, we employ a normalized version of connection strength that takes into account the relative scale differences among tasks:
\begin{equation}
\begin{split}
    \hat{S}_{p}^{\tau_i} = \frac{S_{p}^{\tau_i}}{\sum_{p=1}^{N_O} S_{p}^{\tau_i}}
\end{split}
\label{eq:eq5}
\end{equation}
Comparing \cref{eq:eq5} for each task allows us to determine task priority. Since normalized connection strength represents the relative contribution of each channel across the entire layer, using it to determine task priority also has the advantage of preventing a specific task from having priority over the entire layer. Connection strength depends on network parameters, necessitating design considerations based on the network structure. While this paper provides an example for convolutional layers, a similar application can be extended to transformer blocks or linear layers. In the following optimization, we employ task-specific connections and their strength to learn task priority and conserve it.

\subsection{Phase 1: Learning the task priority}
Our first approach is very simple and intuitive. Here, the notation follows \cref{preliminaries:problem_definition} and \cref{method:motivation}. For simplicity, we assume all tasks' losses are equally weighted $w_1=w_2=...=w_{\mathcal{K}}=1/{\mathcal{K}}$. According to conventional gradient descent (GD), we have 
\begin{align}
\begin{cases}
    \Theta_{s}^{t+1} = \Theta_{s}^{t}-\eta\sum_{i=1}^{\mathcal{K}}w_i\nabla_{\Theta_{s}^{t}}\mathcal{L}_{i}(\mathcal{X}^t,\Theta_s^t,\Theta_i^t)
    \\
    \Theta_{i}^{t+1} = \Theta_{i}^{t}-\eta \nabla_{\Theta_{i}^{t}} \mathcal{L}_{i}(\mathcal{X}^t,\Theta_{s}^t,\Theta_i^t)
\end{cases}
\label{eq:conventional_gd}
\end{align}
for $i=1,...,\mathcal{K}$. In standard GD, the network struggles to prioritize tasks since all tasks' gradients are updated simultaneously at each step. Instead, we sequentially update each task's gradients, as outlined below:
\begin{align}
\begin{cases}
    \Theta_{s}^{t+i/{\mathcal{K}}} = \Theta_{s}^{t+\frac{(i-1)}{\mathcal{K}}}-\eta \nabla_{\Theta_{s}^{t+\frac{(i-1)}{\mathcal{K}}}} \mathcal{L}_{i}(\mathcal{X}^t,\Theta_s^{t+\frac{(i-1)}{\mathcal{K}}},\Theta_i^{t})\\
    \Theta_{i}^{t+1} = \Theta_{i}^{t}-\eta \nabla_{\Theta_{i}^{t}} \mathcal{L}_{i}(\mathcal{X}^t,\Theta_{s}^{t+\frac{(i-1)}{\mathcal{K}}},\Theta_i^{t})
\end{cases}
\end{align}
for $i=1,...,\mathcal{K}$. The intuition behind this optimization is to let the network divide shared parameters $\Theta_{s}$ into $\{\theta_{s,1}, \theta_{s,2},...,\theta_{s,\mathcal{K}}\}$ based on task priority by updating each task-specific connection sequentially. After the initial gradient descent step modifies both $\Theta_s$ and $\Theta_1$, $\theta_{s,1}$ start to better align with $\tau_1$. In the second step, the network can determine whether $\theta_{s,1}$ would be beneficial for $\tau_2$. Throughout this process, task priorities are learned by updating the task's loss in turn. Recognizing task priority effectively enables the tasks to parse out task-specific information.

\DontPrintSemicolon
\begin{algorithm}[t]
\begin{algorithmic}
\caption{Connection Strength-based Optimization for Multi-task Learning}\label{alg:alg1}
\REQUIRE output channel set $\{c^{out}_{p}\}_{p=1}^{N_O}$, task set $\{\tau_i\}^\mathcal{K}_{i=1}$, loss function set $\{\mathcal{L}_i\}^\mathcal{K}_{i=1}$, channel group $\{CG_i\}^\mathcal{K}_{i=1}$, number of epochs $E$, current epoch $e$ \newline

Randomly choose $\mathcal{P} \sim U(0,1)$ \\
\tcp{Phase 1: Learning the task priority}
\eIf{$\mathcal{P} \geq e/E$ }{
    \For{$i$ $\gets$ $1$ $to$ $\mathcal{K}$}{
    \textbf{update:} $g_{i} \gets \nabla_{\theta} L_i$
    \tcp*{Update task's gradients one-by-one}}
    
}
{\tcp{Phase 2: Conserving the task priority}
    \textbf{Initialize} all $CG_i$ as empty set $\{$ $\}$ in the shared convolutional layer \\
    \For{$p$ $\gets$ 1 to $N_O$}{
        $\nu = \argmax_i \hat{S}_p^{\tau_i}$\\
        \tcp*{Determine the top priority task $\nu$}
        $CG_{\nu} = CG_{\nu} + \{c^{out}_p\}$\\
        \tcp*{Classify channel with task $\nu$}
    }

    \For{$i$ $\gets$ $1$ $to$ $\mathcal{K}$}{
        Let $\{G_{i,1}, ... , G_{i,\mathcal{K}}\}$ are gradients of $CG_{i}$ \\
        \For{$j$ $\gets$ $1$ $to$ $\mathcal{K}$ and $i \neq j$}{
            \If{$G_{i,i}\cdot G_{i,j} < 0$}{
                $G_{i,j} = G_{i,j}$ - $\frac{G_{i,i}\cdot G_{i,j}}{||G_{i,i}||^2}$$\cdot G_{i,i}$\\
                \tcp*{Project gradients with priorities}
            }
        }
    }
    \textbf{update:} $g_{final}$ = $\sum_{i=1}^{\mathcal{K}} g_{i}$\\
    \tcp*{Update modified gradients}
}

\end{algorithmic}
\end{algorithm}

\begin{table*}[h]
    \caption{The experimental results of different multi-task learning optimization methods on NYUD-v2 with HRNet-18. The weights of tasks are manually tuned. Experiments are repeated over 3 random seeds and average values are presented.}
    \vspace{-5pt}
    \centering
    \footnotesize
    \renewcommand\arraystretch{0.80}
    \begin{tabular}{ll@{\enspace}ccl@{\enspace}cccl@{\enspace}cccccl@{\enspace}c}
        \toprule
        \multicolumn{1}{c}{Tasks} && \multicolumn{2}{c}{Depth} && \multicolumn{3}{c}{SemSeg} && \multicolumn{5}{c}{Surface Normal} && \\
        \cmidrule(r){1-1} \cmidrule(r){3-4} \cmidrule(r){6-8} \cmidrule(r){10-14}
        
        \multicolumn{1}{c}{\multirow{2}{*}{Method}}            && \multicolumn{2}{c}{\begin{tabular}[c]{@{}c@{}}Distance\\ (Lower Better)\end{tabular}} && \multicolumn{3}{c}{\begin{tabular}[c]{@{}c@{}}(\%)\\ (Higher Better)\end{tabular}}        && \multicolumn{2}{c}{\begin{tabular}[c]{@{}c@{}}Angle Distance\\ (Lower Better)\end{tabular}} & \multicolumn{3}{c}{\begin{tabular}[c]{@{}c@{}}Within t degree (\%)\\ (Higher Better)\end{tabular}} && MTP \\
        
        && rmse & abs\_rel && mIoU & PAcc & mAcc && mean & median & 11.25 & 22.5 & 30 && $\triangle_m$ $\uparrow$(\%) \\ \toprule

        Independent && 0.667 & 0.186 && 33.18 & 65.04 & 45.07 && 20.75 & 14.04 & 41.32 & 68.26 & 78.04 && + 0.00 \\ \midrule

        GD && 0.594 & 0.150 && 38.67 & 69.16 & 51.12 && 20.52 & 13.46 & 42.63 & 69.00 & 78.42 && + 9.53\\

        MGDA \citep{RN36} && 0.603 & 0.159 && 38.89 & 69.39 & 51.53 && 20.58 & 13.56 & 42.28 & 68.79 & 78.33 && + 9.21\\ 
        
        PCGrad \citep{RN20} && 0.596 & 0.149 && 38.61 & 69.30 & 51.51 && 20.50 & 13.54 & 42.56 & 69.14 & 78.55 && + 9.40 \\
        
        CAGrad \citep{RN18} && 0.595 & 0.153 && 38.80 & 68.95 & 50.78 && 20.38 & 13.53 & 42.89 & 69.33 & 78.71 && + 9.84 \\ 

        Aligned-MTL \citep{senushkin2023independent} && 0.592 & 0.150 && 39.02 & 68.98 & 51.83 && 20.40 & 13.57 & 42.83 & 69.26 & 78.69 && + 10.17 \\

        \rowcolor{_gray}
        Ours && \textbf{0.565} & \textbf{0.148} && \textbf{41.10} & \textbf{70.37} & \textbf{53.74} &&  \textbf{19.54} & \textbf{12.45} & \textbf{46.11} & \textbf{71.54} & \textbf{80.12} && \textbf{+ 15.00} \\ \bottomrule
    \end{tabular}
    \label{tab:nyud_hrnet_tuned}
    \vspace{-10pt}
\end{table*}

\subsection{Phase 2: Conserving the task priority}
Due to negative transfer between tasks, task losses fluctuate during training, resulting in variations in multi-task performance.
Therefore, we introduce a secondary optimization phase to update gradients preserving task priority. For this phase, we employ the connection strength defined in \cref{eq:eq5}. Because of its normalization, individual tasks cannot be highly prioritized across the entire network. The top priority task $\tau_\nu$ for the channel $c^{out}_p$ is determined by evaluating the connection strength as follows:
\begin{equation}
    \nu = \argmax_i \hat{S}_p^{\tau_i}
    \label{eq:eq6}
\end{equation}
After determining the priority of tasks in each output channel, the gradient vector of each task is aligned with the gradient of the top priority task. In detail, we categorize output channel $\{c^{out}_p\}^{N_O}_{p=1}$ into channel groups $\{CG_i\}^\mathcal{K}_{i=1}$ based on their top priority task. The parameter of each channel group $CG_i$ corresponds to $\theta_{s,i}$ in $\Theta_{s} = \{\theta_{s,1}, \theta_{s,2},...,\theta_{s,\mathcal{K}}\}$. Let $\{G_{i,1}, G_{i,2}, ... , G_{i,\mathcal{K}}\}$ are task-specific gradients of $CG_{i}$. Then $G_{i,i}$ acts as the reference vector for identifying conflicting gradients. When another gradient vector $G_{i,j}$, where $i\neq j$, clashes with $G_{i,i}$, we adjust $G_{i,j}$ to lie on the perpendicular plane of the reference vector $G_{i,i}$ to minimize negative transfer. After projecting gradients based on task priority, the sum of them is finally updated.

In the final step, we blend two optimization stages by picking a number $\mathcal{P}$ from a uniform distribution spanning from 0 to 1. We define $E$ as the total number of epochs and $e$ as the current epoch. The choice of optimization for that epoch hinges on whether $\mathcal{P}$ exceeds $e/E$. As we approach the end of the training, the probability of selecting Phase 2 increases. This is to preserve the task priority learned in Phase 1 while updating the gradient in Phase 2. A detailed view of the optimization process is provided in \cref{alg:alg1}. The reason for mixing two phases instead of completely separating them is that the speed of learning task priority varies depending on the position within the network.

Previous studies \citep{RN36, RN18, RN20, senushkin2023independent} deal with conflicting gradients by adjusting them to align in the same direction. These studies attempt to find an intermediate point among gradient vectors, which often leads to negative transfer due to the influence of the dominant task. In comparison, our approach facilitates the network's understanding of which shared parameter holds greater significance for a given task, thereby minimizing negative transfer more efficiently. The key distinction between earlier methods and ours is the inclusion of task priority.

%% file: sec/5_exper.tex
\begin{table*}[t]
    \vspace{-15pt}
    \caption{The experimental results of different multi-task learning optimization methods on PASCAL-Context with HRNet-18. The weights of tasks are manually tuned. Experiments are repeated over 3 random seeds and average values are presented.}
    \vspace{-5pt}
    \centering
    \footnotesize
    \renewcommand\arraystretch{0.80}
    \begin{tabular}{l@{\hspace{6pt}}l@{\hspace{6pt}}cc@{\hspace{2pt}}l@{\hspace{2pt}}c@{\hspace{4pt}}l@{\hspace{2pt}}cc@{\hspace{2pt}}l@{\hspace{4pt}}ccccc@{\hspace{6pt}}l@{\hspace{4pt}}c}
        \toprule
        \multicolumn{1}{c}{Tasks} && \multicolumn{2}{c}{SemSeg} && \multicolumn{1}{c}{PartSeg} && \multicolumn{2}{c}{Saliency} && \multicolumn{5}{c}{Surface Normal} \\
        \cmidrule(l){1-1} \cmidrule(r){3-4} \cmidrule(r){6-6} \cmidrule(r){8-9} \cmidrule(r){11-15}
        \multicolumn{1}{c}{\multirow{2}{*}{Method}}            && \multicolumn{2}{c}{\begin{tabular}[c]{@{}c@{}}(Higher Better)\end{tabular}} && \multicolumn{1}{c}{\begin{tabular}[c]{@{}c@{}}(Higher Better)\end{tabular}}        && \multicolumn{2}{c}{\begin{tabular}[c]{@{}c@{}} (Higher Better)\end{tabular}} && \multicolumn{2}{c}{\begin{tabular}[c]{@{}c@{}}Angle Distance\\ (Lower Better)\end{tabular}} & \multicolumn{3}{c}{\begin{tabular}[c]{@{}c@{}}Within t degree (\%)\\ (Higher Better)\end{tabular}} && MTP \\
        
        && mIoU & PAcc && mIoU && mIoU & maxF && mean & median & 11.25 & 22.5 & 30 && $\triangle_m$ $\uparrow$(\%) \\ \toprule

        Independent && 60.30 & 89.88 && 60.56 && 67.05 & 78.98 && 14.76 & 11.92 & 47.61 & 81.02 & 90.65 && + 0.00 \\ \midrule

        GD && 62.17 & 90.27 && 61.15 && 67.99 & 79.60 && 14.70 & 11.81 & 47.55 & 80.97 & 90.56 && + 1.47\\ 

        MGDA \citep{RN36} && 61.75 & 89.98 && 61.69 && 67.32 & 78.98 && 14.77 & 12.22 & 47.02 & 80.91 & 90.14 && + 1.15 \\ 
        
        PCGrad \citep{RN20} && 62.47 & 90.57 && 61.46 && 67.86 & 79.38 && 14.59 & 11.77 & 47.72 & 81.28 & 90.81 && + 1.86  \\ 
        
        CAGrad \citep{RN18} && 62.22 & 90.01 && 61.89 && 67.46 & 79.12 && 14.97 & 12.10 & 47.23 & 80.54 & 90.30 && + 1.14 \\ 

        Aligned-MTL \citep{senushkin2023independent} && 62.43 & 90.51 && 62.05 && 67.94 & 79.57 && 14.76 & 11.86 & 47.44 & 80.78 & 90.46 && + 1.83 \\

        \rowcolor{_gray}
        Ours && \textbf{63.86} & \textbf{90.65} && \textbf{63.05} && \textbf{68.30} & \textbf{79.26} &&  \textbf{14.33} & \textbf{11.45} & \textbf{49.08} & \textbf{81.86} & \textbf{91.05} && \textbf{+ 3.70} \\ \bottomrule
    \end{tabular}
    \label{tab:pascal_hrnet_tuned}
\vspace{-10pt}
\end{table*}


\begin{table*}[t]
\begin{minipage}{0.55\textwidth}
\centering
\scriptsize
\caption{The comparison of multi-task performance on Cityscapes. Ours demonstrate competitive results without any significant addition to the network's parameters.}
\vspace{-5pt}
\begin{tabular}{lcccccc}
    \toprule
     & \multicolumn{2}{c}{Segmentation} & \multicolumn{2}{c}{Depth} & & \\
    Method       & \multicolumn{2}{c}{(Higher Better)} & \multicolumn{2}{c}{(Lower Better)} & $\triangle_m$ $\uparrow$(\%) & \#P. \\
    \cmidrule(lr){2-3} \cmidrule(lr){4-5}
                 & mIoU   & Pix Acc  & Abs Err  & Rel Err  &         &        \\ \toprule
    Single-task  & 74.36  & 93.22    & 0.0128   & 29.98    &         & 190.59 \\
    Cross-Stitch \citep{RN31} & 74.05  & 93.17    & 0.0162   & 116.66   & - 79.04  & 190.59 \\
    RotoGrad \citep{RN22}    & 73.38  & 92.97    & 0.0147   & 82.31    & - 47.81  & 103.43 \\ \midrule
    GD           & 74.13  & 93.13    & 0.0166   & 116.00   & - 79.32  & 95.43  \\
    w/ Recon \citep{guangyuan2022recon}     & 71.17  & 93.21    & 0.0136   & 43.18    & - 12.63  & 108.44 \\ \midrule
    MGDA \citep{RN36}         & 70.74  & 92.19    & 0.0130   & 47.09    & - 16.22  & 95.43  \\
    w/ Recon \citep{guangyuan2022recon}     & 71.01  & 92.17    & 0.0129   & \textbf{33.41}    & \textbf{- 4.46}   & 108.44 \\ \midrule
    Graddrop \citep{RN21}    & 74.08  & 93.08    & 0.0173   & 115.79   & - 80.48  & 95.43  \\
    w/ Recon \citep{guangyuan2022recon}    & 74.17  & 93.11    & 0.0134   & 41.37    & - 10.69  & 108.44 \\ \midrule
    PCGrad \citep{RN20}       & 73.98  & 93.08    & 0.02     & 114.50   & - 78.39  & 95.43  \\
    w/ Recon \citep{guangyuan2022recon}     & 74.18  & 93.14    & 0.0136   & 46.02    & - 14.92  & 108.44 \\ \midrule
    CAGrad \citep{RN18}       & 73.81  & 93.02    & 0.0153   & 88.29    & - 53.81  & 95.43  \\
    w/ Recon \citep{guangyuan2022recon}    & 74.22  & 93.10    & 0.0130   & 38.27    & - 7.38   & 108.44 \\ \midrule \rowcolor{_gray}
    Ours         & \textbf{74.75}  & \textbf{93.39}    & \textbf{0.0125}   & 41.60    & - 10.08  & 95.48  \\ 
    \bottomrule
\end{tabular}
\label{tab:cityscape_segnet}
\end{minipage}
\begin{minipage}{.40\textwidth}

\centering
\subcaptionbox{NYUD-v2}{
\includegraphics[width=0.99\linewidth]{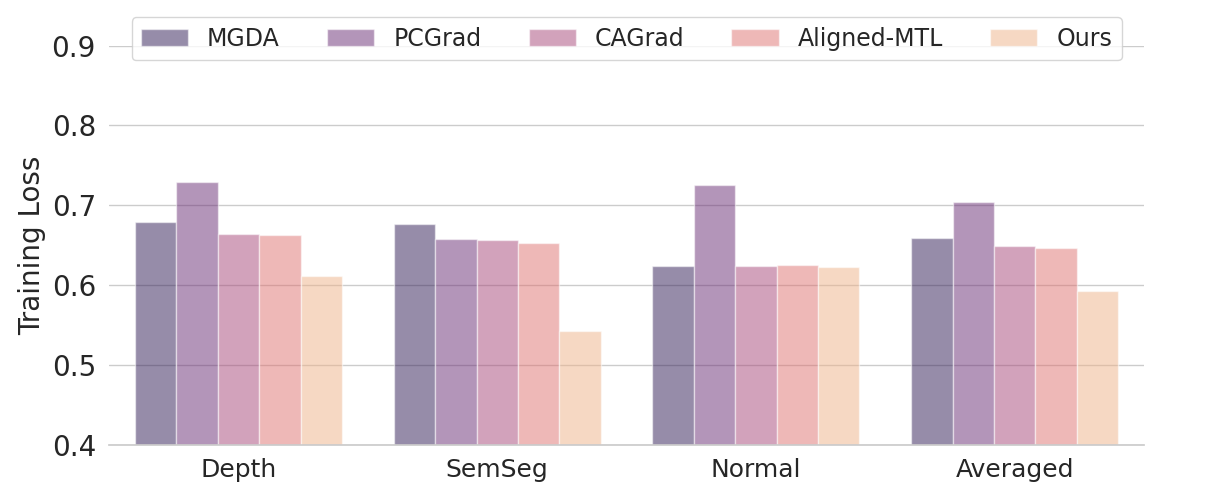}}

\subcaptionbox{PASCAL-Context}{
\includegraphics[width=0.99\linewidth]{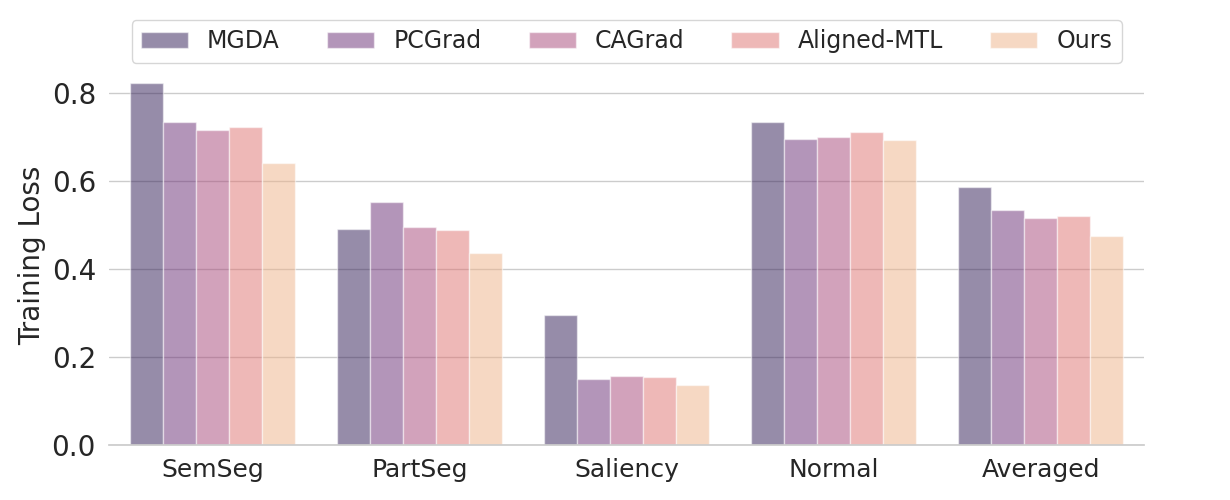}}
\vspace{-5pt}
\captionof{figure}{The comparison of training losses on the NYUDv2 and PASCAL-Context. Ours find a new Pareto optimal solution for multiple tasks. \label{fig:loss}}
\end{minipage}
\vspace{-10pt}
\end{table*}

\section{Experiments}
\subsection{Experimental Setup}
\noindent
\textbf{Datasets.} Our method is evaluated on three multi-task datasets: NYUD-v2 \citep{RN15}, PASCAL-Context \citep{mottaghi2014role}, and Cityscapes \citep{cordts2016cityscapes}. These datasets contain different kinds of vision tasks. NYUD-v2 contains 4 vision tasks: Our evaluation is based on depth estimation, semantic segmentation, and surface normal prediction, with edge detection as an auxiliary task. PASCAL-Context contains 5 tasks: We evaluate semantic segmentation, human parts estimation, saliency estimation, and surface normal prediction, with edge detection as an auxiliary task. Cityscapes contains 2 tasks: We use semantic segmentation and depth estimation.

\noindent
\textbf{Baselines.} We conduct extensive experiments with the following baselines: 1) single-task learning: training each task separately; 2) GD: simply updating all tasks' gradients jointly without any manipulation; 3) multi-task optimization methods with gradient manipulation: MGDA \citep{RN36}, PCGrad \citep{RN20}, CAGrad \citep{RN18}, Aligned-MTL \citep{senushkin2023independent}; 3) loss scaling methods: We consider 4 types of loss weighting where two of them are fixed during training and the other two use dynamically varying weights. Static setting includes equal loss: all tasks are weighted equally; manually tuned loss: all tasks are weighted manually following works in \citep{RN29, RN32}. Dynamic setting includes uncertainty-based approach \citep{RN23}: tasks' weights are determined dynamically based on homoscedastic uncertainty; DWA \citep{RN26}: tasks' losses are determined considering the descending rate of loss to determine tasks' weight dynamically. 4) Architecture design methods including NAS-like approaches: Cross-Stitch \citep{RN31} architecture based on SegNet \citep{badrinarayanan2017segnet}; Recon \citep{guangyuan2022recon}: turn shared layers into task-specific layers when conflicting gradients are detected. All experiments are conducted 3 times with different random seeds for a fair comparison.

\noindent
\textbf{Evaluation Metrics.} To evaluate the multi-task performance (MTP), we utilized the metric proposed in \citep{RN2}. It measures the per-task performance by averaging it with respect to the single-task baseline b, as shown in $\triangle_m = (1/T)\sum_{i=1}^{T}(-1)^{l_i}(M_{m,i}-M_{b,i})/M_{b,i}$ where $l_i=1$ if a lower value of measure $M_i$ means better performance for task $i$, and 0 otherwise. We measured the single-task performance of each task $i$ with the same backbone as baseline $b$. To evaluate the performance of tasks, we employed widely used metrics. More details are provided in \cref{Append:exp_details}.

%% file: sec/6_exper_result.tex
\begin{table*}[t]
\vspace{-15pt}
\begin{minipage}{0.50\textwidth}
\centering
\scriptsize
\caption{Comparison of multi-task performance using each phase individually, sequentially, and by the proposed mixing method on NYUD-v2.}
\vspace{-5pt}
\begin{tabular}{cc|ccc|cc}
    \toprule
    \multicolumn{2}{c|}{Phase} & Depth  & Seg & Norm  & MTP & Averaged     \\ \midrule
    \multicolumn{1}{c}{1} & 2 & rmse & mIoU & mean & $\triangle_m$ $\uparrow$ & Loss \\ \midrule
    \multicolumn{1}{c}{\checkmark} & &0.581 &40.36 & 19.55 &+ 13.44      &\textbf{0.5396}  \\
    \multicolumn{1}{c}{}  & \checkmark &0.597 &39.23 &20.39 &+ 10.32 & 0.6519  \\
    \multicolumn{1}{c}{$\checkmark_{seq}$}  & $\checkmark_{seq}$ &0.574 &40.38 &19.56 &+ 13.79 & 0.5788  \\
    \multicolumn{1}{c}{$\checkmark_{mix}$} & $\checkmark_{mix}$ &\textbf{0.565} &\textbf{41.10} &\textbf{19.54} &\textbf{+ 15.50} & 0.5942 \\
    \bottomrule
\end{tabular}
\label{tab:ablation}
\begin{subfigure}{\columnwidth}
\includegraphics[width=0.99\columnwidth]{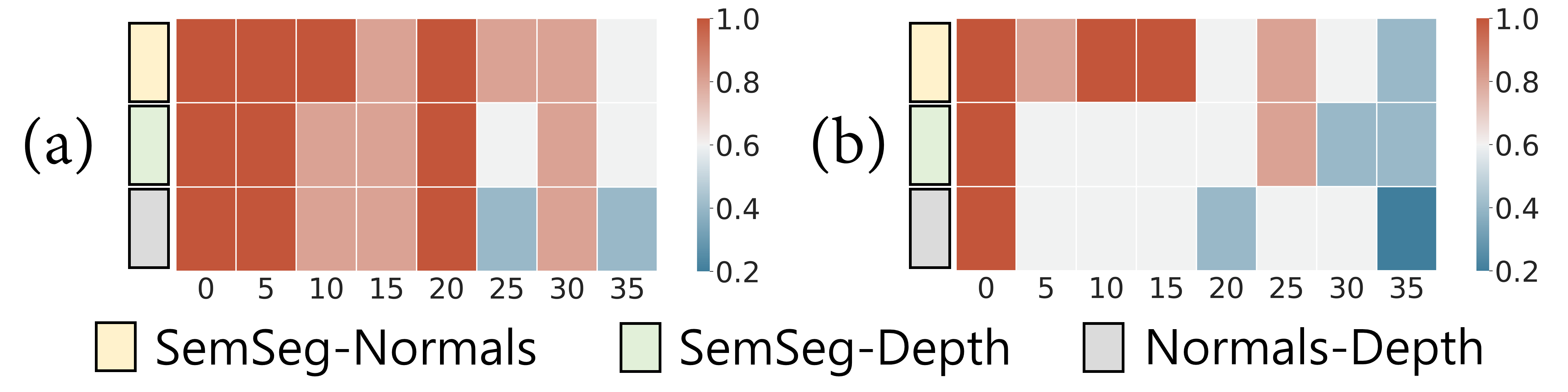}
\end{subfigure}
\captionof{figure}{Correlation of loss trends across tasks during the epochs.\\
a) Phase 1, b) Phase 2. \label{fig:corres_mat}}
\end{minipage}
\hfill
\vspace{-8pt}
\begin{minipage}{.48\textwidth}
\centering
{
\includegraphics[width=0.98\columnwidth]{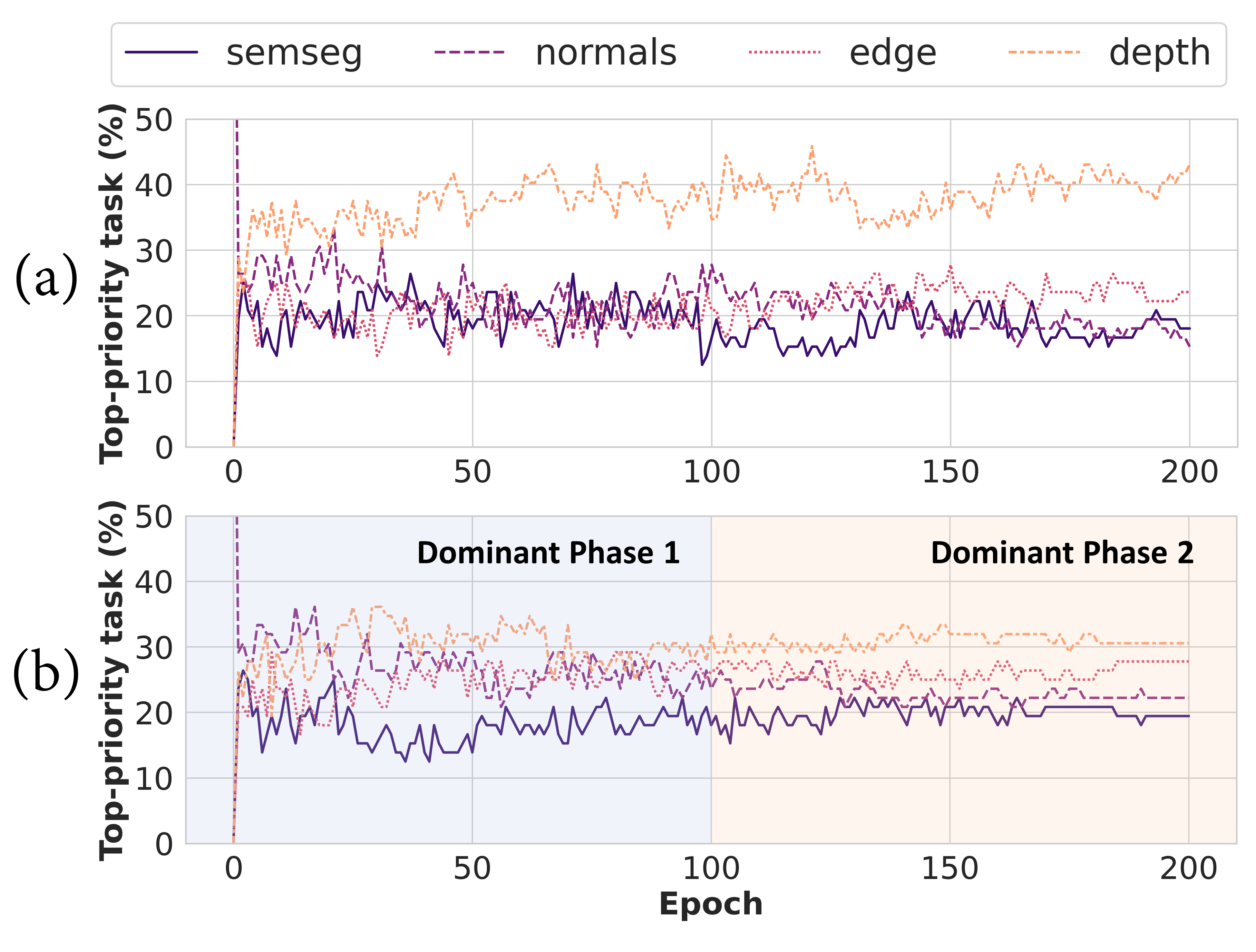}}
\vspace{-12pt}
\captionof{figure}{Visualization of the percentage of top-priority tasks over training epoch. a) Phase 1, b) Mixing Phase 1 and Phase 2 \label{fig:connection_strength}}
\end{minipage}
\vspace{-7pt}
\end{table*}


\subsection{Experimental Results}
\noindent
\textbf{Our method achieves the largest improvements in multi-task performance.} The main results on NYUD-v2, PASCAL-Context are presented in \cref{tab:nyud_hrnet_tuned} and \cref{tab:pascal_hrnet_tuned} respectively. For a fair comparison, we compare various optimization methods on exactly the same architecture with identical task-specific layers. Tasks' losses are tuned manually following the setting in \citep{RN29, RN32}. Compared to previous methods, our approach shows better performance on most tasks and datasets. It proves our method tends to induce less task interference.

\noindent
\textbf{Proposed optimization works robustly on various loss scaling methods.} To prove the generality of our method, we conduct extensive experiments on NYUD-v2 as shown in \cref{tab:nyud_hrnet_equal,tab:nyud_hrnet_tuned,tab:nyud_hrnet_homosce,tab:nyud_hrnet_dwa} (\cref{Append:nyud_hrnet18}) and PASCAL-Context as shown in \cref{tab:pascal_hrnet_equal,tab:pascal_hrnet_tuned,tab:pascal_hrnet_homosce,tab:pascal_hrnet_dwa} (\cref{Append:pascal_hrnet18}). In almost all types of loss scaling, our method shows the best multi-task performance. Unlike conventional approaches where the effectiveness of optimization varies depending on the loss scaling method, ours can be applied to various types of loss weighting and shows robust results.

\noindent
\textbf{Our method can be applied to various types of network architecture.} We use MTI-Net \citep{RN32} with HRNet-18 \citep{RN37} and ResNet-18 \citep{he2016deep} on NYUD-v2 and PASCAL-Context. HRNet-18 and ResNet-18 are pre-trained on ImageNet \citep{imagenet}. On the other hand, we use SegNet \citep{badrinarayanan2017segnet} for Cityscapes from scratch following the experiments setting in \citep{RN18, guangyuan2022recon}. Our optimization shows robustly better performance with different neural network architectures. The results with ResNet-18 are also experimented with various loss scaling as shown in \cref{tab:nyud_resnet_equal,tab:nyud_resnet_tuned,tab:nyud_resnet_homosce,tab:nyud_resnet_dwa} (\cref{Append:nyud_resnet18}).

\noindent
\textbf{Results are compatible with various architectures with fewer parameters.}
In \cref{tab:cityscape_segnet}, we evaluate our methods in different aspects by considering the various types of architecture. In the table, we include the results of Recon \citep{guangyuan2022recon} to show our method can mitigate negative transfer between tasks more parameter efficiently. Compared to Cross-Stitch \citep{RN31} and RotoGrad \citep{RN22}, ours show better multi-task performance with fewer parameters. Compared to Recon, our method is more parameter efficient as it increases the number of parameters by about 0.05\% with the use of task-specific batch normalization. Our method shows comparable performance on Cityscapes with fewer parameters.

\noindent
\textbf{Our method finds new Pareto optimal solutions for multiple tasks.} The final task-specific loss and their average are shown in \cref{fig:loss} for NYUD-v2 and PASCAL-Context. We compare our method with previous gradient manipulation techniques and repeat the experiments over 3 random seeds. For both NYUD-v2 and PASCAL-Context, ours show the lowest average training loss. When comparing each task individually, ours still shows the lowest final loss on every task. This provides proof that our method leads to the expansion of the Pareto frontier of previous approaches.

%% file: sec/7_ablation.tex
\subsection{Ablation Study}

\noindent
\textbf{Phase 1 learns task priority to find Pareto-optimal solutions.}
We perform ablation studies on each stage of optimization as shown in \cref{tab:ablation}. When solely utilizing phase 2, its performance has no big difference from the previous optimization techniques. However, when the first phase was used, the lowest averaged multi-task loss was achieved. Additionally, we show the correlation of loss trends in \cref{fig:corres_mat}. The closer the value is to 1, the more it means that the loss of the task pair decreases together. In the initial stages of optimization, phase 1 appears to align the loss more effectively than solely relying on phase 2. This shows that phase 1 aids the network in differentiating task-specific details, leading to the identification of optimal Pareto solutions.

\noindent
\textbf{During Phase 2, the task's priority is likely to be maintained.}
We evaluate the top priority task within the shared space of network using \cref{eq:eq6}. Subsequently, we visualized the percentage of top priority tasks in \cref{fig:connection_strength}. It illustrates how much of the output channels in the shared convolutional layer each task has priority over. We compared when we used only Phase 1 and when we used both Phase 1 and Phase 2. We found Phase 2 at the latter half of the optimization has an effect on conserving learned task priority. This method of priority allocation prevents a specific task from exerting a dominant influence over the entire network as discussed with \cref{eq:eq5}.

\noindent
\textbf{Mixing two phases shows higher performance than using each phase separately.} 
In \cref{tab:ablation}, using only Phase 1 results in a lower multi-task loss than when mixing the two phases. Nonetheless, combining both phases enhances multi-task performance. This improvement can be attributed to the normalized connection strength (refer to \cref{eq:eq5}), which ensures that no single task dominates the entire network during Phase 2. When the two phases are applied sequentially, performance declines compared to our mixing strategy. The reason for this performance degradation seems to be the application of Phase 1 at the later stages of Optimization. This continuously alters the established task priority, which in turn disrupts the gradient's proper updating based on the learned priority.

%% file: sec/8_conclusion.tex
\vspace{-2pt}\section{Conclusion}
\vspace{-3pt}In this paper, we present a novel optimization technique for multi-task learning named connection strength-based optimization. By recognizing task priority within shared network parameters and measuring it using connection strength, we pinpoint which parameters are crucial for distinct tasks. By learning and preserving this task priority during optimization, we are able to identify new Pareto optimal solutions, boosting multi-task performance. We validate the efficacy of our approaches through comprehensive experiments and analysis.

%% file: sec/X_suppl.tex
\clearpage
\newpage
\onecolumn
    \centering
    \Large
    \textbf{\thetitle}\\
    \vspace{0.5em}Supplementary Material \\
    \vspace{1.0em}
\setcounter{page}{1}
\normalsize
\raggedright
\appendix

\section{Theoretical Analysis}
\label{Append:theoretical_analysis}

\subsection{Proof of \cref{theorem1}}
\label{Append:theorem1}
\theom*

\begin{proof}
We start from shared parameters $\Theta_s$ and we can divide them with task priority.
\begin{equation}
\Theta_{s} = \{\theta_{s,1}, \theta_{s,2},...,\theta_{s,\mathcal{K}}\}
\label{Append:eq1}
\end{equation}

Let $\tilde{\Theta}_{s,i}$ represent the parameters in $\Theta_{s}$, excluding $\theta_{s,i}$.
For the sake of simplicity in our proof, we begin by focusing on a subset of shared parameters, specifically $\theta_{s,i}$, to demonstrate that accounting for task priority leads to a reduced multi-task loss compared to neglecting it. Subsequently, we will apply the same process to the remaining shared parameters to complete the proof.
Let $\hat{\mathrm{g}}_k^t$ be the gradient of $\theta_{s,i}^{t}$ for task $\tau_k$ as follows:
\begin{equation}
    \hat{\mathrm{g}}_k^t = \nabla_{\theta^{t}_{s,i}} \mathcal{L}_{k}(\mathcal{X}^t,\tilde{\Theta}_{s,i}^t,\theta_{s,i}^t,\Theta_k^t)
\label{Append:eq3}
\end{equation}

Previous optimization methods involving gradient manipulation update the weighted summation of task-specific gradients. Therefore, we can update $\theta^{t}_{s,i}$ to $\theta^{t+1}_{s,i}$ as follows:
\begin{equation}
\begin{split}
    \mathrm{g}^t = \sum_{j=1}^{\mathcal{K}}\nabla_{\theta^{t}_{s,i}} w_j\mathcal{L}_{j}(\mathcal{X}^t,\tilde{\Theta}_{s,i}^t,\theta_{s,i}^t,\Theta_i^t)
    = \sum_{j=1}^{\mathcal{K}}w_j \hat{\mathrm{g}}_{j}^t
    ,\qquad
    \theta_{s,i}^{t+1} = \theta_{s,i}^{t} - \eta\mathrm{g}^t
\end{split}
\label{Append:eq2}
\end{equation}
where $w_i$ is loss weights of $\tau_i$ and $\sum_{i=1}^{\mathcal{K}}w_i = 1$.

From the first order Taylor approximation of $\mathcal{L}_i$ for $\theta_{s,i}$, we have
\begin{equation}
    \mathcal{L}_{i}(\mathcal{X}^t,\tilde{\Theta}_{s,i}^t,\theta_{s,i}^{t+1},\Theta_i^t) = \mathcal{L}_{i}(\mathcal{X}^t,\tilde{\Theta}_{s,i}^t,\theta_{s,i}^t,\Theta_i^t)
    +(\theta_{s,i}^{t+1}-\theta_{s,i}^{t})^\top \mathrm{\hat{g}}_i^t + O(\eta^2)
\label{Append:eq4}
\end{equation}

On the other hand, when considering task priority, we can update $\theta_{s,i}^{t}$ to $\hat{\theta}_{s,i}^{t+1}$ using $\hat{\mathrm{g}}_i$ as follows:
\begin{equation}
    \hat{\theta}_{s,i}^{t+1} = \theta_{s,i}^{t} - \eta\hat{\mathrm{g}}_i^t
\end{equation}

From the first order Taylor approximation of $\mathcal{L}_i$ from $\theta_{s,i}^{t}$ to $\hat{\theta}_{s,i}^{t+1}$, we have
\begin{equation}
    \mathcal{L}_{i}(\mathcal{X}^t,\tilde{\Theta}_{s,i}^t,\hat{\theta}_{s,i}^{t+1},\Theta_i^t) = \mathcal{L}_{i}(\mathcal{X}^t,\tilde{\Theta}_{s,i}^t,\theta_{s,i}^t,\Theta_i^t)
    +(\hat{\theta}_{s,i}^{t+1}-\theta_{s,i}^{t})^\top \hat{\mathrm{g}}_i^t + O(\eta^2)
\label{Append:eq5}
\end{equation}

The difference between \cref{Append:eq4} and \cref{Append:eq5} is
\begin{align}
     \mathcal{L}_{i}(\mathcal{X}^t,&\tilde{\Theta}_{s,i}^t,\theta_{s,i}^{t+1},\Theta_i^t)-
     \mathcal{L}_{i}(\mathcal{X}^t,\tilde{\Theta}_{s,i}^t,\hat{\theta}_{s,i}^{t+1},\Theta_i^t) = 
     (\theta_{s,i}^{t+1}-\theta_{s,i}^{t})^\top \hat{\mathrm{g}}_i^t - 
     (\hat{\theta}_{s,i}^{t+1}-\theta_{s,i}^{t})^\top \hat{\mathrm{g}}_i^t\\
     &= -\eta (\mathrm{g}^t - \mathrm{\hat{g}}_i^t)^\top \hat{\mathrm{g}}_i^t  
     \label{Append:eq6_3}
\end{align}

Similar to \cref{Append:eq4} and \cref{Append:eq5}, we have the following two inequalities for the last of the losses $\mathcal{L}_j$ where $i\neq j$:
\begin{align}
    \mathcal{L}_{j}(\mathcal{X}^t,\tilde{\Theta}_{s,i}^t,\theta_{s,i}^{t+1},\Theta_i^t) = \mathcal{L}_{j}(\mathcal{X}^t,\tilde{\Theta}_{s,i}^t,\theta_{s,i}^t,\Theta_i^t)
    +(\theta_{s,i}^{t+1}-\theta_{s,i}^{t})^\top \hat{\mathrm{g}}_j^t + O(\eta^2)
    \label{Append:eq7_1}\\
    \mathcal{L}_{j}(\mathcal{X}^t,\tilde{\Theta}_{s,i}^t,\hat{\theta}_{s,i}^{t+1},\Theta_i^t) = \mathcal{L}_{j}(\mathcal{X}^t,\tilde{\Theta}_{s,i}^t,\theta_{s,i}^t,\Theta_i^t)
    +(\hat{\theta}_{s,i}^{t+1}-\theta_{s,i}^{t})^\top \hat{\mathrm{g}}_j^t + O(\eta^2)
    \label{Append:eq7_2}
\end{align}
The result in \cref{Append:eq7_1} corresponds to updating the weighted summation of task-specific gradients, while \cref{Append:eq7_2} reflects the result when updating gradients with consideration for task priority.

The difference between \cref{Append:eq7_1} and \cref{Append:eq7_2} is
\begin{align}
     \mathcal{L}_{j}(\mathcal{X}^t,&\tilde{\Theta}_{s,i}^t,\theta_{s,i}^{t+1},\Theta_i^t)-
     \mathcal{L}_{j}(\mathcal{X}^t,\tilde{\Theta}_{s,i}^t,\hat{\theta}_{s,i}^{t+1},\Theta_i^t) = 
     (\theta_{s,i}^{t+1}-\theta_{s,i}^{t})^\top \hat{\mathrm{g}}_j^t - 
     (\hat{\theta}_{s,i}^{t+1}-\theta_{s,i}^{t})^\top \hat{\mathrm{g}}_j^t\\
     &= -\eta (\mathrm{g}^t - \hat{\mathrm{g}}_i^t)^\top \hat{\mathrm{g}}_j^t  
     \label{Append:eq8_3}
\end{align}

If we sum \cref{Append:eq8_3} over all task losses $\{\mathcal{L}_k\}_{k=1}^{\mathcal{K}}$ along with their corresponding task-specific weights $\{w_k\}_{k=1}^{\mathcal{K}}$, the following result is obtained:
\begin{align}
     \sum_{k=1}^{\mathcal{K}} w_k \mathcal{L}_{k}&(\mathcal{X}^t,\tilde{\Theta}_{s,i}^t,\theta_{s,i}^{t+1},\Theta_i^t)-
     \sum_{k=1}^{\mathcal{K}} w_k \mathcal{L}_{k}(\mathcal{X}^t,\tilde{\Theta}_{s,i}^t,\hat{\theta}_{s,i}^{t+1},\Theta_i^t) \\
     &=-\eta \sum_{k=1}^{\mathcal{K}} w_k (\mathrm{g}^t - \hat{\mathrm{g}}_i^t)^\top \hat{\mathrm{g}}_k^t \\
     &=-\eta \sum_{k=1}^{\mathcal{K}} w_k (\sum_{j=1}^{\mathcal{K}}\nabla_{\theta^{t}_{s,i}} w_j\mathcal{L}_{j}(\mathcal{X}^t,\tilde{\Theta}_{s,i}^t,\theta_{s,i}^t,\Theta_i^t)
     -\nabla_{\theta^{t}_{s,i}} \mathcal{L}_{i}(\mathcal{X}^t,\tilde{\Theta}_{s,i}^t,\theta_{s,i}^t,\Theta_i^t))^\top \hat{\mathrm{g}}_k^t \\
     &=-\eta \sum_{k=1}^{\mathcal{K}} w_k \left(\sum_{j=1}^{\mathcal{K}}w_j\left(\nabla_{\theta^{t}_{s,i}} \mathcal{L}_{j}(\mathcal{X}^t,\tilde{\Theta}_{s,i}^t,\theta_{s,i}^t,\Theta_i^t)
     -\nabla_{\theta^{t}_{s,i}} \mathcal{L}_{i}(\mathcal{X}^t,\tilde{\Theta}_{s,i}^t,\theta_{s,i}^t,\Theta_i^t)\right)\right)^\top \hat{\mathrm{g}}_k^t 
     \label{Append:eq9_4}\\
     &\geq 0
     \label{Append:eq9_5}
\end{align}
The elements within the brackets of \cref{Append:eq9_4} represent a pairwise comparison of the changes in loss resulting from updating the gradients of each task. Thus, the inequality of \cref{Append:eq9_5} holds from \cref{def:task_priority} of task priority. The results indicate that taking task priority into account yields a lower multi-task loss compared to neglecting it. Following a similar process for all shared parameters $\Theta_s=\{\theta_{s,1}, \theta_{s,2},...,\theta_{s,\mathcal{K}}\}$, we can conclude considering task priority leads to the expansion of the known Pareto frontier.

\end{proof}

\subsection{Convergence Analysis}
This section provides theoretical analyses of the proposed optimization method, including a convergence analysis. The overview is as follows:
\begin{enumerate}[leftmargin=*]
\item We present the concept of Pareto-stationarity. Previous methods \citep{RN36, RN20, RN18, senushkin2023independent} have shown their convergence to Pareto stationary points in multi-task optimization. (See \cref{Append:pareto-stationarity}).
\item We offer a convergence analysis for Phase 1 of connection strength-based optimization. The analysis is conducted separately for shared and task-specific parameters. For task-specific parameters, it converges to the Pareto optimal point, similar to simple gradient descent. However, for shared parameters, Phase 1 doesn't ensure convergence to the Pareto optimal point; instead, it enhances the correlation between the gradients of tasks. (See \cref{Append:convergence_phase1})
\item We provide the convergence rate of Phase 1, with a focus on task-specific parameters. (See \cref{Append:convergence_rate_phase1})
\item We present a convergence analysis for Phase 2 of connection strength-based optimization, specifically focusing on the shared parameters of the network. Our analysis shows that Phase 2 converges to the Pareto optimal point, distinguishing it from previous works that converge to Pareto stationary points. (See \cref{Append:convergence_phase2})
\item We provide the convergence rate of Phase 2. (See \cref{Append:convergence_rate_phase2})
\end{enumerate}

\subsubsection{Pareto-stationarity}
\label{Append:pareto-stationarity}
Initially, we establish the concept of a Pareto stationary point. Previous methods \citep{RN36, RN20, RN18, senushkin2023independent} have shown their convergence to Pareto stationary points in multi-task optimization.

\begin{definition}[Pareto stationarity]
The network parameter $\Theta$ is defined with task-specific losses $\{\mathcal{L}_{i}\}^{\mathcal{K}}_{i=1}$. If the sum of weighted gradients $\sum_{i=1}^{\mathcal{K}} w_i \nabla_{\Theta} \mathcal{L}_i=0$, then the point is termed Pareto stationary, indicating the absence of a descent direction from that point.
\end{definition}

Previous research \citep{RN36, RN20, RN18, senushkin2023independent} has demonstrated their convergence to Pareto stationary points, which carries the risk of leading to sub-optimal solutions. This is due to the fact that Pareto-stationarity is a necessary condition for Pareto-optimality. In contrast, our work establishes convergence to the Pareto optimal point during Phase 2 of connection strength-based optimization. Phase 1 doesn't assure attainment of the Pareto optimal solution. Instead, it enhances the correlation between task gradients, amplifying the significance of task-specific parameters to learn task priorities.

\subsubsection{Convergence of Phase 1}
\label{Append:convergence_phase1}

In the subsequent convergence analysis, we omit the input $\mathcal{X}^t$ for clarity.
\begin{theorem}[Convergence of Phase 1]
\label{Append:conv_phase1}
Assume losses $\{\mathcal{L}_i\}_{i=1}^{\mathcal{K}}$ are convex and differentiable and the gradient of $\{\mathcal{L}_i\}_{i=1}^{\mathcal{K}}$ is Lipschitz continuous with constant $H>0$, i.e. $||\nabla \mathcal{L}_i (x) - \nabla \mathcal{L}_i (y)|| \leq H||x-y||$ for $i=1,2,...,\mathcal{K}$. Phase 1 of connection strength optimization, with a step size $\eta \leq \frac{1}{H}$, will converge to the Pareto optimal point for task-specific parameters $\{\Theta_i\}_{i=1}^{\mathcal{K}}$. For shared parameters $\Theta_s$ with a step size $\eta \leq \frac{2}{H}$, it does not guarantee convergence to the Pareto optimal point, but it optimizes in the direction to increase the correlation between tasks' gradients.
\end{theorem}

\begin{proof}
We begin by conducting a quadratic expansion of the task-specific loss $\mathcal{L}_i(\Theta_s^t, \Theta_i^t)$ concerning the parameters $\Theta_s^t$ and $\Theta_i^t$ at each update step of Phase 1 for sequential tasks.
\begin{align}
    \mathcal{L}_i  (\Theta_s^{t+i/\mathcal{K}}, \Theta_i^{t+1})
    \leq &\mathcal{L}_i(\Theta_s^{t+(i-1)/\mathcal{K}}, \Theta_i^{t})\\
    &+\nabla_{\Theta_s^{t+(i-1)/\mathcal{K}}}\mathcal{L}_i(\Theta_s^{t+(i-1)/\mathcal{K}}, \Theta_i^{t})(\Theta_s^{t+i/\mathcal{K}}-\Theta_s^{t+(i-1)/\mathcal{K}})\\
    &+\frac{1}{2}\nabla_{\Theta_s^{t+(i-1)/\mathcal{K}}}^{2}\mathcal{L}_i(\Theta_s^{t+(i-1)/\mathcal{K}}, \Theta_i^{t})(\Theta_s^{t+i/\mathcal{K}}-\Theta_s^{t+(i-1)/\mathcal{K}})^{2}\\
    &+\nabla_{\Theta_i^{t}}\mathcal{L}_i(\Theta_s^{t+(i-1)/\mathcal{K}}, \Theta_i^{t})(\Theta_i^{t+1}-\Theta_i^{t})\\
    &+\frac{1}{2}\nabla_{\Theta_i^{t+(i-1)/\mathcal{K}}}^{2}\mathcal{L}_i(\Theta_s^{t+(i-1)/\mathcal{K}}, \Theta_i^{t})(\Theta_i^{t+1}-\Theta_i^{t})^{2}\\
    \leq &\mathcal{L}_i(\Theta_s^{t+(i-1)/\mathcal{K}}, \Theta_i^{t}) \label{Append:conv1_eq2}\\
    &+\nabla_{\Theta_s^{t+(i-1)/\mathcal{K}}}\mathcal{L}_i(\Theta_s^{t+(i-1)/\mathcal{K}}, \Theta_i^{t})(\Theta_s^{t+i/\mathcal{K}}-\Theta_s^{t+(i-1)/\mathcal{K}})\\
    &+\frac{1}{2} H (\Theta_s^{t+i/\mathcal{K}}-\Theta_s^{t+(i-1)/\mathcal{K}})^{2}\\
    &+\nabla_{\Theta_i^{t}}\mathcal{L}_i(\Theta_s^{t+(i-1)/\mathcal{K}}, \Theta_i^{t})(\Theta_i^{t+1}-\Theta_i^{t})\\
    &+\frac{1}{2}H(\Theta_i^{t+1}-\Theta_i^{t})^{2} \label{Append:conv1_eq2_last}
\end{align}

for $i=1,2,...,\mathcal{K}$. The inequality in \cref{Append:conv1_eq2} holds as $\nabla\mathcal{L}$ is Lipschitz continuous with constant $H$ which implies that $\nabla^{2}\mathcal{L}-HI\leq0$. We follow the gradient update rule for Phase 1 in connection strength-based optimization:
\begin{align}
  \Theta^{t+i/\mathcal{K}}=\Theta^{t+(i-1)/\mathcal{K}}-\eta w_i \nabla_{\Theta_s^{t+(i-1)/\mathcal{K}}}\mathcal{L}_i(\Theta_s^{t+(i-1)/\mathcal{K}}, \Theta_i^{t})
  \label{Append:conv1_eq3}\\
  \Theta_i^{t+1}=\Theta_i^{t} - \eta w_i \nabla_{\Theta_i^{t}} \mathcal{L}_i (\Theta_i^{t})
  \label{Append:conv1_eq4}
\end{align}

for $i=1,2,...,\mathcal{K}$. To simplify the proof, we partition the equation into two subsets—one for shared parameters $\Theta_s$ and the other for task-specific parameters $\Theta_i$. 

(i) For task-specific parameter $\Theta_i$, the following inequality holds:
\begin{align}
    \mathcal{L}_i  (\Theta_s^{t}, \Theta_i^{t+1}) \leq \mathcal{L}_i(\Theta_s^{t}, \Theta_i^{t})
    +\nabla_{\Theta_i^{t}}\mathcal{L}_i(\Theta_s^{t}, \Theta_i^{t})(\Theta_i^{t+1}-\Theta_i^{t})
    +\frac{1}{2}H(\Theta_i^{t+1}-\Theta_i^{t})^{2}
    \label{Append:conv1_eq5}
\end{align}

We denote $\mathrm{g}_i^t$ as the gradient of $\Theta_i^{t}$ for task $\tau_i$ as follows:
\begin{equation}
    \mathrm{g}_i^t = \nabla_{\Theta_i^{t}} \mathcal{L}_{i}(\Theta_s^{t}, \Theta_i^{t})
\end{equation}

If we substitute \cref{Append:conv1_eq4} into \cref{Append:conv1_eq5}, it becomes as follows:
\begin{align}
    \mathcal{L}_i(\Theta_s^{t}, \Theta_i^{t+1})
    \leq&\mathcal{L}_i(\Theta_s^{t}, \Theta_i^{t})
    - \eta w_i ||\mathrm{g}_i^{t}||^2 + \frac{\eta_i^2 w_i^2}{2}H ||{\mathrm{g}_i^{t}}||^2 \\
    =&\mathcal{L}_i(\Theta_s^{t}, \Theta_i^{t})
    - \eta w_i (1-\frac{1}{2} \eta w_i H) ||\mathrm{g}_i^{t}||^2\\
    \leq&\mathcal{L}_i(\Theta_s^{t}, \Theta_i^{t})
    - \frac{1}{2} \eta w_i ||\mathrm{g}_i^{t}||^2
    \label{Append:conv1_eq7}
\end{align}

\cref{Append:conv1_eq7} is valid when the step size $\eta$ is sufficiently small, specifically, when $\eta \leq \frac{1}{H w_i}$. When we sum \cref{Append:conv1_eq7} over all task losses $\{\mathcal{L}_k\}_{k=1}^{\mathcal{K}}$ along with their corresponding task-specific weights $\{w_k\}_{k=1}^{\mathcal{K}}$, the following result is obtained:
\begin{align}
    \sum_{i=1}^{\mathcal{K}} w_i \mathcal{L}_i(\Theta_s^{t}, \Theta_i^{t+1})
    \leq& \sum_{i=1}^{\mathcal{K}} w_i \mathcal{L}_i(\Theta_s^{t}, \Theta_i^{t})
    - \frac{1}{2} \sum_{i=1}^{\mathcal{K}} \eta w_i^2 ||\mathrm{g}_i^{t}||^2
    \label{Append:conv1_eq8}
\end{align}
According to \cref{Append:conv1_eq8}, we can infer that the application of Phase 1 in connection strength-based optimization can result in $\mathrm{g}_i=0$ for $i = {1,2,...,\mathcal{K}}$. The condition $g_i^t=0$ indicates that the proposed updating rule converges to the Pareto-optimal point for task-specific parameters $\Theta_i$ for $i = \{1,2,...,\mathcal{K}\}$.
\newline

(ii) For shared parameter $\Theta_s$, the following inequality holds:
\begin{align}
    \mathcal{L}_i  (\Theta_s^{t+i/\mathcal{K}}, \Theta_i^{t}) \leq \mathcal{L}_i(\Theta_s^{t+(i-1)/\mathcal{K}}, \Theta_i^{t})
    +&\nabla_{\Theta_s^{t+(i-1)/\mathcal{K}}}\mathcal{L}_i(\Theta_s^{t+(i-1)/\mathcal{K}}, \Theta_i^{t})(\Theta_s^{t+i/\mathcal{K}}-\Theta_s^{t+(i-1)/\mathcal{K}}) \label{Append:conv1_eq9}\\
    +&\frac{1}{2}H(\Theta_s^{t+i/\mathcal{K}}-\Theta_s^{t+(i-1)/\mathcal{K}})^{2}
    \label{Append:conv1_eq10}
\end{align}

In case (ii), we denote $\mathrm{g}_i^t$ as the gradient of $\Theta_s^{t}$ for task $\tau_i$, and $\mathrm{g}^t$ as the weighted sum of $\{\mathrm{g}_i^t\}_{i=1}^{\mathcal{K}}$ with $\{w_i\}_{i=1}^{\mathcal{K}}$ as follows:
\begin{equation}
    \mathrm{g}_i^t = \nabla_{\Theta_s^{t}} \mathcal{L}_{i}(\Theta_s^{t}, \Theta_i^{t})
    ,\qquad
    \mathrm{g}^t = \sum_{i=1}^{\mathcal{K}} w_i \nabla \mathcal{L}_{i}(\Theta_s^t, \Theta_i^t)
\end{equation}

If we substitute \cref{Append:conv1_eq3} into \cref{Append:conv1_eq9} and \cref{Append:conv1_eq10}, it becomes as follows:
\begin{align}
     \mathcal{L}_i  (\Theta_s^{t+i/\mathcal{K}}, \Theta_i^{t})
    \leq&\mathcal{L}_i(\Theta_s^{t+(i-1)/\mathcal{K}}, \Theta_i^{t})
    - \eta w_i ||\mathrm{g}_i^{t+(i-1)/\mathcal{K}}||^2 + \frac{\eta^2 w_i^2}{2}H ||{\mathrm{g}_i^{t+(i-1)/\mathcal{K}}}||^2 \label{Append:conv1_eq11}
\end{align}

Similarly, the quadratic expansion of $\mathcal{L}_j$ for $\Theta_s^{t+i/\mathcal{K}}$ when $i\neq j$ is as follows:
\begin{align}
    \mathcal{L}_j  (\Theta_s^{t+i/\mathcal{K}}, \Theta_j^{t}) 
    &\leq \mathcal{L}_j(\Theta_s^{t+(i-1)/\mathcal{K}}, \Theta_j^{t})
    +\nabla_{\Theta_s^{t+(i-1)/\mathcal{K}}}\mathcal{L}_j(\Theta_s^{t+(i-1)/\mathcal{K}}, \Theta_j^{t})(\Theta_s^{t+i/\mathcal{K}}-\Theta_s^{t+(i-1)/\mathcal{K}})\\
    &\quad +\frac{1}{2}H(\Theta_s^{t+i/\mathcal{K}}-\Theta_s^{t+(i-1)/\mathcal{K}})^{2}\\
    &\leq \mathcal{L}_j(\Theta_s^{t+(i-1)/\mathcal{K}}, \Theta_j^{t})
    - \eta w_i \mathrm{g}_i^{t+(i-1)/\mathcal{K}} \cdot \mathrm{g}_j^{t+(i-1)/\mathcal{K}}
    + \frac{\eta^2 w_i^2}{2} H ||{\mathrm{g}_i^{t+(i-1)/\mathcal{K}}}||^2
    \label{Append:conv1_eq18}
\end{align}

When we sum \cref{Append:conv1_eq11} and \cref{Append:conv1_eq18} over all task losses $\{\mathcal{L}_k\}_{k=1}^{\mathcal{K}}$ along with their corresponding task-specific weights $\{w_k\}_{k=1}^{\mathcal{K}}$, the following result is obtained:
\begin{align}
    \sum_{k=1}^{\mathcal{K}} w_k \mathcal{L}_k(\Theta_s^{t+i/\mathcal{K}}, \Theta_k^{t})
    &\leq \sum_{k=1}^{\mathcal{K}} w_k \mathcal{L}_k(\Theta_s^{t+(i-1)/\mathcal{K}}, \Theta_k^{t})
    - \eta w_i \sum_{k=1}^{\mathcal{K}} w_k \mathrm{g}_i^{t+(i-1)/\mathcal{K}} \cdot \mathrm{g}_k^{t+(i-1)/\mathcal{K}}
    + \frac{\eta^2 w_i^2}{2} H ||{\mathrm{g}_i^{t+(i-1)/\mathcal{K}}}||^2\\
    &= \sum_{k=1}^{\mathcal{K}} w_k \mathcal{L}_k(\Theta_s^{t+(i-1)/\mathcal{K}}, \Theta_k^{t})
    - \eta w_i \mathrm{g}_i^{t+(i-1)/\mathcal{K}} \cdot \mathrm{g}^{t+(i-1)/\mathcal{K}}
    + \frac{\eta^2 w_i^2}{2} H ||{\mathrm{g}_i^{t+(i-1)/\mathcal{K}}}||^2\\
    &\leq \sum_{k=1}^{\mathcal{K}} w_k \mathcal{L}_k(\Theta_s^{t+(i-1)/\mathcal{K}}, \Theta_k^{t})
    - \eta w_i \mathrm{g}_i^{t+(i-1)/\mathcal{K}} \cdot \mathrm{g}^{t+(i-1)/\mathcal{K}}
    + \eta w_i^2 ||{\mathrm{g}_i^{t+(i-1)/\mathcal{K}}}||^2
    \label{Append:conv1_eq21}\\
    &= \sum_{k=1}^{\mathcal{K}} w_k \mathcal{L}_k(\Theta_s^{t+(i-1)/\mathcal{K}}, \Theta_k^{t})
    - \eta w_i \mathrm{g}_i^{t+(i-1)/\mathcal{K}} \cdot (\mathrm{g}^{t+(i-1)/\mathcal{K}} -w_i \mathrm{g}_i^{t+(i-1)/\mathcal{K}})
    \label{Append:conv1_eq22}
\end{align}

\cref{Append:conv1_eq21} is valid when the step size $\eta$ is sufficiently small, specifically, when $\eta \leq \frac{2}{H}$. As shown in \cref{Append:conv1_eq22}, Phase 1 of connection strength-based optimization does not strictly ensure convergence. This is attributed to its sequential updating of task-specific connections, leading to fluctuations in their losses during training. Nevertheless, as illustrated in \cref{Append:conv1_eq22}, we can note that the optimization moves in the direction of minimizing the dot product between the gradient of the currently updated task $\mathrm{g}_i^{t+(i-1)/\mathcal{K}}$ and the weighted sum of gradients from the remaining losses $(\mathrm{g}^{t+(i-1)/\mathcal{K}} -w_i \mathrm{g}_i^{t+(i-1)/\mathcal{K}})$. This observation aligns with the experimental results presented in \cref{fig:corres_mat}. Phase 1 effectively increases the correlation between tasks in shared parameters $\Theta_s$, which exaggerates the role of task-specific parameters, allowing it to sufficiently grasp and establish task priorities.

\end{proof}

\subsubsection{Convergence rate of Phase 1}
\label{Append:convergence_rate_phase1}

\begin{theorem}[Convergence rate of Phase 1]
Assume losses $\{\mathcal{L}_i\}_{i=1}^{\mathcal{K}}$ are convex and differentiable and the gradient of $\{\mathcal{L}_i\}_{i=1}^{\mathcal{K}}$ is Lipschitz continuous with constant $H>0$, i.e. $||\nabla \mathcal{L}_i (x) - \nabla \mathcal{L}_i (y)|| \leq H||x-y||$ for $i=1,2,...,\mathcal{K}$. Then, in phase 1 of connection strength optimization with a step size $\eta \leq \frac{1}{H}$, the system will reach the Pareto optimal point for task-specific parameters $\{\Theta_i\}_{i=1}^{\mathcal{K}}$ at a rate of $O(1/T)$, where $T$ is the total number of iterations. This is guaranteed by the following inequality:

\begin{align}
    \min_{0\leq t \leq T} \sum_{\substack{k=1}}^{\mathcal{K}} w_k^2 ||\mathrm{g}_{k}^t||^{2}
    \leq& \frac{2}{\eta T} (\mathcal{L} (\Theta^{0}) - \mathcal{L} (\Theta^{*}))
\end{align}
where $\Theta^{*}$ represents the converged parameters, and $T$ is the total number of iterations.
\end{theorem}

\begin{proof}
We begin with the result from \cref{Append:conv1_eq8}. To simplify, let $\mathcal{L}$ represent the total loss, defined as $\mathcal{L} (\Theta^t) = \sum_{i=1}^{\mathcal{K}} w_i \mathcal{L}_i (\Theta^t)$. We only consider task-specific parameters $\{\Theta_i\}_{i=1}^{\mathcal{K}}$ for analysis.
\begin{align}
    \mathcal{L} (\Theta^{t+1})
    \leq& \mathcal{L} (\Theta^{t})
    - \frac{1}{2} \sum_{i=1}^{\mathcal{K}} \eta w_i^2 ||\mathrm{g}_i^{t}||^2
    \label{Append:conv1_rate_eq2}
\end{align}

By rearranging the term in \cref{Append:conv1_rate_eq2}:
\begin{align}
    \sum_{\substack{i=1}}^{\mathcal{K}} w_i^2 ||\mathrm{g}_{i}^t||^{2}
    \leq \frac{2}{\eta} (\mathcal{L} (\Theta^{t}) - \mathcal{L} (\Theta^{t+1}))
    \label{Append:conv1_rate_eq5}
\end{align}

If we consider iterations for $t \in [0, T]$, then we have:
\begin{align}
    \min_{0\leq t \leq T} \sum_{\substack{k=1}}^{\mathcal{K}} w_k^2 ||\mathrm{g}_{k}^t||^{2}
    \leq& \frac{2}{\eta T} \sum_{t=0}^{T-1} (\mathcal{L} (\Theta^{t}) - \mathcal{L} (\Theta^{t+1}))\\
    =& \frac{2}{\eta T} (\mathcal{L} (\Theta^{0}) - \mathcal{L} (\Theta^{T}))\\
    \leq& \frac{2}{\eta T} (\mathcal{L} (\Theta^{0}) - \mathcal{L} (\Theta^{*}))
\end{align}
where $\Theta^{*}$ represents the converged parameters. Our approach maintains a convergence rate of $O(1/T)$ for task-specific parameters $\{\Theta_i\}_{i=1}^{\mathcal{K}}$.

\end{proof}

\subsubsection{Convergence of Phase 2}
\label{Append:convergence_phase2}

In the subsequent convergence analysis, we omit the input $\mathcal{X}^t$ for clarity.
\begin{theorem}[Convergence of Phase 2]
\label{Append:conv_phase2}
Assume losses $\{\mathcal{L}_i\}_{i=1}^{\mathcal{K}}$ are convex and differentiable and the gradient of $\{\mathcal{L}_i\}_{i=1}^{\mathcal{K}}$ is Lipschitz continuous with constant $H>0$, i.e. $||\nabla \mathcal{L}_i (x) - \nabla \mathcal{L}_i (y)|| \leq H||x-y||$ for $i=1,2,...,\mathcal{K}$ Then, phase 2 of connection strength optimization with step size $\eta \leq \frac{1}{H w_i}$ for all $i=1,2,...,\mathcal{K}$ will converge to the Pareto-optimal point.
\end{theorem}

\begin{proof}
We start from quadratic expansion of task-specific loss of task $\tau_i$ for $\theta_{s,j}$.
\begin{align}
    \mathcal{L}_i (\theta_{s,j}^{t+1}, \tilde{\Theta}_{s,j}^t, \Theta_i^t)
    \leq&\mathcal{L}_{i}(\theta_{s,j}^t, \tilde{\Theta}_{s,j}^{t}, \Theta_i^t)
    +\nabla\mathcal{L}_i(\theta_{s,j}^t, \tilde{\Theta}_{s,j}^{t}, \Theta_i^t)(\theta_{s,i}^{t+1}-\theta_{s,i}^t)
    +\frac{1}{2}\nabla^{2}\mathcal{L}_{i}(\theta_{s,j}^t, \tilde{\Theta}_{s,j}^{t}, \Theta_i^t)(\theta_{s,i}^{t+1}-\theta_{s,i}^t)^{2}\\
    \leq&\mathcal{L}_{i}(\theta_{s,j}^t, \tilde{\Theta}_{s,j}^{t}, \Theta_i^t)
    +\nabla\mathcal{L}_i(\theta_{s,j}^t, \tilde{\Theta}_{s,j}^{t}, \Theta_i^t)(\theta_{s,i}^{t+1}-\theta_{s,i}^t)
    +\frac{1}{2}H(\theta_{s,i}^{t+1}-\theta_{s,i}^t)^{2} \label{Append:proof2_eq2}
\end{align}

The inequality in \cref{Append:proof2_eq2} holds as $\nabla\mathcal{L}$ is Lipschitz continuous with constant $H$. It implies that $\nabla^{2}\mathcal{L}-HI\leq0$.

Let $\mathrm{g}_k^t$ be the gradient of $\theta_{s,j}^{t}$ for task $\tau_k$ as follows:
\begin{equation}
    \mathrm{g}_k^t = \nabla_{\theta^{t}_{s,j}} \mathcal{L}_{k}(\tilde{\Theta}_{s,j}^t,\theta_{s,i}^t,\Theta_k^t)
\label{Append:proof2_eq3}
\end{equation}

The gradient update rule for Phase 1 in connection strength-based optimization is as follows:
\begin{equation}
  \theta_{s,i}^{t+1}=\begin{cases}
    \theta_{s,i}^{t}-\eta w_i (\mathrm{g}_i^t), & \text{if $i=j$}.\\
    \theta_{s,i}^{t}-\eta w_j (\mathrm{g}_j^t-\frac{\mathrm{g}_i^t \cdot \mathrm{g}_j^t}{||\mathrm{g}_i^t||^2}\mathrm{g}_i^t), & \text{otherwise}.
  \end{cases}
  \label{Append:proof2_eq4}
\end{equation}

(i) When $i=j$, if we substitute \cref{Append:proof2_eq4} into \cref{Append:proof2_eq2}, it becomes as follows.
\begin{align}
    \mathcal{L}_i (\theta_{s,j}^{t+1}, \tilde{\Theta}_{s,j}^t, \Theta_i^t)
    \leq&\mathcal{L}_{i}(\theta_{s,j}^t, \tilde{\Theta}_{s,j}^{t}, \Theta_i^t)
    -\eta w_i ||\mathrm{g}_i^t||^2+\frac{\eta^2 w_i^2}{2} H ||\mathrm{g}_i^t||^2\\
    =&\mathcal{L}_{i}(\theta_{s,j}^t, \tilde{\Theta}_{s,j}^{t}, \Theta_i^t)
    -\eta w_i ||\mathrm{g}_i^t||^2 (1- \frac{1}{2} \eta w_i H)
\end{align}

Assuming that the step size $\eta$ is sufficiently small, such that $\eta \leq \frac{1}{H w_i}$. Thus the following inequality holds:
\begin{align}
    \mathcal{L}_i (\theta_{s,j}^{t+1}, \tilde{\Theta}_{s,j}^t, \Theta_i^t)
    \leq&\mathcal{L}_{i}(\theta_{s,j}^t, \tilde{\Theta}_{s,j}^{t}, \Theta_i^t)
    -\frac{1}{2} \eta w_i ||\mathrm{g}_i^t||^2
    \label{Append:proof2_result0}
\end{align}

(ii) When $i\neq j$, if we substitute \cref{Append:proof2_eq4} into \cref{Append:proof2_eq2} similarly, it becomes as follows.
\begin{align}
    \mathcal{L}_i (\theta_{s,j}^{t+1},& \tilde{\Theta}_{s,j}^t, \Theta_i^t)
    \leq \mathcal{L}_{i}(\theta_{s,j}^t, \tilde{\Theta}_{s,j}^{t}, \Theta_i^t)
    -\eta w_j \mathrm{g}_j^t (\mathrm{g}_j^t-\frac{\mathrm{g}_i^t \cdot \mathrm{g}_j^t}{||\mathrm{g}_{i}^t||^2} \mathrm{g}_i^t)
    + \frac{\eta^2 w_j^2}{2} H ||(\mathrm{g}_j^t-\frac{\mathrm{g}_i^t \cdot \mathrm{g}_j^t}{||\mathrm{g}_{i}^t||^2} \mathrm{g}_i^t)||^2 \\
    =& \mathcal{L}_{i}(\theta_{s,j}^t, \tilde{\Theta}_{s,j}^{t}, \Theta_i^t)
    -\eta w_j(||\mathrm{g}_{j}^t||^{2}-\frac{(\mathrm{g}_i^t \cdot \mathrm{g}_j^t)^2}{||\mathrm{g}_{i}^t||^2}))
    +\frac{\eta^2 w_j^2}{2} H (||\mathrm{g}_{j}^t||^{2}-2 \frac{(\mathrm{g}_i^t \cdot \mathrm{g}_j^t)^2}{||\mathrm{g}_{i}^t||^2}
    +\frac{(\mathrm{g}_i^t \cdot \mathrm{g}_j^t)^2}{||\mathrm{g}_{i}^t||^2})\\
    =& \mathcal{L}_{i}(\theta_{s,j}^t, \tilde{\Theta}_{s,j}^{t}, \Theta_i^t)
    -\eta w_j (1-\frac{1}{2}\eta w_j H)(||\mathrm{g}_{j}^t||^{2}-\frac{(\mathrm{g}_i^t \cdot \mathrm{g}_j^t)^2}{||\mathrm{g}_{i}^t||^2}))
\end{align}

Given that the step size $\eta$ satisfies $\eta \leq \frac{1}{H w_j}$, the following inequality holds.
\begin{align}
    \mathcal{L}_i (\theta_{s,j}^{t+1}, \tilde{\Theta}_{s,j}^t, \Theta_i^t)
    \leq& \mathcal{L}_{i}(\theta_{s,j}^t, \tilde{\Theta}_{s,j}^{t}, \Theta_i^t)
    -\frac{1}{2}\eta w_j (||\mathrm{g}_{j}||^{2}-\frac{(\mathrm{g}_i^t \cdot \mathrm{g}_j^t)^2}{||\mathrm{g}_{i}^t||^2}))\\
    =& \mathcal{L}_{i}(\theta_{s,j}^t, \tilde{\Theta}_{s,j}^{t}, \Theta_i^t)
    -\frac{1}{2}\eta w_j ||\mathrm{g}_{j}^t||^{2}(1-\frac{(\mathrm{g}_i^t \cdot \mathrm{g}_j^t)^2}{||\mathrm{g}_{i}^t||^2 ||\mathrm{g}_{j}^t||^2}))\\
    =& \mathcal{L}_{i}(\theta_{s,j}^t, \tilde{\Theta}_{s,j}^{t}, \Theta_i^t)
    -\frac{1}{2}\eta w_j ||\mathrm{g}_{j}^t||^{2}(1-cos^{2}\phi_{ij}^t)
    \label{Append:proof2_result}
\end{align}
where $\phi_{ij}^t$ is the angle between $\mathrm{g}_i^t$ and $\mathrm{g}_j^t$. When we sum \cref{Append:proof2_result0} and \cref{Append:proof2_result} over all task losses $\{\mathcal{L}_k\}_{k=1}^{\mathcal{K}}$ along with their corresponding task-specific weights $\{w_k\}_{k=1}^{\mathcal{K}}$, the following result is obtained:
\begin{align}
    \sum_{k=1}^{\mathcal{K}} w_k \mathcal{L}_k (\theta_{s,j}^{t+1}, \tilde{\Theta}_{s,j}^t, \Theta_k^t)
    \leq \sum_{k=1}^{\mathcal{K}} w_k \mathcal{L}_k (\theta_{s,j}^{t}, \tilde{\Theta}_{s,j}^t, \Theta_k^t)
    -\frac{1}{2}\eta (w_j^2 ||\mathrm{g}_j^t||^2  + \sum_{\substack{k=1 \\ k\neq j}}^{\mathcal{K}} w_k^2 ||\mathrm{g}_{k}^t||^{2}(1-cos^{2}\phi_{jk}^t))
    \label{Append:proof2_result2}
\end{align}

We can follow a similar process for all shared parameters $\Theta_s = \{\theta_{s,1}, \theta_{s,2}, ..., \theta_{s,\mathcal{K}}\}$. The second term on the right side of \cref{Append:proof2_result2} is not smaller than zero, proving their convergence. This term can be zero only when $\mathrm{g}_{k}^t=0$ for all $k=1,2,...,\mathcal{K}$. Thus, we can conclude that the application of Phase 2 in connection strength-based optimization can lead to a Pareto-optimal state, as all task-specific gradients converge to zero in the optimization process. Understanding the task priority of each parameter enables the expansion of the known Pareto frontier which is consistent with the results of \cref{theorem1}. Repeatedly applying Phase 2 of connection strength-based optimization ultimately leads to Pareto optimality.
\end{proof}

\subsubsection{Convergence rate of Phase 2}
\label{Append:convergence_rate_phase2}

\begin{theorem}[Convergence rate of Phase 2]
Assume losses $\{\mathcal{L}_i\}_{i=1}^{\mathcal{K}}$ are differentiable and the gradient of $\{\mathcal{L}_i\}_{i=1}^{\mathcal{K}}$ is Lipschitz continuous with constant $H>0$, i.e. $||\nabla \mathcal{L}_i (x) - \nabla \mathcal{L}_i (y)|| \leq H||x-y||$ for $i=1,2,...,\mathcal{K}$ Then, phase 2 of connection strength optimization with step size $\eta \leq \frac{1}{H}$, the system will reach the Pareto optimal point at a rate of $O(1/T)$, where $T$ is the total number of iterations. This is guaranteed by the following inequality:

\begin{align}
    \min_{0\leq t \leq T} \sum_{\substack{k=1}}^{\mathcal{K}} w_k^2 ||\mathrm{g}_{k}^t||^{2}
    \leq& \frac{2}{\eta (1-\alpha^2)T} (\mathcal{L} (\Theta^{0}) - \mathcal{L} (\Theta^{*}))
\end{align}
where $\Theta^{*}$ represents the converged parameters, $\alpha$ is a constant satisfying $\alpha > -1$, and $T$ is the total number of iterations.
\end{theorem}

\begin{proof}
We start with the outcome (\cref{Append:proof2_result2}) derived in \cref{Append:conv_phase2}. For simplicity, consider the following notation.
\begin{equation}
\begin{split}
    \mathcal{L} (\Theta^t) = \sum_{k=1}^{\mathcal{K}} w_k \mathcal{L}_k (\Theta^t)
    ,\qquad
    \mathrm{g}^t = \sum_{j=1}^{\mathcal{K}} w_j \nabla \mathcal{L}_{j}(\Theta^t)
\end{split}
\end{equation}

And each update iteration $t$ is indicated as a superscript for the gradients. Therefore, \cref{Append:proof2_result2} can be expressed as follows:
\begin{align}
    \mathcal{L} (\Theta^{t+1})
    \leq& \mathcal{L} (\Theta^{t})
    -\frac{1}{2}\eta (w_j^2 ||\mathrm{g}_j^t||^2  + \sum_{\substack{k=1 \\ k\neq j}}^{\mathcal{K}} w_k^2 ||\mathrm{g}_{k}^t||^{2}(1-cos^{2}\phi_{jk}^t))
    \label{Append:proof3_eq2}
\end{align}
The term $(1-\cos^{2}\phi_{jk}^t) \leq 1$ holds for all $k=1,2,...,\mathcal{K}$.\\ 
Let $c$ represent the task number that minimizes the term $1-\cos^{2}\phi_{jk}^t$ excluding $j$.
\begin{align}
    c = \argmin_{\substack{k \\ k\neq j}}(1-cos^{2}\phi_{jk}^t))
    \label{Append:proof3_eq3}
\end{align}

By employing \cref{Append:proof3_eq3} in \cref{Append:proof3_eq2}, the following inequality holds:
\begin{align}
    \mathcal{L} (\Theta^{t+1})
    \leq& \mathcal{L} (\Theta^{t})
    -\frac{1}{2}\eta (w_j^2 ||\mathrm{g}_j^t||^2  + \sum_{\substack{k=1 \\ k\neq j}}^{\mathcal{K}} w_k^2 ||\mathrm{g}_{k}^t||^{2}(1-cos^{2}\phi_{jc}^t))\\
    \leq& \mathcal{L} (\Theta^{t})
    -\frac{1}{2}\eta (w_j^2 ||\mathrm{g}_j^t||^2 (1-cos^{2}\phi_{jc}^t)  + \sum_{\substack{k=1 \\ k\neq j}}^{\mathcal{K}} w_k^2 ||\mathrm{g}_{k}^t||^{2}(1-cos^{2}\phi_{jc}^t))\\
    =& \mathcal{L} (\Theta^{t})
    -\frac{1}{2}\eta (1-cos^{2}\phi_{jc}^t) \sum_{\substack{k=1}}^{\mathcal{K}} w_k^2 ||\mathrm{g}_{k}^t||^{2}
    \label{Append:proof3_eq4}
\end{align}

By rearranging the term in \cref{Append:proof3_eq4}:
\begin{align}
    \sum_{\substack{k=1}}^{\mathcal{K}} w_k^2 ||\mathrm{g}_{k}^t||^{2}
    \leq& \frac{2}{\eta (1-cos^{2}\phi_{jc}^t)} (\mathcal{L} (\Theta^{t}) - \mathcal{L} (\Theta^{t+1}))
\end{align}

If we consider iterations for $t \in [0, T]$ and let $\alpha$ satisfy $cos\phi_{jc}^t \geq \alpha > -1$, then we have:
\begin{align}
    \min_{0\leq t \leq T} \sum_{\substack{k=1}}^{\mathcal{K}} w_k^2 ||\mathrm{g}_{k}^t||^{2}
    \leq& \frac{2}{\eta (1-\alpha^2)T} \sum_{t=0}^{T-1} (\mathcal{L} (\Theta^{t}) - \mathcal{L} (\Theta^{t+1}))\\
    =& \frac{2}{\eta (1-\alpha^2)T} (\mathcal{L} (\Theta^{0}) - \mathcal{L} (\Theta^{T}))\\
    \leq& \frac{2}{\eta (1-\alpha^2)T} (\mathcal{L} (\Theta^{0}) - \mathcal{L} (\Theta^{*}))
\end{align}
where $\Theta^{*}$ represents the converged parameters. Our approach maintains a convergence rate of $O(1/T)$.

\end{proof}


\section{Loss scaling methods}
\label{Append:loss}
In this paper, we used 4 different loss scaling methods to weigh multiple tasks' losses.
\begin{enumerate}[leftmargin=*]
\item All tasks' losses are weighted equally.
\item The weights of tasks are tuned manually following the previous works \citep{RN29, RN32}. For NYUD-v2, the weight of losses is as follows: \newline
\centerline{Depth : SemSeg : Surface Normal : Edge = 1.0 : 1.0 : 10.0 : 50.0} \newline
For PASCAL-Context, the weight of losses is as follows: \newline
\centerline{Semseg : PartSeg : Saliency : Surface Normal : Edge = 1.0 : 2.0 : 5.0 : 10.0 : 50.0}
\item The losses are dynamically weighted by homoscedastic uncertainty \citep{RN23}. \newline
An uncertainty that cannot be reduced with increasing data is called Aleatoric uncertainty.
Homoscedastic uncertainty is a kind of Aleatoric uncertainty that stays constant for all input data and varies between different tasks. So it is also called task-dependent uncertainty. Homoscedastic uncertainty is formulated differently depending on whether the task is a regression task or a classification task as each of them uses different output functions: A regression task uses Gaussian Likelihood, in contrast, a classification task uses softmax function. The objectives of uncertainty weighting are as follows:
\begin{equation}
\mathcal{L}_{Total}= \sum_{i=1}^{\mathcal{K}}\hat{\mathcal{L}}_i \qquad where \qquad \hat{\mathcal{L}}_i= 
\left\{
\begin{aligned}
    &\frac{1}{2\sigma_1^2}\mathcal{L}_i+\log\sigma_i \qquad \text{for regression task}\\ 
    &\frac{1}{\sigma_2^2}\mathcal{L}_i+\log\sigma_i \qquad \text{for classification task}
\end{aligned}
\right\}
\end{equation}

\item The losses are dynamically weighted by descending rate of loss \citep{RN26} which is called Dynamic Weight Average (DWA). The weight of task $w_i$ is defined as follows with DWA:
\begin{equation}
\begin{aligned}
    w_i(t) = \frac{\mathcal{K}\exp(w_i(t-1)/T)}{\sum_{i=1}^{\mathcal{K}}exp(w_i(t-1)/T)}
    \qquad
    where \quad
    w_i(t-1) = \frac{\mathcal{L}_{k}(t-1)}{\mathcal{L}_{k}(t-2)}
\end{aligned}
\end{equation}
where $t$ is an iteration index and $\mathcal{K}$ is the number of tasks. $T$ represents the temperature parameter governing the softness of task weighting. As $T$ increases, the tasks become likely to be weighted equally. We used $T=2$ for our experiments following the works in \citep{RN26}.

\end{enumerate}

\section{Experimental Details}
\label{Append:exp_details}
\noindent\textbf{Implementation details.} To train MTI-Net \citep{RN32} on both NYUD-v2 and PASCAL-Context, we adopted the loss schema and augmentation strategy from PAD-Net\citep{RN29} and MTI-Net\citep{RN32}. For depth estimation, we utilized L1 loss, while the cross-entropy loss was used for semantic segmentation. To train for saliency estimation and edge detection, we employed the well-known balanced cross-entropy loss. Surface normal prediction used L1 loss. We augmented input images by randomly scaling them with a ratio from {1, 1.2, 1.5} and horizontally flipping them with a 50\% probability. The network was trained for 200 epochs for NYUD-v2 and 50 epochs for PASCAL-Context using the Adam optimizer. We employed a learning rate of $10^{-4}$ with a poly learning rate decay policy. We used a weight decay of $10^{-4}$ and batch size of $8$.

In contrast, for Cityscapes with SegNet \citep{badrinarayanan2017segnet}, we followed the experimental setting in \citep{RN18, guangyuan2022recon}. We used L1 loss and cross-entropy loss for depth estimation and semantic segmentation, respectively. The network was trained for 200 epochs using the Adam optimizer. We employed a learning rate of $5\times 10^{-5}$ with multi-step learning rate scheduling. We used a batch size of $8$.

\noindent\textbf{Evaluation metric.}
To evaluate the performance of tasks, we employed widely used metrics. For semantic segmentation, we utilized mean Intersection over Union (mIoU), Pixel Accuracy (PAcc), and mean Accuracy (mAcc). Surface normal prediction's performance was measured by calculating the mean and median angle distances between the predicted output and ground truth. We also used the proportion of pixels within the angles of $11.25^{\circ}$, $22.5^{\circ}$, $30^{\circ}$ to the ground truth, as suggested by \citep{eval_normal}. To evaluate the depth estimation task, we followed the methods proposed in \citep{eval_depth1, eval_depth2, eval_depth3}. We used Root Mean Squared Error (RMSE), and Mean Relative Error (abs\_rel).
For saliency estimation and human part segmentation, we employed mean Intersection over Union (mIoU).

\newpage
\section{Additional Experimental Results}
We compare GD, MGDA \citep{RN36}, PCGrad \citep{RN20}, CAGrad \citep{RN18}, Aligned-MTL \citep{senushkin2023independent}, and connection strength-based optimization on 4 different multi-task loss scaling methods mentioned in \cref{Append:loss}. We have summarized the experimental overview as follows.

\begin{enumerate}[leftmargin=*]
\item NYUD-v2 with HRNet-18 on various loss scaling is evaluated in \cref{tab:nyud_hrnet_equal,tab:nyud_hrnet_dwa,tab:nyud_hrnet_homosce}.
\item NYUD-v2 with ResNet-18 on various loss scaling is evaluated in \cref{tab:nyud_resnet_tuned,tab:nyud_resnet_equal,tab:nyud_resnet_dwa,tab:nyud_resnet_homosce}.
\item PASCAL-Context with HRNet-18 on various loss scaling is evaluated in \cref{tab:pascal_hrnet_equal,tab:pascal_hrnet_dwa,tab:pascal_hrnet_homosce}.
\end{enumerate}

\subsection{NYUD-v2 with HRNet-18}
\label{Append:nyud_hrnet18}
\begin{table}[H]
    \caption{The experimental results of different multi-task optimization methods on NYUD-v2 with HRNet-18. The losses of all tasks are evenly weighted. Experiments are repeated over 3 random seeds and average values are presented. $\triangle_m$ $\uparrow$(\%) is used to indicate the percentage improvement in multi-task performance (MTP). The best results are expressed in \textbf{bold} numbers.}
    \centering
    \footnotesize
    \renewcommand\arraystretch{1.20}
    \begin{tabular}{ll@{\enspace}ccl@{\enspace}cccl@{\enspace}cccccl@{\enspace}c}
        \toprule
        \multicolumn{1}{c}{Tasks} && \multicolumn{2}{c}{Depth} && \multicolumn{3}{c}{SemSeg} && \multicolumn{5}{c}{Surface Normal} && \\
        \cmidrule(r){1-1} \cmidrule(r){3-4} \cmidrule(r){6-8} \cmidrule(r){10-14}
        
        \multicolumn{1}{c}{\multirow{2}{*}{Method}}            && \multicolumn{2}{c}{\begin{tabular}[c]{@{}c@{}}Distance\\ (Lower Better)\end{tabular}} && \multicolumn{3}{c}{\begin{tabular}[c]{@{}c@{}}(\%)\\ (Higher Better)\end{tabular}}        && \multicolumn{2}{c}{\begin{tabular}[c]{@{}c@{}}Angle Distance\\ (Lower Better)\end{tabular}} & \multicolumn{3}{c}{\begin{tabular}[c]{@{}c@{}}Within t degree (\%)\\ (Higher Better)\end{tabular}} && MTP \\
        
        && rmse & abs\_rel && mIoU & PAcc & mAcc && mean & median & 11.25 & 22.5 & 30 && $\triangle_m$ $\uparrow$(\%) \\ \toprule

        Independent && 0.667 & 0.186 && 33.18 & 65.04 & 45.07 && 20.75 & 14.04 & 41.32 & 68.26 & 78.04 && + 0.00 \\ \midrule

        GD && 0.595 & 0.150 && 40.67 & 70.11 & 53.41 && 21.45 & 15.02 & 39.06 & 66.42 & 76.87 && + 10.00\\ 

        MGDA \citep{RN36} && 0.587 & 0.148 && 40.69 & 70.40 & 53.15 && 21.30 & 14.73 & 39.59 & 66.85 & 77.12 && + 10.66 \\ 
        
        PCGrad \citep{RN20} && 0.581 & 0.155 && 40.33 & 70.44 & 52.83 && 21.23 & 14.59 & 40.01 & 67.17 & 77.31 && + 10.71 \\ 
        
        CAGrad \citep{RN18} && \textbf{0.576} & 0.149 && 40.00 & 70.45 & 51.75 && 21.09 & 14.50 & 40.18 & 67.40 & 77.47 && + 10.85 \\ 

        Aligned-MTL \citep{senushkin2023independent} && 0.588 & 0.152 && 40.58 & 70.37 & 52.71 && 21.17 & 14.55 & 40.07 & 67.23 & 77.39 && + 10.71 \\

        \rowcolor{_gray}
        Ours && \textbf{0.576} & \textbf{0.143} && \textbf{41.20} & \textbf{71.03} & \textbf{53.76} &&  \textbf{20.42} & \textbf{13.75} & \textbf{42.20} & \textbf{69.22} & \textbf{78.88} && \textbf{+ 13.13} \\ \bottomrule
    \end{tabular}
    \label{tab:nyud_hrnet_equal}
\end{table}

\begin{table}[H]
    \caption{The experimental results of different multi-task optimization methods on NYUD-v2 with HRNet-18. The losses are weighted using Dynamic Weight Average (DWA). Experiments are repeated over 3 random seeds and average values are presented. $\triangle_m$ $\uparrow$(\%) is used to indicate the percentage improvement in multi-task performance (MTP). The best results are expressed in \textbf{bold} numbers.}
    \vspace{2pt}
    \centering
    \footnotesize
    \renewcommand\arraystretch{1.20}
    \begin{tabular}{ll@{\enspace}ccl@{\enspace}cccl@{\enspace}cccccl@{\enspace}c}
        \toprule
        \multicolumn{1}{c}{Tasks} && \multicolumn{2}{c}{Depth} && \multicolumn{3}{c}{SemSeg} && \multicolumn{5}{c}{Surface Normal} && \\
        \cmidrule(r){1-1} \cmidrule(r){3-4} \cmidrule(r){6-8} \cmidrule(r){10-14}
        
        \multicolumn{1}{c}{\multirow{2}{*}{Method}}            && \multicolumn{2}{c}{\begin{tabular}[c]{@{}c@{}}Distance\\ (Lower Better)\end{tabular}} && \multicolumn{3}{c}{\begin{tabular}[c]{@{}c@{}}(\%)\\ (Higher Better)\end{tabular}}        && \multicolumn{2}{c}{\begin{tabular}[c]{@{}c@{}}Angle Distance\\ (Lower Better)\end{tabular}} & \multicolumn{3}{c}{\begin{tabular}[c]{@{}c@{}}Within t degree (\%)\\ (Higher Better)\end{tabular}} && MTP \\
        
        && rmse & abs\_rel && mIoU & PAcc & mAcc && mean & median & 11.25 & 22.5 & 30 && $\triangle_m$ $\uparrow$(\%) \\ \toprule

        Independent && 0.667 & 0.186 && 33.18 & 65.04 & 45.07 && 20.75 & 14.04 & 41.32 & 68.26 & 78.04 && + 0.00 \\ \midrule

        GD && 0.592 & 0.146 && 40.86 & 70.19 & 53.01 && 21.15 & 14.52 & 40.20 & 67.36 & 77.48 && + 10.82\\ 

        MGDA \citep{RN36} && 0.593 & 0.147 && 40.46 & 70.10 & 52.83 && 21.30 & 14.68 & 39.73 & 66.90 & 77.16 && + 10.13 \\ 
        
        PCGrad \citep{RN20} && 0.593 & 0.147 && 40.34 & 70.00 & 52.37 && 21.36 & 14.77 & 39.57 & 66.78 & 77.07 && + 9.91 \\ 
        
        CAGrad \citep{RN18} && 0.576 & 0.146 && 40.52 & 70.23 & 52.73 && 21.09 & 14.59 & 40.18 & 67.40 & 77.49 && + 11.38 \\ 

        Aligned-MTL \citep{senushkin2023independent} && 0.590 & 0.147 && 40.43 & 70.09 & 52.66 && 21.18 & 14.61 & 39.98 & 67.21 & 77.39 && + 10.44 \\

        \rowcolor{_gray}
        Ours && \textbf{0.565} & \textbf{0.141} && \textbf{41.64} & \textbf{70.97} & \textbf{54.49} &&  \textbf{20.35} & \textbf{13.48} & \textbf{43.04} & \textbf{69.60} & \textbf{78.95} && \textbf{+ 14.24} \\ \bottomrule
    \end{tabular}
    \label{tab:nyud_hrnet_dwa}
\end{table}

\begin{table}[H]
    \caption{The experimental results of different multi-task optimization methods on NYUD-v2 with HRNet-18. The losses are weighted by homoscedastic uncertainty. Experiments are repeated over 3 random seeds and average values are presented. $\triangle_m$ $\uparrow$(\%) is used to indicate the percentage improvement in multi-task performance (MTP). The best results are expressed in \textbf{bold} numbers.}
    \centering
    \footnotesize
    \renewcommand\arraystretch{1.20}
    \begin{tabular}{ll@{\enspace}ccl@{\enspace}cccl@{\enspace}cccccl@{\enspace}c}
        \toprule
        \multicolumn{1}{c}{Tasks} && \multicolumn{2}{c}{Depth} && \multicolumn{3}{c}{SemSeg} && \multicolumn{5}{c}{Surface Normal} && \\
        \cmidrule(r){1-1} \cmidrule(r){3-4} \cmidrule(r){6-8} \cmidrule(r){10-14}
        
        \multicolumn{1}{c}{\multirow{2}{*}{Method}}            && \multicolumn{2}{c}{\begin{tabular}[c]{@{}c@{}}Distance\\ (Lower Better)\end{tabular}} && \multicolumn{3}{c}{\begin{tabular}[c]{@{}c@{}}(\%)\\ (Higher Better)\end{tabular}}        && \multicolumn{2}{c}{\begin{tabular}[c]{@{}c@{}}Angle Distance\\ (Lower Better)\end{tabular}} & \multicolumn{3}{c}{\begin{tabular}[c]{@{}c@{}}Within t degree (\%)\\ (Higher Better)\end{tabular}} && MTP \\
        
        && rmse & abs\_rel && mIoU & PAcc & mAcc && mean & median & 11.25 & 22.5 & 30 && $\triangle_m$ $\uparrow$(\%) \\ \toprule

        Independent && 0.667 & 0.186 && 33.18 & 65.04 & 45.07 && 20.75 & 14.04 & 41.32 & 68.26 & 78.04 && + 0.00 \\ \midrule

        GD && 0.589 & 0.148 && 39.93 & 70.15 & 51.99 && 21.13 & 14.46 & 40.47 & 67.28 & 77.38 && + 9.87\\ 

        MGDA \citep{RN36} && 0.590 & 0.148 && 39.78 & 69.77 & 51.80 && 21.24 & 14.69 & 39.78 & 66.94 & 77.22 && + 9.69 \\ 
        
        PCGrad \citep{RN20} && 0.587 & 0.147 && 40.56 & 69.97 & 53.07 && 21.19 & 14.40 & 40.51 & 67.46 & 77.41 && + 10.71 \\ 
        
        CAGrad \citep{RN18} && 0.583 & 0.147 && 40.23 & 70.06 & 52.74 && 21.09 & 14.47 & 40.23 & 67.48 & 77.50 && + 10.73 \\ 

        Aligned-MTL \citep{senushkin2023independent} && 0.589 & 0.147 && 40.08 & 69.91 & 52.23 && 21.15 & 14.47 & 10.19 & 67.45 & 77.45 && + 10.17 \\

        \rowcolor{_gray}
        Ours && \textbf{0.569} & \textbf{0.140} && \textbf{41.16} & \textbf{70.83} & \textbf{53.65} &&  \textbf{20.19} & \textbf{13.39} & \textbf{43.33} & \textbf{70.07} & \textbf{79.30} && \textbf{+ 13.81} \\ \bottomrule
    \end{tabular}
    \label{tab:nyud_hrnet_homosce}
\end{table}

\subsection{NYUD-v2 with ResNet-18}
\label{Append:nyud_resnet18}

\begin{table}[H]
    \caption{The experimental results of different multi-task optimization methods on NYUD-v2 with ResNet-18. The losses of all tasks are evenly weighted. Experiments are repeated over 3 random seeds and average values are presented. $\triangle_m$ $\uparrow$(\%) is used to indicate the percentage improvement in multi-task performance (MTP). The best results are expressed in \textbf{bold} numbers.}
    \vspace{2pt}
    \centering
    \footnotesize
    \renewcommand\arraystretch{1.20}
    \begin{tabular}{ll@{\enspace}ccl@{\enspace}cccl@{\enspace}cccccl@{\enspace}c}
        \toprule
        \multicolumn{1}{c}{Tasks} && \multicolumn{2}{c}{Depth} && \multicolumn{3}{c}{SemSeg} && \multicolumn{5}{c}{Surface Normal} && \\
        \cmidrule(r){1-1} \cmidrule(r){3-4} \cmidrule(r){6-8} \cmidrule(r){10-14}
        
        \multicolumn{1}{c}{\multirow{2}{*}{Method}}            && \multicolumn{2}{c}{\begin{tabular}[c]{@{}c@{}}Distance\\ (Lower Better)\end{tabular}} && \multicolumn{3}{c}{\begin{tabular}[c]{@{}c@{}}(\%)\\ (Higher Better)\end{tabular}}        && \multicolumn{2}{c}{\begin{tabular}[c]{@{}c@{}}Angle Distance\\ (Lower Better)\end{tabular}} & \multicolumn{3}{c}{\begin{tabular}[c]{@{}c@{}}Within t degree (\%)\\ (Higher Better)\end{tabular}} && MTP \\
        
        && rmse & abs\_rel && mIoU & PAcc & mAcc && mean & median & 11.25 & 22.5 & 30 && $\triangle_m$ $\uparrow$(\%) \\ \toprule

        Independent && 0.659 & 0.183 && 34.46 & 65.51 & 46.50 && 23.36 & 16.67 & 34.89 & 62.19 & 73.33 && + 0.00 \\ \midrule

        GD && 0.613 & \textbf{0.160} && 38.54 & 68.89 & 51.04 && 22.09 & 15.35 & 38.29 & 65.12 & 75.61 && + 8.09\\ 

        MGDA \citep{RN36} && 0.616 & 0.165 && \textbf{39.49} & 69.30 & \textbf{52.30} && 22.52 & 15.61 & 37.92 & 64.25 & 74.77 && + 8.24 \\ 
        
        PCGrad \citep{RN20} && 0.618 & 0.164 && 38.76 & 69.01 & 51.12 && 22.05 & 15.28 & 38.55 & 65.36 & 75.77 && + 8.10 \\ 
        
        CAGrad \citep{RN18} && 0.610 & \textbf{0.160} && 39.20 & \textbf{69.38} & 51.58 && 22.18 & 15.61 & 37.65 & 64.70 & 75.42 && + 8.75 \\ 

        Aligned-MTL \citep{senushkin2023independent} && 0.612 & 0.161 && 39.35 & 69.21 & 51.80 && 22.34 & 15.47 & 38.12 & 64.83 & 75.61 && + 8.56 \\

        \rowcolor{_gray}
        Ours && \textbf{0.601} & 0.162 && 38.30 & 68.78 & 51.01 &&  \textbf{21.09} & \textbf{14.31} & \textbf{40.95} & \textbf{67.57} & \textbf{77.50} && \textbf{+ 9.89} \\ \bottomrule
    \end{tabular}
    \label{tab:nyud_resnet_equal}
\end{table}

\begin{table}[H]
    \caption{The experimental results of different multi-task optimization methods on NYUD-v2 with ResNet-18. The weights of tasks are manually tuned. Experiments are repeated over 3 random seeds and average values are presented. $\triangle_m$ $\uparrow$(\%) is used to indicate the percentage improvement in multi-task performance (MTP). The best results are expressed in \textbf{bold} numbers.}
    \centering
    \footnotesize
    \renewcommand\arraystretch{1.20}
    \begin{tabular}{ll@{\enspace}ccl@{\enspace}cccl@{\enspace}cccccl@{\enspace}c}
        \toprule
        \multicolumn{1}{c}{Tasks} && \multicolumn{2}{c}{Depth} && \multicolumn{3}{c}{SemSeg} && \multicolumn{5}{c}{Surface Normal} && \\
        \cmidrule(r){1-1} \cmidrule(r){3-4} \cmidrule(r){6-8} \cmidrule(r){10-14}
        
        \multicolumn{1}{c}{\multirow{2}{*}{Method}}            && \multicolumn{2}{c}{\begin{tabular}[c]{@{}c@{}}Distance\\ (Lower Better)\end{tabular}} && \multicolumn{3}{c}{\begin{tabular}[c]{@{}c@{}}(\%)\\ (Higher Better)\end{tabular}}        && \multicolumn{2}{c}{\begin{tabular}[c]{@{}c@{}}Angle Distance\\ (Lower Better)\end{tabular}} & \multicolumn{3}{c}{\begin{tabular}[c]{@{}c@{}}Within t degree (\%)\\ (Higher Better)\end{tabular}} && MTP \\
        
        && rmse & abs\_rel && mIoU & PAcc & mAcc && mean & median & 11.25 & 22.5 & 30 && $\triangle_m$ $\uparrow$(\%) \\ \toprule

        Independent && 0.659 & 0.183 && 34.46 & 65.51 & 46.50 && 23.36 & 16.67 & 34.89 & 62.19 & 73.33 && + 0.00 \\ \midrule

        GD && 0.622 & 0.163 && 38.07 & 68.31 & 50.84 && 21.49 & 14.63 & 40.04 & 66.87 & 76.87 && + 8.03\\ 

        MGDA \citep{RN36} && 0.635 & 0.166 && 38.18 & 68.22 & 49.70 && 22.07 & 15.01 & 39.11 & 65.81 & 75.90 && + 6.65 \\ 
        
        PCGrad \citep{RN20} && 0.617 & 0.165 && 37.80 & 67.94 & 50.00 && 21.52 & 14.53 & 40.27 & 66.91 & 76.71 && + 7.98\\ 
        
        CAGrad \citep{RN18} && 0.620 & 0.163 && 37.02 & 67.96 & 49.71 && 21.67 & 14.80 & 39.55 & 66.46 & 76.56 && + 6.86 \\ 

        Aligned-MTL \citep{senushkin2023independent} && 0.625 & 0.166 && 38.01 & 68.12 & 50.43 && 21.62 & 14.75 & 39.62 & 66.58 & 76.68 && + 7.64 \\

        \rowcolor{_gray}
        Ours && \textbf{0.600} & \textbf{0.157} && \textbf{39.00} & \textbf{69.02} & \textbf{51.11} &&  \textbf{20.65} & \textbf{13.77} & \textbf{42.78} & \textbf{68.97} & \textbf{78.30} && \textbf{+ 11.24} \\ \bottomrule
    \end{tabular}
    \label{tab:nyud_resnet_tuned}
\end{table}

\begin{table}[H]
    \caption{The experimental results of different multi-task optimization methods on NYUD-v2 with ResNet-18. The losses are weighted using Dynamic Weight Average (DWA). Experiments are repeated over 3 random seeds and average values are presented. $\triangle_m$ $\uparrow$(\%) is used to indicate the percentage improvement in multi-task performance (MTP). The best results are expressed in \textbf{bold} numbers.}
    \centering
    \footnotesize
    \renewcommand\arraystretch{1.20}
    \begin{tabular}{ll@{\enspace}ccl@{\enspace}cccl@{\enspace}cccccl@{\enspace}c}
        \toprule
        \multicolumn{1}{c}{Tasks} && \multicolumn{2}{c}{Depth} && \multicolumn{3}{c}{SemSeg} && \multicolumn{5}{c}{Surface Normal} && \\
        \cmidrule(r){1-1} \cmidrule(r){3-4} \cmidrule(r){6-8} \cmidrule(r){10-14}
        
        \multicolumn{1}{c}{\multirow{2}{*}{Method}}            && \multicolumn{2}{c}{\begin{tabular}[c]{@{}c@{}}Distance\\ (Lower Better)\end{tabular}} && \multicolumn{3}{c}{\begin{tabular}[c]{@{}c@{}}(\%)\\ (Higher Better)\end{tabular}}        && \multicolumn{2}{c}{\begin{tabular}[c]{@{}c@{}}Angle Distance\\ (Lower Better)\end{tabular}} & \multicolumn{3}{c}{\begin{tabular}[c]{@{}c@{}}Within t degree (\%)\\ (Higher Better)\end{tabular}} && MTP \\
        
        && rmse & abs\_rel && mIoU & PAcc & mAcc && mean & median & 11.25 & 22.5 & 30 && $\triangle_m$ $\uparrow$(\%) \\ \toprule

        Independent && 0.659 & 0.183 && 34.46 & 65.51 & 46.50 && 23.36 & 16.67 & 34.89 & 62.19 & 73.33 && + 0.00 \\ \midrule

        GD && 0.607 & 0.159 && 38.65 & 68.99 & 51.72 && 22.17 & 15.52 & 38.51 & 65.11 & 75.47 && + 8.38 \\ 

        MGDA \citep{RN36} && 0.616 & 0.165 && 39.38 & 69.18 & 51.78 && 22.53 & 15.69 & 37.68 & 64.12 & 74.67 && + 8.12 \\ 
        
        PCGrad \citep{RN20} && 0.612 & 0.162 && 38.56 & 68.97 & 51.16 && 22.11 & 15.40 & 38.20 & 65.07 & 75.58 && + 8.13 \\ 
        
        CAGrad \citep{RN18} && 0.609 & 0.157 && \textbf{39.40} & \textbf{69.30} & \textbf{51.84} && 22.28 & 15.68 & 37.62 & 64.46 & 75.24 && + 8.85 \\ 

        Aligned-MTL \citep{senushkin2023independent} && 0.609 & 0.161 && 39.22 & 69.04 & 69.01 && 22.15 & 15.48 & 38.30 & 65.08 & 75.52 && + 8.86 \\

        \rowcolor{_gray}
        Ours && \textbf{0.592} & \textbf{0.148} && 38.41 & 68.82 & 51.15 &&  \textbf{20.96} & \textbf{14.25} & \textbf{40.97} & \textbf{67.59} & \textbf{77.10} && \textbf{+ 10.63} \\ \bottomrule
    \end{tabular}
    \label{tab:nyud_resnet_dwa}
\end{table}

\begin{table}[H]
    \caption{The experimental results of different multi-task optimization methods on NYUD-v2 with ResNet-18. The losses are weighted by homoscedastic uncertainty. Experiments are repeated over 3 random seeds and average values are presented. $\triangle_m$ $\uparrow$(\%) is used to indicate the percentage improvement in multi-task performance (MTP). The best results are expressed in \textbf{bold} numbers.}
    \centering
    \footnotesize
    \renewcommand\arraystretch{1.20}
    \begin{tabular}{ll@{\enspace}ccl@{\enspace}cccl@{\enspace}cccccl@{\enspace}c}
        \toprule
        \multicolumn{1}{c}{Tasks} && \multicolumn{2}{c}{Depth} && \multicolumn{3}{c}{SemSeg} && \multicolumn{5}{c}{Surface Normal} && \\
        \cmidrule(r){1-1} \cmidrule(r){3-4} \cmidrule(r){6-8} \cmidrule(r){10-14}
        
        \multicolumn{1}{c}{\multirow{2}{*}{Method}}            && \multicolumn{2}{c}{\begin{tabular}[c]{@{}c@{}}Distance\\ (Lower Better)\end{tabular}} && \multicolumn{3}{c}{\begin{tabular}[c]{@{}c@{}}(\%)\\ (Higher Better)\end{tabular}}        && \multicolumn{2}{c}{\begin{tabular}[c]{@{}c@{}}Angle Distance\\ (Lower Better)\end{tabular}} & \multicolumn{3}{c}{\begin{tabular}[c]{@{}c@{}}Within t degree (\%)\\ (Higher Better)\end{tabular}} && MTP \\
        
        && rmse & abs\_rel && mIoU & PAcc & mAcc && mean & median & 11.25 & 22.5 & 30 && $\triangle_m$ $\uparrow$(\%) \\ \toprule

        Independent && 0.659 & 0.183 && 34.46 & 65.51 & 46.50 && 23.36 & 16.67 & 34.89 & 62.19 & 73.33 && + 0.00 \\ \midrule

        GD && 0.608 & 0.158 && 39.02 & 69.29 & 51.48 && 22.06 & 15.47 & 37.98 & 65.01 & 75.68 && + 8.85\\ 

        MGDA \citep{RN36} && 0.623 & 0.162 && \textbf{39.43} & \textbf{69.30} & \textbf{51.79} && 22.65 & 15.77 & 37.39 & 64.03 & 74.66 && + 7.64 \\ 
        
        PCGrad \citep{RN20} && 0.606 & 0.158 && 39.40 & 69.25 & 51.68 && 22.25 & 15.43 & 38.05 & 64.81 & 75.35 && + 9.04 \\ 
        
        CAGrad \citep{RN18} && 0.600 & 0.156 && 38.62 & 68.74 & 51.03 && 22.27 & 15.43 & 38.11 & 64.85 & 75.32 && + 8.56 \\ 

        Aligned-MTL \citep{senushkin2023independent} && 0.605 & 0.158 && 39.10 & 69.23 & 51.56 && 22.13 & 15.49 & 37.77 & 64.89 & 75.51 && + 8.97 \\

        \rowcolor{_gray}
        Ours && \textbf{0.595} & \textbf{0.153} && 38.67 & 69.01 & 51.01 &&  \textbf{21.05} & \textbf{14.11} & \textbf{41.43} & \textbf{67.91} & \textbf{77.59} && \textbf{+ 10.61} \\ \bottomrule
    \end{tabular}
    \label{tab:nyud_resnet_homosce}
\end{table}

\subsection{PASCAL-Context with HRNet-18}
\label{Append:pascal_hrnet18}
\begin{table}[H]
    \caption{The experimental results of different multi-task optimization methods on PASCAL-Context dataset with HRNet-18. The losses of all tasks are evenly weighted. Experiments are repeated over 3 random seeds and average values are presented. $\triangle_m$ $\uparrow$(\%) is used to indicate the percentage improvement in multi-task performance (MTP). The best results are expressed in \textbf{bold} numbers.}
    \centering
    \footnotesize
    \renewcommand\arraystretch{1.20}
    \begin{tabular}{l@{\hspace{4pt}}l@{\hspace{4pt}}cc@{\hspace{2pt}}l@{\hspace{2pt}}c@{\hspace{4pt}}l@{\hspace{2pt}}cc@{\hspace{2pt}}l@{\hspace{4pt}}ccccc@{\hspace{6pt}}l@{\hspace{4pt}}c}
        \toprule
        \multicolumn{1}{c}{Tasks} && \multicolumn{2}{c}{SemSeg} && \multicolumn{1}{c}{PartSeg} && \multicolumn{2}{c}{Saliency} && \multicolumn{5}{c}{Surface Normal} \\
        \cmidrule(l){1-1} \cmidrule(r){3-4} \cmidrule(r){6-6} \cmidrule(r){8-9} \cmidrule(r){11-15}
        
        \multicolumn{1}{c}{\multirow{2}{*}{Method}}            && \multicolumn{2}{c}{\begin{tabular}[c]{@{}c@{}}(Higher Better)\end{tabular}} && \multicolumn{1}{c}{\begin{tabular}[c]{@{}c@{}}(Higher Better)\end{tabular}}        && \multicolumn{2}{c}{\begin{tabular}[c]{@{}c@{}} (Higher Better)\end{tabular}} && \multicolumn{2}{c}{\begin{tabular}[c]{@{}c@{}}Angle Distance\\ (Lower Better)\end{tabular}} & \multicolumn{3}{c}{\begin{tabular}[c]{@{}c@{}}Within t degree (\%)\\ (Higher Better)\end{tabular}} && MTP \\
        
        && mIoU & PAcc && mIoU && mIoU & maxF && mean & median & 11.25 & 22.5 & 30 && $\triangle_m$ $\uparrow$(\%) \\ \toprule

        Independent && 60.30 & 89.88 && 60.56 && 67.05 & 78.98 && 14.76 & 11.92 & 47.61 & 81.02 & 90.65 && + 0.00 \\ \midrule

        GD && 61.65 & 90.14 && 58.35 && 65.80 & 78.07 && 16.71 & 13.82 & 39.70 & 75.18 & 87.17 && - 4.12 \\ 

        MGDA \citep{RN36} && \textbf{63.52} & \textbf{90.68} && 60.38 && 64.99 & 77.57 && 17.00 & 14.13 & 38.58 & 74.47 & 86.77 && - 3.30  \\ 
        
        PCGrad \citep{RN20} && 63.21 & 90.33 && 60.42 && 64.77 & 77.48 && 16.65 & 13.71 & 39.64 & 75.10 & 87.07 && - 2.90  \\ 
        
        CAGrad \citep{RN18} && 63.44 & 90.53 && 60.11 && 64.83 & 77.52 && 16.92 & 13.98 & 39.03 & 75.01 & 86.92 && - 3.37 \\ 

        Aligned-MTL \citep{senushkin2023independent} && 62.38 & 90.31 && 60.36 && 65.68 & 79.92 && 16.73 & 13.88 & 39.68 & 75.18 & 87.10 && - 3.07 \\

        \rowcolor{_gray}
        Ours && 62.64 & 90.39 && \textbf{61.42} && \textbf{67.10} & \textbf{78.91} &&  \textbf{15.58} & \textbf{12.68} & \textbf{43.93} & \textbf{78.69} & \textbf{89.26} && \textbf{- 0.05} \\ \bottomrule
    \end{tabular}
    \label{tab:pascal_hrnet_equal}
\end{table}

\begin{table}[H]
    \caption{The experimental results of different multi-task optimization methods on PASCAL-Context dataset with HRNet-18. The losses are weighted using Dynamic Weight Average (DWA). Experiments are repeated over 3 random seeds and average values are presented. $\triangle_m$ $\uparrow$(\%) is used to indicate the percentage improvement in multi-task performance (MTP). The best results are expressed in \textbf{bold} numbers.}
    \centering
    \footnotesize
    \renewcommand\arraystretch{1.20}
    \begin{tabular}{l@{\hspace{4pt}}l@{\hspace{4pt}}cc@{\hspace{2pt}}l@{\hspace{2pt}}c@{\hspace{4pt}}l@{\hspace{2pt}}cc@{\hspace{2pt}}l@{\hspace{4pt}}ccccc@{\hspace{6pt}}l@{\hspace{4pt}}c}
        \toprule
        \multicolumn{1}{c}{Tasks} && \multicolumn{2}{c}{SemSeg} && \multicolumn{1}{c}{PartSeg} && \multicolumn{2}{c}{Saliency} && \multicolumn{5}{c}{Surface Normal} \\
        \cmidrule(l){1-1} \cmidrule(r){3-4} \cmidrule(r){6-6} \cmidrule(r){8-9} \cmidrule(r){11-15}
        
        \multicolumn{1}{c}{\multirow{2}{*}{Method}}            && \multicolumn{2}{c}{\begin{tabular}[c]{@{}c@{}}(Higher Better)\end{tabular}} && \multicolumn{1}{c}{\begin{tabular}[c]{@{}c@{}}(Higher Better)\end{tabular}}        && \multicolumn{2}{c}{\begin{tabular}[c]{@{}c@{}} (Higher Better)\end{tabular}} && \multicolumn{2}{c}{\begin{tabular}[c]{@{}c@{}}Angle Distance\\ (Lower Better)\end{tabular}} & \multicolumn{3}{c}{\begin{tabular}[c]{@{}c@{}}Within t degree (\%)\\ (Higher Better)\end{tabular}} && MTP \\
        
        && mIoU & PAcc && mIoU && mIoU & maxF && mean & median & 11.25 & 22.5 & 30 && $\triangle_m$ $\uparrow$(\%) \\ \toprule

        Independent && 60.30 & 89.88 && 60.56 && 67.05 & 78.98 && 14.76 & 11.92 & 47.61 & 81.02 & 90.65 && + 0.00 \\ \midrule

        GD && \textbf{64.70} & \textbf{91.18} && 60.60 && 66.54 & 78.18 && 15.13 & 12.23 & 45.77 & 79.91 & 89.96 && + 1.02\\ 

        MGDA \citep{RN36} && 64.56 & 90.72 && 60.69 && 65.93 & 77.37 && 16.87 & 13.95 & 39.35 & 74.69 & 86.82 && - 2.17  \\ 
        
        PCGrad \citep{RN20} && 64.35 & 90.98 && 60.99 && 66.12 & 77.65 && 15.92 & 13.11 & 41.98 & 76.21 & 88.03 && - 0.45  \\ 
        
        CAGrad \citep{RN18} && 64.03 & 90.77 && 60.62 && 66.01 & 77.42 && 16.63 & 13.86 & 40.02 & 75.22 & 87.41 && - 1,98  \\ 

        Aligned-MTL \citep{senushkin2023independent} && 64.41 & 91.00 && 60.77 && 66.09 & 77.51 && 16.22 & 13.48 & 42.26 & 76.92 & 88.66 && - 1.04 \\

        \rowcolor{_gray}
        Ours && 63.89 & 90.73 && \textbf{61.89} && \textbf{67.39 }& \textbf{79.08} &&  \textbf{14.94} & \textbf{12.10} & \textbf{46.27} & \textbf{80.57} & \textbf{90.41} && \textbf{+ 1.86} \\ \bottomrule
    \end{tabular}
    \label{tab:pascal_hrnet_dwa}
\end{table}

\begin{table}[H]
    \caption{The experimental results of different multi-task optimization methods on PASCAL-Context dataset with HRNet-18. The losses are weighted by homoscedastic uncertainty. Experiments are repeated over 3 random seeds and average values are presented. $\triangle_m$ $\uparrow$(\%) is used to indicate the percentage improvement in multi-task performance (MTP). The best results are expressed in \textbf{bold} numbers.}
    \centering
    \footnotesize
    \renewcommand\arraystretch{1.20}
    \begin{tabular}{l@{\hspace{4pt}}l@{\hspace{4pt}}cc@{\hspace{2pt}}l@{\hspace{2pt}}c@{\hspace{4pt}}l@{\hspace{2pt}}cc@{\hspace{2pt}}l@{\hspace{4pt}}ccccc@{\hspace{6pt}}l@{\hspace{4pt}}c}
        \toprule
        \multicolumn{1}{c}{Tasks} && \multicolumn{2}{c}{SemSeg} && \multicolumn{1}{c}{PartSeg} && \multicolumn{2}{c}{Saliency} && \multicolumn{5}{c}{Surface Normal} \\
        \cmidrule(l){1-1} \cmidrule(r){3-4} \cmidrule(r){6-6} \cmidrule(r){8-9} \cmidrule(r){11-15}
        
        \multicolumn{1}{c}{\multirow{2}{*}{Method}}            && \multicolumn{2}{c}{\begin{tabular}[c]{@{}c@{}}(Higher Better)\end{tabular}} && \multicolumn{1}{c}{\begin{tabular}[c]{@{}c@{}}(Higher Better)\end{tabular}}        && \multicolumn{2}{c}{\begin{tabular}[c]{@{}c@{}} (Higher Better)\end{tabular}} && \multicolumn{2}{c}{\begin{tabular}[c]{@{}c@{}}Angle Distance\\ (Lower Better)\end{tabular}} & \multicolumn{3}{c}{\begin{tabular}[c]{@{}c@{}}Within t degree (\%)\\ (Higher Better)\end{tabular}} && MTP \\
        
        && mIoU & PAcc && mIoU && mIoU & maxF && mean & median & 11.25 & 22.5 & 30 && $\triangle_m$ $\uparrow$(\%) \\ \toprule

        Independent && 60.30 & 89.88 && 60.56 && 67.05 & 78.98 && 14.76 & 11.92 & 47.61 & 81.02 & 90.65 && + 0.00 \\ \midrule

        GD && 64.40 & 91.05 && 62.28 && 68.13 & 79.64. && 14.95 & 12.14 & 46.19 & 80.36 & 90.34 && + 2.49\\ 

        MGDA \citep{RN36} && 64.04 & 90.88 && 61.18 && 67.65 & 79.23 && 15.02 & 12.20 & 45.93 & 80.02 & 90.11 && + 1.59 \\ 
        
        PCGrad \citep{RN20} && \textbf{64.75} & \textbf{91.11} && \textbf{62.41} && 68.16 & 79.65 && 14.86 & 11.93 & 47.03 & 80.60 & 90.31 && + 2.85 \\ 
        
        CAGrad \citep{RN18} && 64.01 & 90.77 && 61.32 && 67.55 & 79.01 && 15.08 & 12.31 & 45.87 & 79.98 & 90.05 && + 1.50 \\ 

        Aligned-MTL \citep{senushkin2023independent} && 64.48 & 91.09 && 62.23 && 67.61 & 79.18 && 15.01 & 12.11 & 46.01 & 80.17 & 90.20 && + 2.21 \\

        \rowcolor{_gray}
        Ours && 64.01& 90.70 && 61.78 && \textbf{68.32} & \textbf{81.50} &&  \textbf{14.53} & \textbf{11.52} & \textbf{48.21} & \textbf{81.88} & \textbf{90.74} && \textbf{+ 2.90} \\ \bottomrule
    \end{tabular}
    \label{tab:pascal_hrnet_homosce}
\end{table}

\newpage
\section{Additional Ablation Studies}
\label{Append:ablation}
\textbf{The order of updating tasks in Phase 1 has little impact on multi-task performance.} To learn task priority in shared parameters, Phase 1 updates each task-specific gradient one by one sequentially. To determine the influence of the order of tasks on optimization, we randomly chose 5 sequences of tasks and showed their performance in \cref{tab:nyud_hrnet_random_sequence}. From the results, we can see that the order of updating tasks in Phase 1 does not have a significant impact on multi-task performance.

\begin{table}[h]
    \caption{The experimental results for NYUD-v2 with HRNet-18 involved exploring different task sequence orders in Phase 1. We conducted ablation experiments with five randomly selected task sequences. Each task was represented by a single alphabet letter, as follows: S for semantic segmentation, D for depth estimation, E for edge detection, and N for surface normal estimation.}
    \centering
    \footnotesize
    \renewcommand\arraystretch{1.25}
    \begin{tabular}{ll@{\enspace}ccl@{\enspace}cccl@{\enspace}cccccl@{\enspace}c}
        \toprule
        \multicolumn{1}{c}{Tasks} && \multicolumn{2}{c}{Depth} && \multicolumn{3}{c}{SemSeg} && \multicolumn{5}{c}{Surface Normal} && \\
        \cmidrule(r){1-1} \cmidrule(r){3-4} \cmidrule(r){6-8} \cmidrule(r){10-14}
        
        \multicolumn{1}{c}{\multirow{2}{*}{Method}}            && \multicolumn{2}{c}{\begin{tabular}[c]{@{}c@{}}Distance\\ (Lower Better)\end{tabular}} && \multicolumn{3}{c}{\begin{tabular}[c]{@{}c@{}}(\%)\\ (Higher Better)\end{tabular}}        && \multicolumn{2}{c}{\begin{tabular}[c]{@{}c@{}}Angle Distance\\ (Lower Better)\end{tabular}} & \multicolumn{3}{c}{\begin{tabular}[c]{@{}c@{}}Within t degree (\%)\\ (Higher Better)\end{tabular}} && MTP \\
        
        && rmse & abs\_rel && mIoU & PAcc & mAcc && mean & median & 11.25 & 22.5 & 30 && $\triangle_m$ $\uparrow$(\%) \\ \toprule

        Independent && 0.667 & 0.186 && 33.18 & 65.04 & 45.07 && 20.75 & 14.04 & 41.32 & 68.26 & 78.04 && + 0.00 \\ \midrule
        
        N-D-S-E && 0.574 & 0.157 && 41.12 & 70.44 & 53.77 && 19.60 & 12.52 & 46.01 & 71.33 & 80.02 && + 14.47\\
        D-S-N-E && 0.568 & 0.153 && 40.92 & 70.23 & 53.56 && 19.55 & 12.47 & 46.09 & 71.50 & 80.12 && + 14.65\\ 
        E-D-S-N && 0.568 & 0.150 && 40.97 & 70.22 & 53.59 && 19.58 & 12.50 & 46.08 & 71.44 & 80.07 && + 14.65 \\
        D-N-E-S && 0.571 & 0.153 && 41.03 & 70.31 & 53.68 && 19.49 & 12.44 & 46.17 & 71.58 & 80.17 && + 14.71 \\ 
        S-D-E-N && 0.565 & 0.148 && 41.10 & 70.37 & 53.74 &&  19.54 & 12.45 & 46.11 & 71.54 & 80.12 && + 15.00 \\ \bottomrule
    \end{tabular}
    \label{tab:nyud_hrnet_random_sequence}
\end{table}

\textbf{Our method demands the least computational load when compared to previous optimization methods.} In \cref{tab:training_time}, we show the impact of the proposed optimization on training time. The training time for each method is measured in seconds per epoch. To ensure a fair comparison, all methods were evaluated using the same architecture, guaranteeing an equal number of parameters and memory usage. The majority of the computational burden is concentrated on the forward pass, backpropagation, and gradient manipulation. While all optimization methods follow a similar process in the forward pass and backpropagation, the primary distinction arises from gradient manipulation. In Phase 1, no gradient manipulation is required, resulting in the shortest time consumption. In phase 2, it still exhibits the shortest training time compared to previous optimization methods. Unlike these previous methods that handle all shared components of the network, Phase 2 specifically targets the shared convolutional layer along with the task-specific batch normalization layer. This selective focus significantly reduces the time consumed per epoch.

\begin{table}[h!]
\small
\caption{Training time comparison for different multi-task optimization methods on NYUD-v2 with HRNet18.}
\centering
\begin{tabular}{c||c|c|c|c||c|c}
\hline
Method   & MGDA\citep{RN36}   & PCGrad\citep{RN20} & CAGrad\citep{RN18} & Aligned-MTL \citep{senushkin2023independent} & Phase 1 & Phase 2 \\ \hline
Time (s) & 363.98 & 421.48 & 378.12 & 811.57 & 296.74 & 331.53\\ \hline
\end{tabular}
\label{tab:training_time}
\end{table}

\textbf{The speed of learning the task priority differs based on the convolutional layer's position.} Phase 1 establishes the task priority during the initial stages of the network's optimization. Meanwhile, Phase 2 maintains this learned task priority, ensuring robust learning even when the loss for each task fluctuates. However, The timing at which task priority stabilizes varies based on the position of the convolutional layer within the network, as illustrated in \cref{fig:conn_pos}. This may suggest that optimizing by wholly separating each phase could be inefficient.

\begin{figure}[h]
\centering
\begin{subfigure}{.45\textwidth}
    \includegraphics[width=.99\columnwidth]{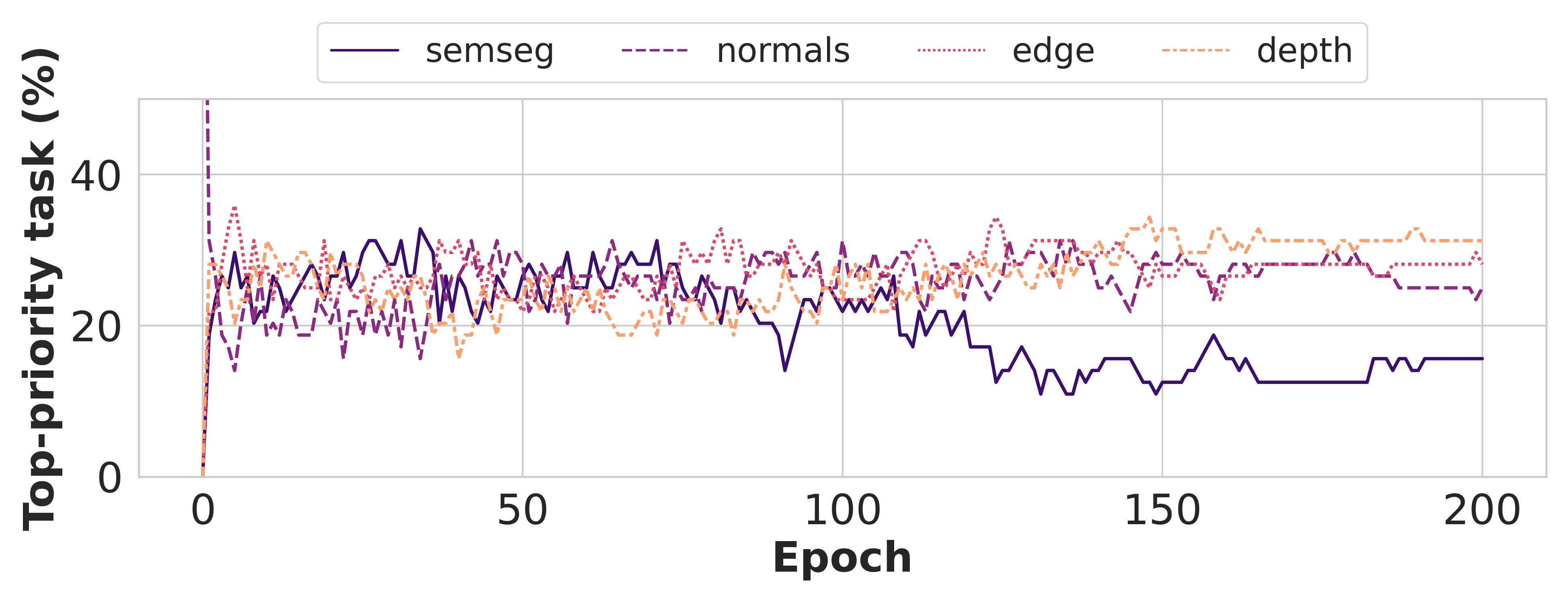}
    \caption{layer0-0-1}
\end{subfigure}
\begin{subfigure}{.45\textwidth}
    \includegraphics[width=.99\columnwidth]{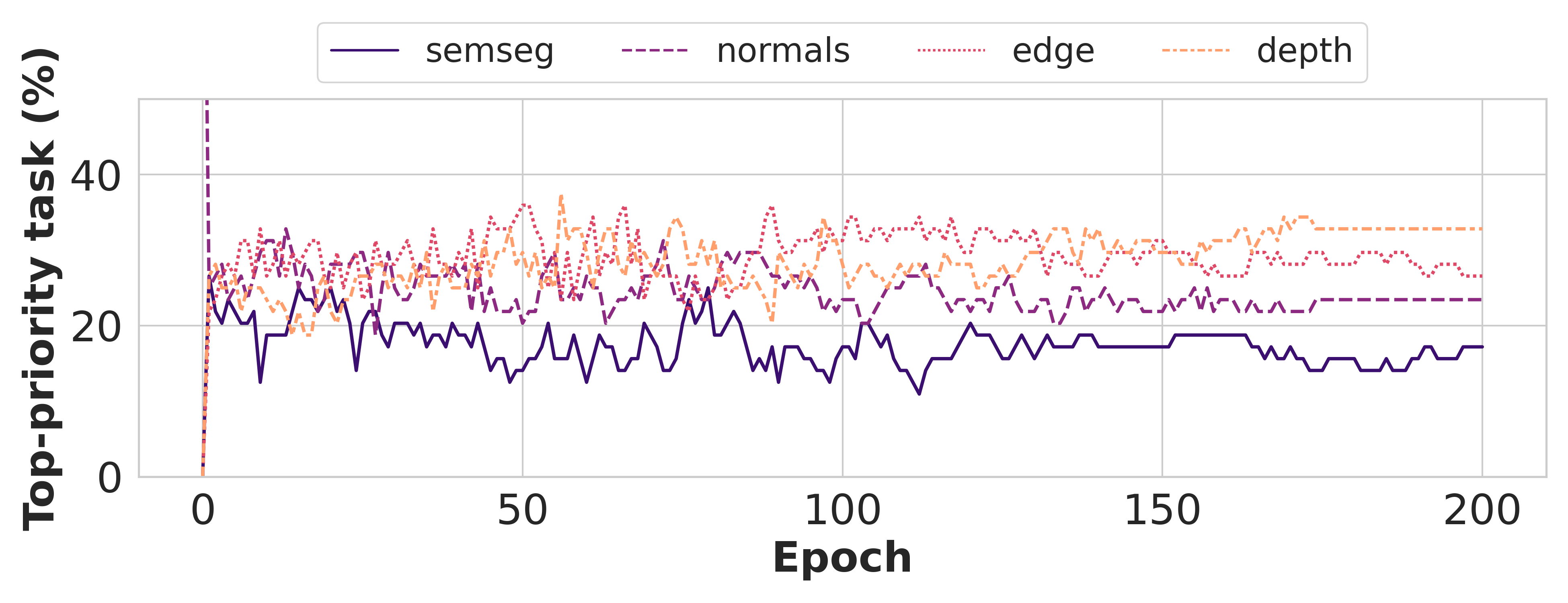}
    \caption{layer0-0-2}
\end{subfigure}

\begin{subfigure}{.45\textwidth}
    \includegraphics[width=.99\columnwidth]{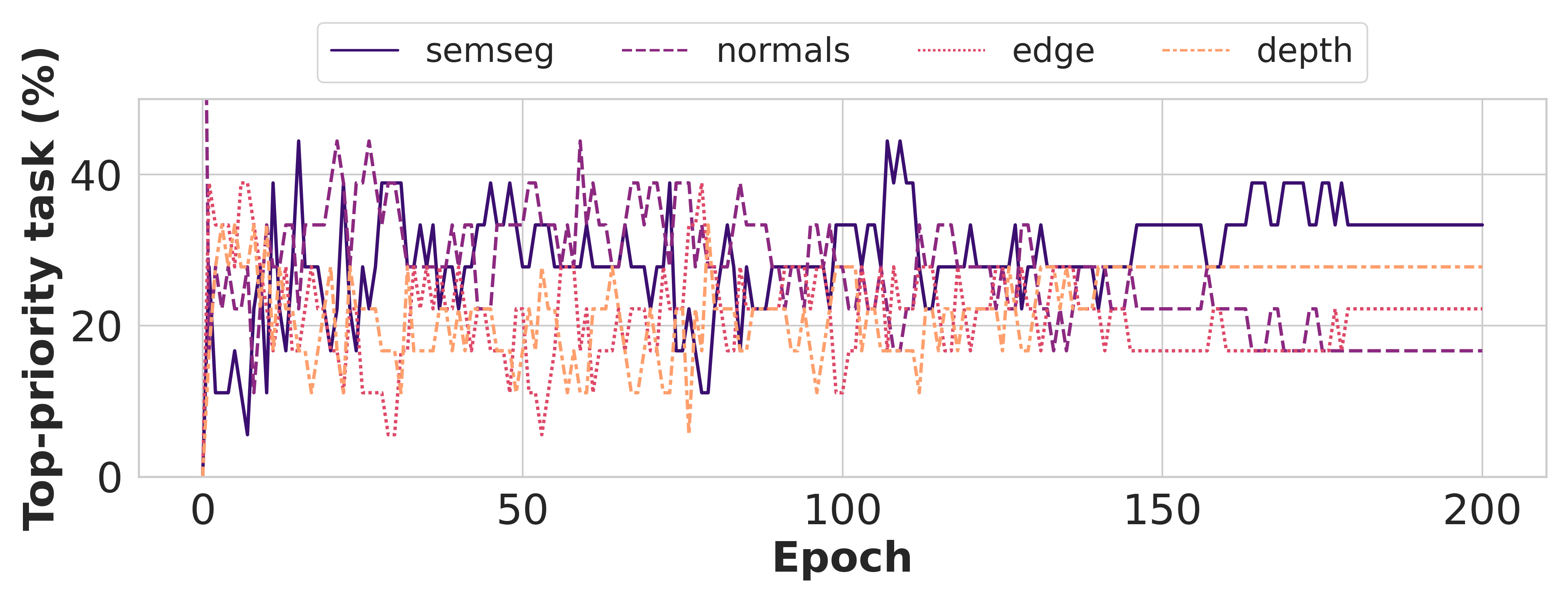}
    \caption{layer1-0-0-0}
\end{subfigure}
\begin{subfigure}{.45\textwidth}
    \includegraphics[width=.99\columnwidth]{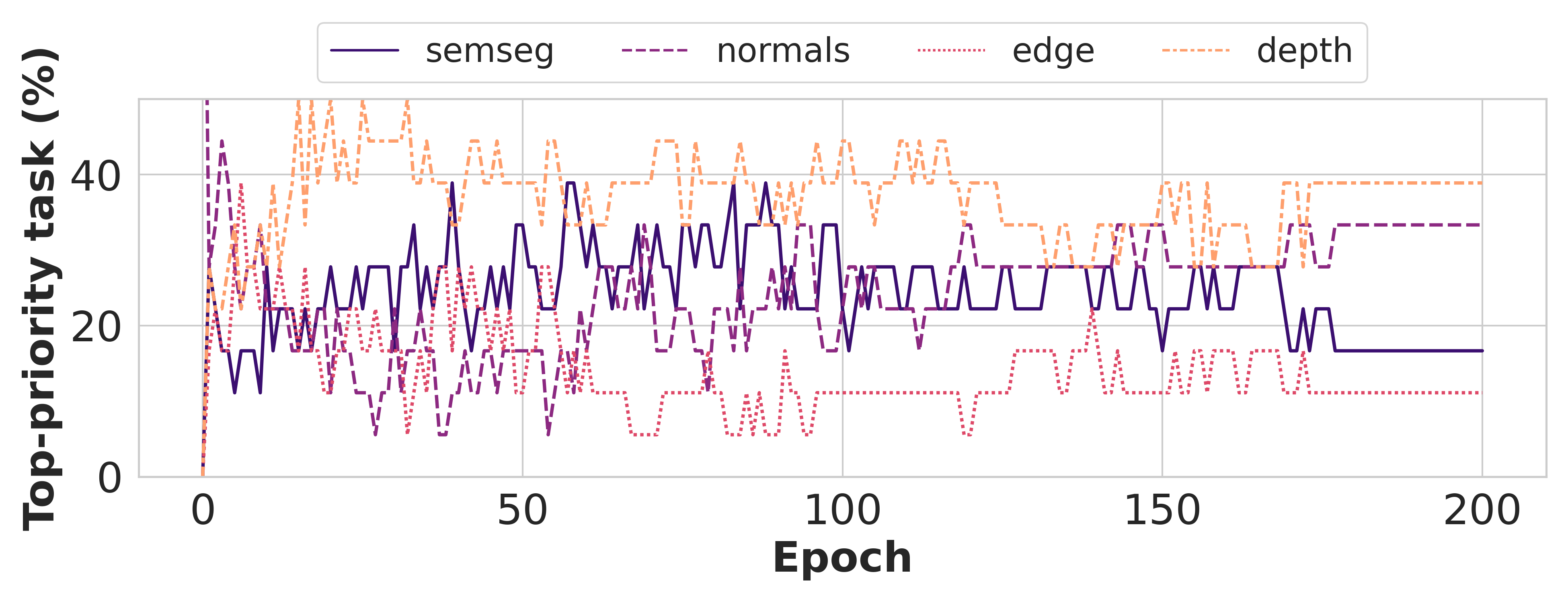}
    \caption{layer1-0-0-1}
\end{subfigure}

\begin{subfigure}{.45\textwidth}
    \includegraphics[width=.99\columnwidth]{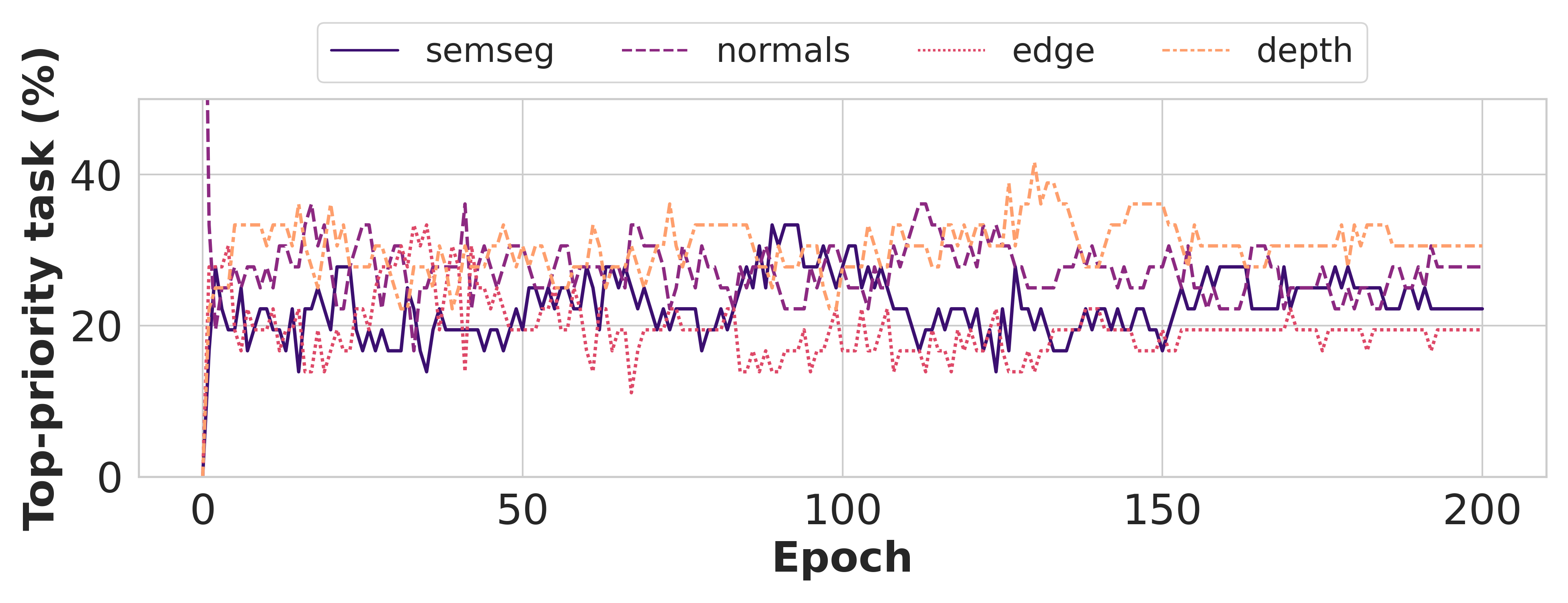}
    \caption{layer1-0-1-0}
\end{subfigure}
\begin{subfigure}{.45\textwidth}
    \includegraphics[width=.99\columnwidth]{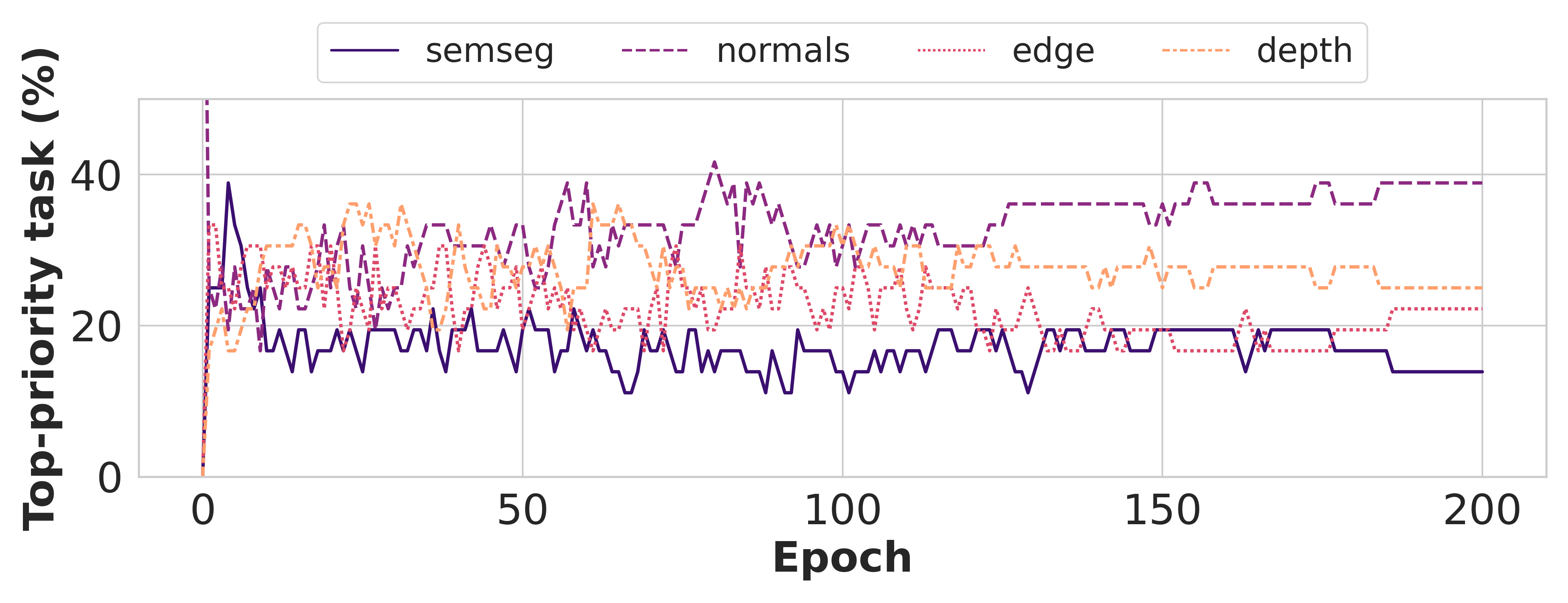}
    \caption{layer1-0-1-1}
\end{subfigure}

\begin{subfigure}{.45\textwidth}
    \includegraphics[width=.99\columnwidth]{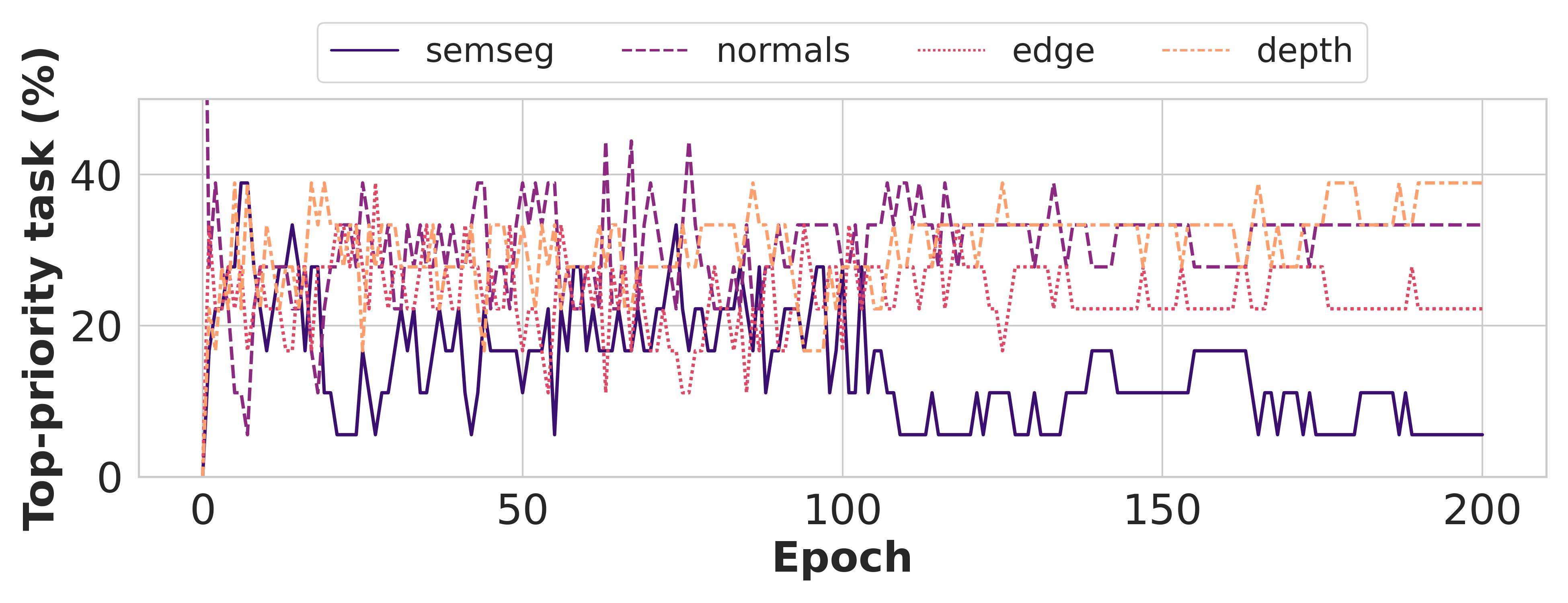}
    \caption{layer2-0-0-0}
\end{subfigure}
\begin{subfigure}{.45\textwidth}
    \includegraphics[width=.99\columnwidth]{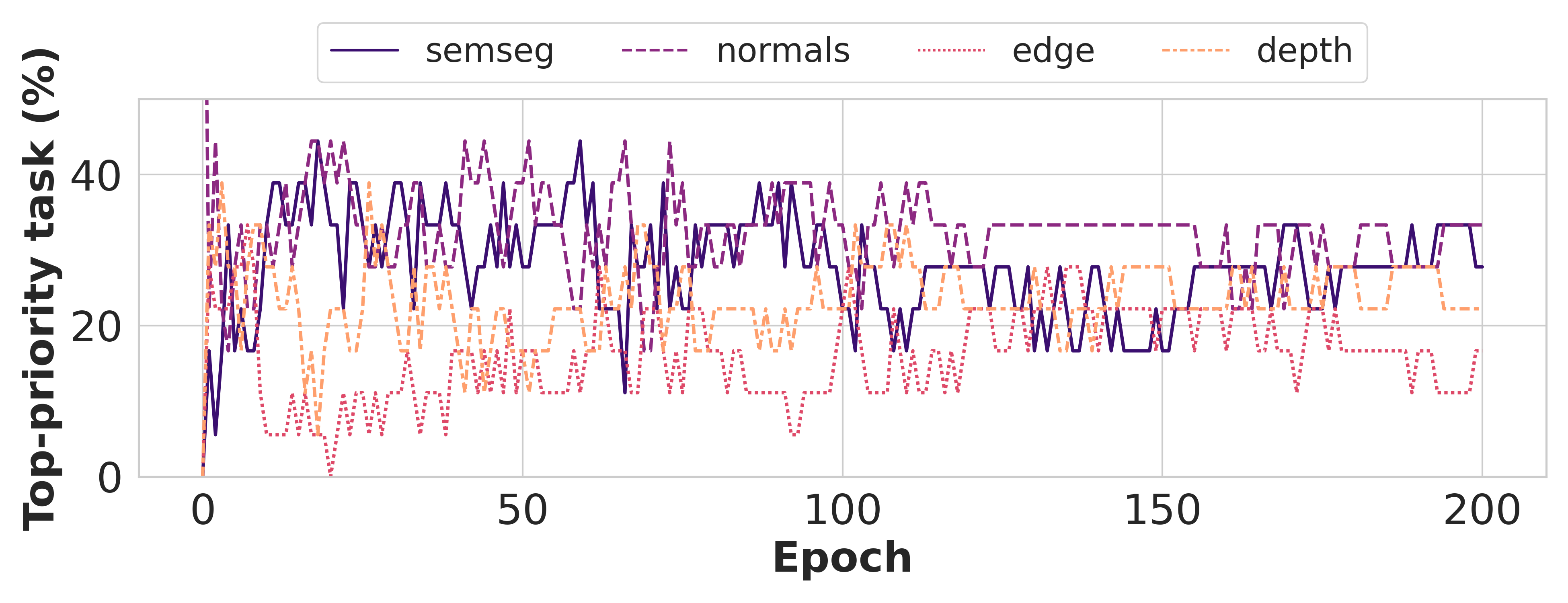}
    \caption{layer2-0-0-1}
\end{subfigure}

\begin{subfigure}{.45\textwidth}
    \includegraphics[width=.99\columnwidth]{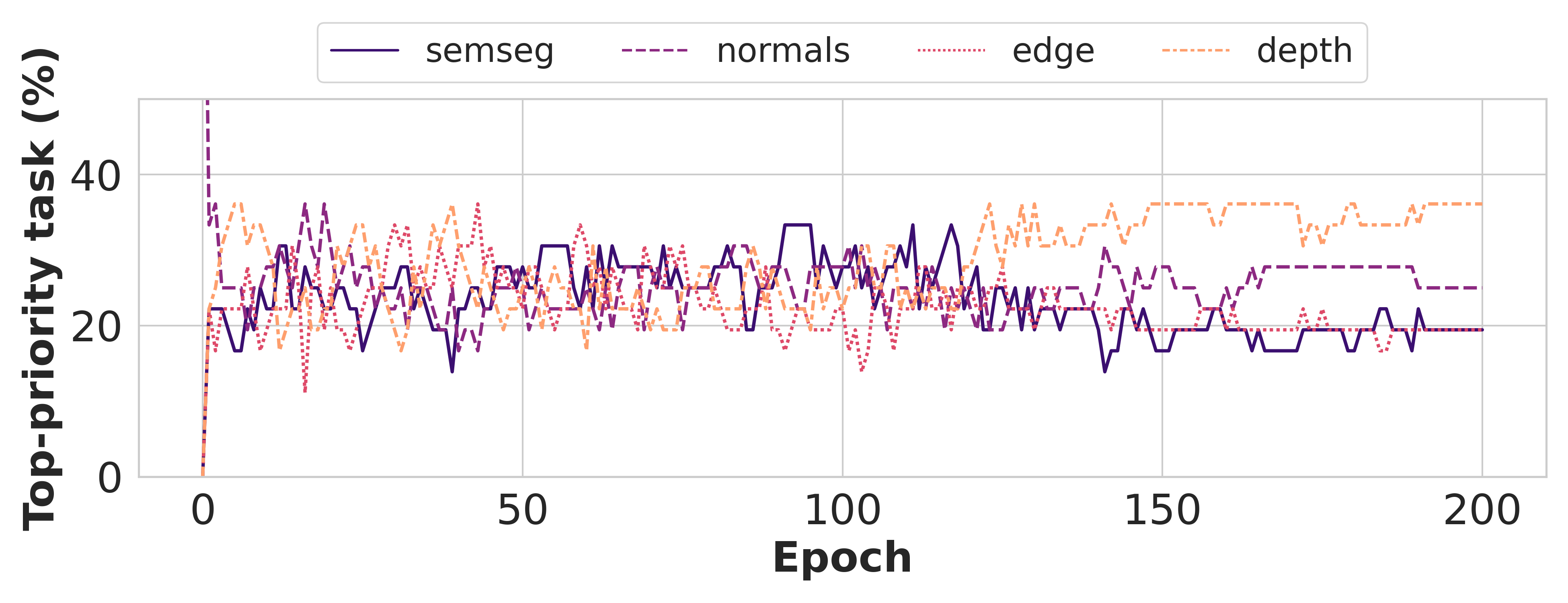}
    \caption{layer2-2-1-0}
\end{subfigure}
\begin{subfigure}{.45\textwidth}
    \includegraphics[width=.99\columnwidth]{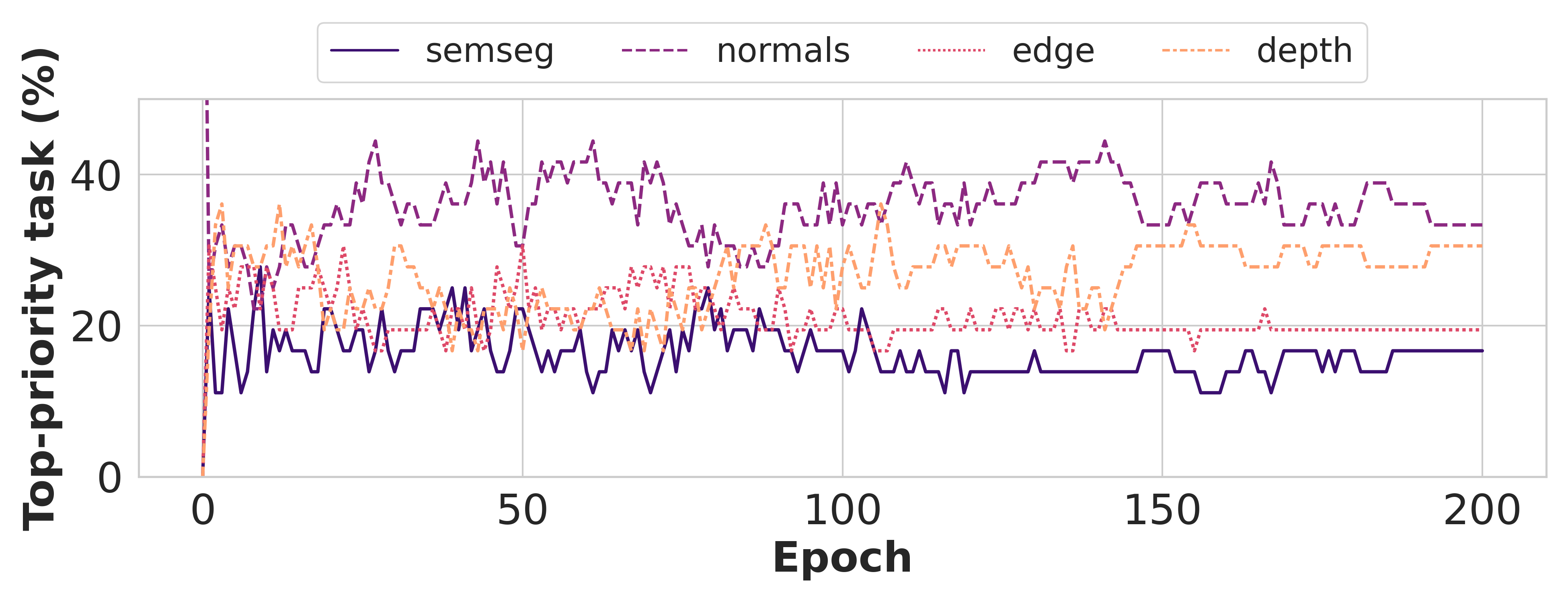}
    \caption{layer2-2-1-1}
\end{subfigure}

\begin{subfigure}{.45\textwidth}
    \includegraphics[width=.99\columnwidth]{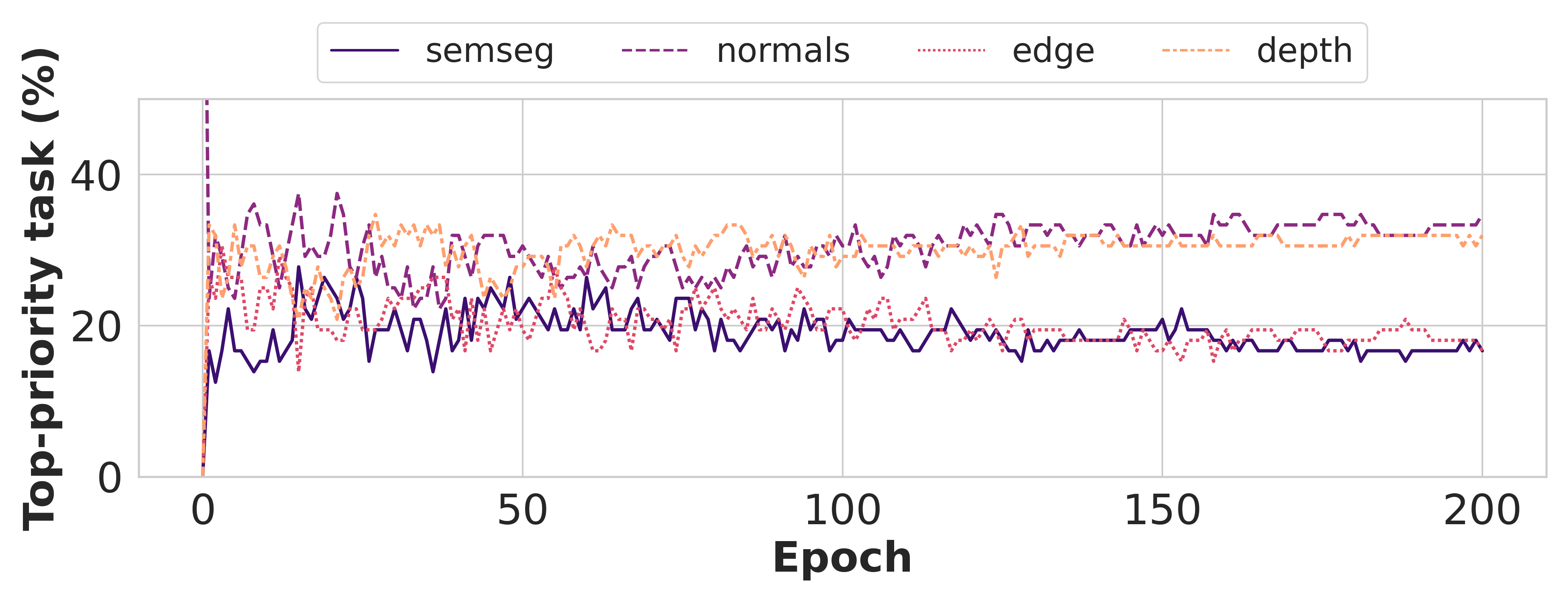}
    \caption{layer2-2-2-0}
\end{subfigure}
\begin{subfigure}{.45\textwidth}
    \includegraphics[width=.99\columnwidth]{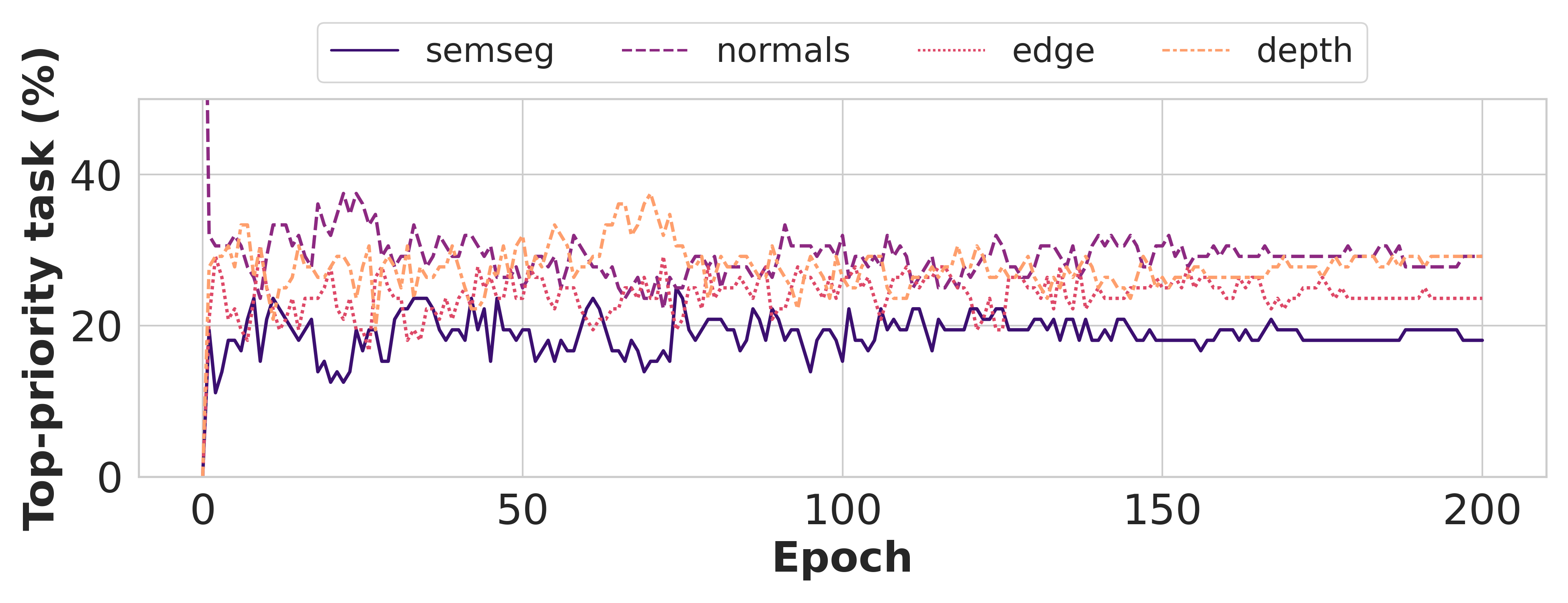}
    \caption{layer2-2-2-1}
\end{subfigure}

\caption{Visualization of the percentage of top-priority tasks over training epoch depending on the position in the network. We randomly selected several convolutional layers from the Network. The timing at which task priority stabilizes varies depending on the position of the convolutional layer. \label{fig:conn_pos}}
\end{figure}


\begin{figure}[h] \ContinuedFloat
\centering
\begin{subfigure}{.45\textwidth}
    \includegraphics[width=.99\columnwidth]{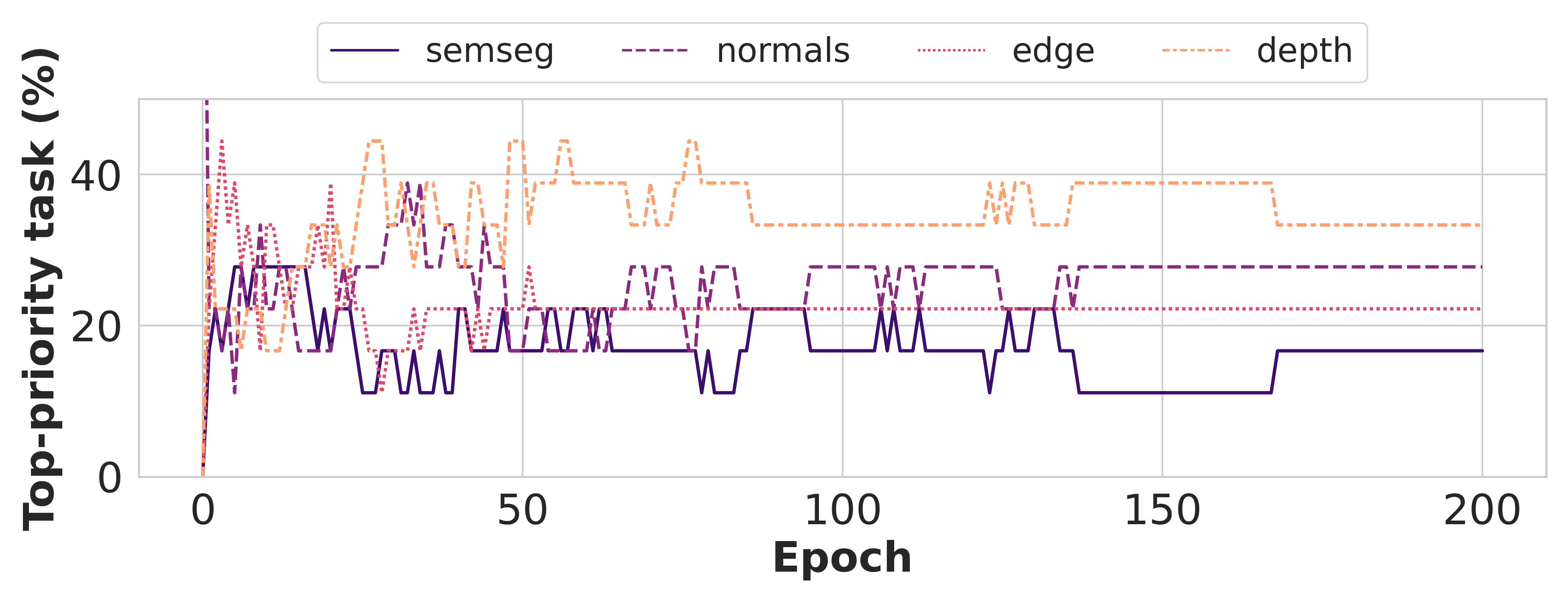}
    \caption{layer3-0-0-0}
\end{subfigure}
\begin{subfigure}{.45\textwidth}
    \includegraphics[width=.99\columnwidth]{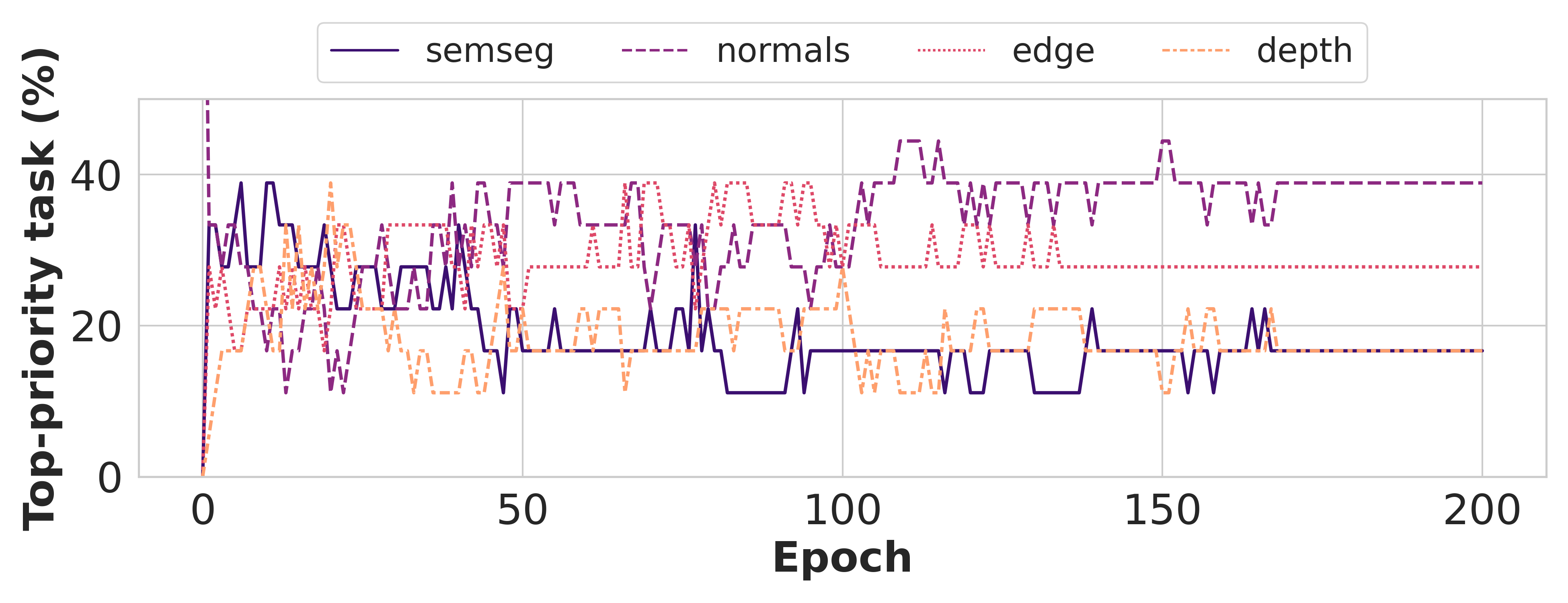}
    \caption{layer3-0-0-1}
\end{subfigure}

\begin{subfigure}{.45\textwidth}
    \includegraphics[width=.99\columnwidth]{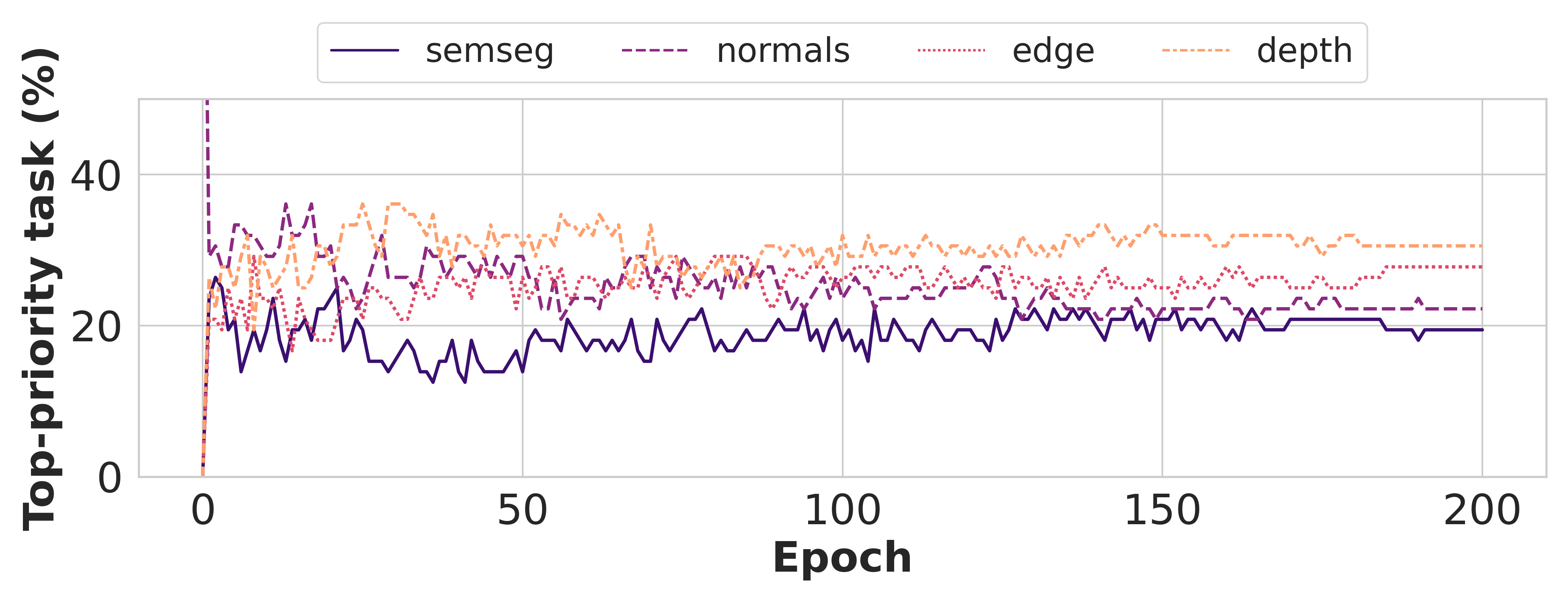}
    \caption{layer3-0-2-0}
\end{subfigure}
\begin{subfigure}{.45\textwidth}
    \includegraphics[width=.99\columnwidth]{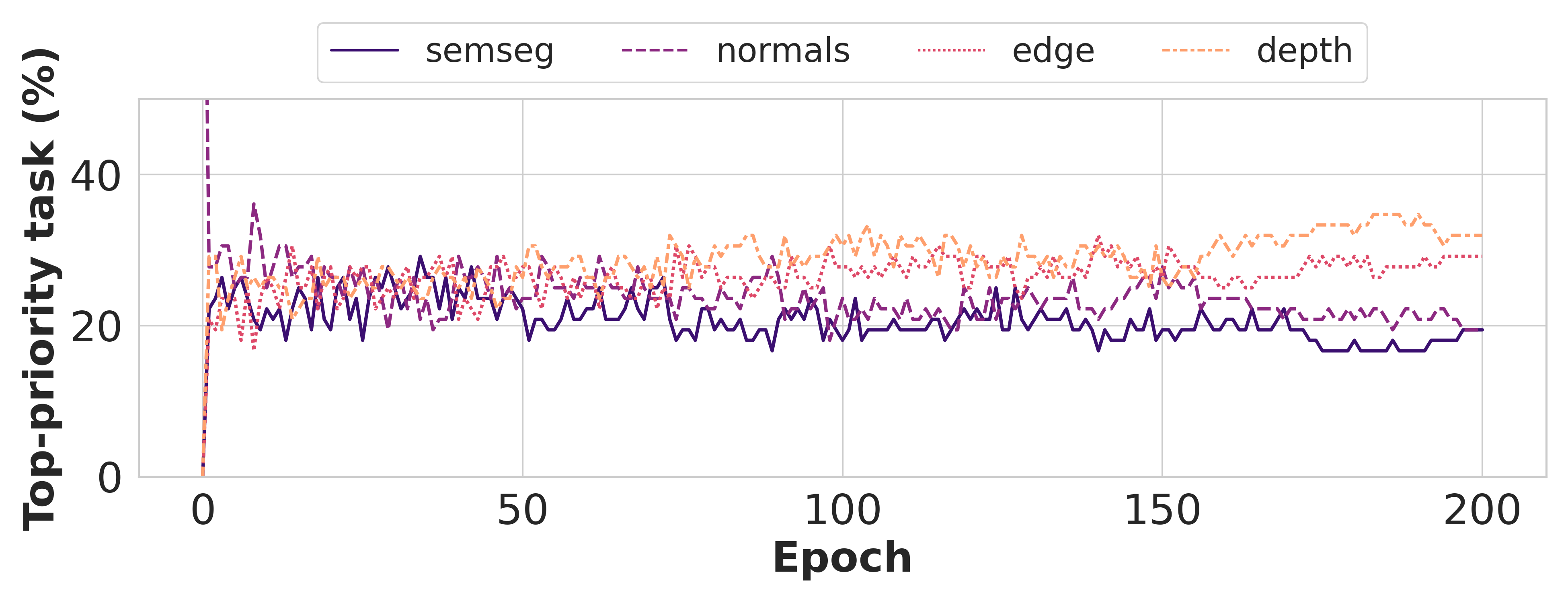}
    \caption{layer3-0-2-1}
\end{subfigure}

\begin{subfigure}{.45\textwidth}
    \includegraphics[width=.99\columnwidth]{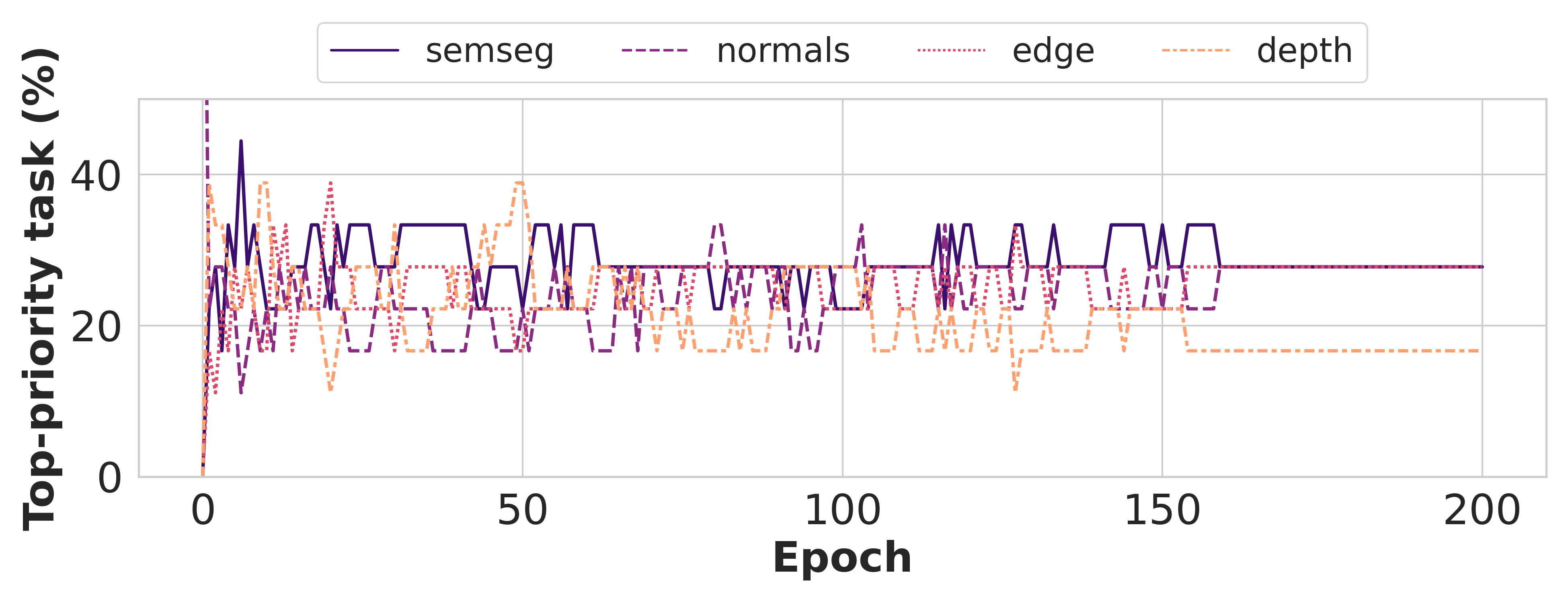}
    \caption{layer3-1-0-0}
\end{subfigure}
\begin{subfigure}{.45\textwidth}
    \includegraphics[width=.99\columnwidth]{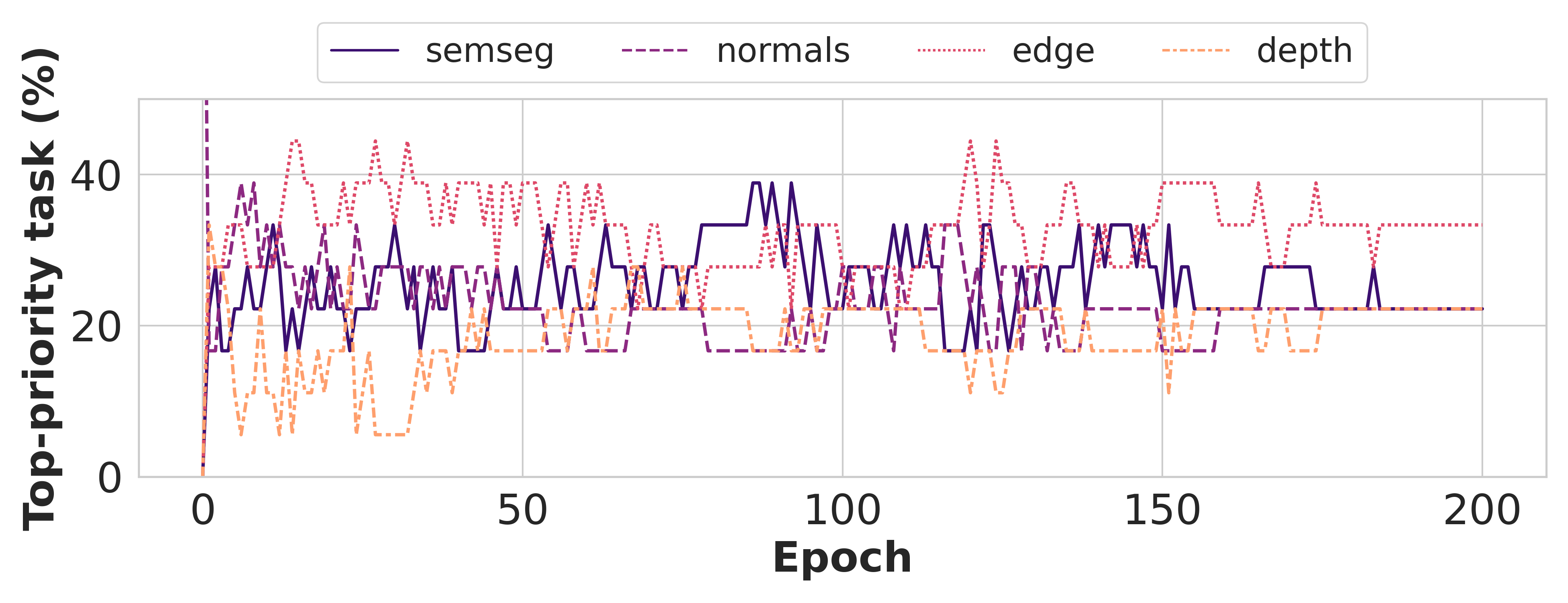}
    \caption{layer3-1-0-1}
\end{subfigure}

\begin{subfigure}{.45\textwidth}
    \includegraphics[width=.99\columnwidth]{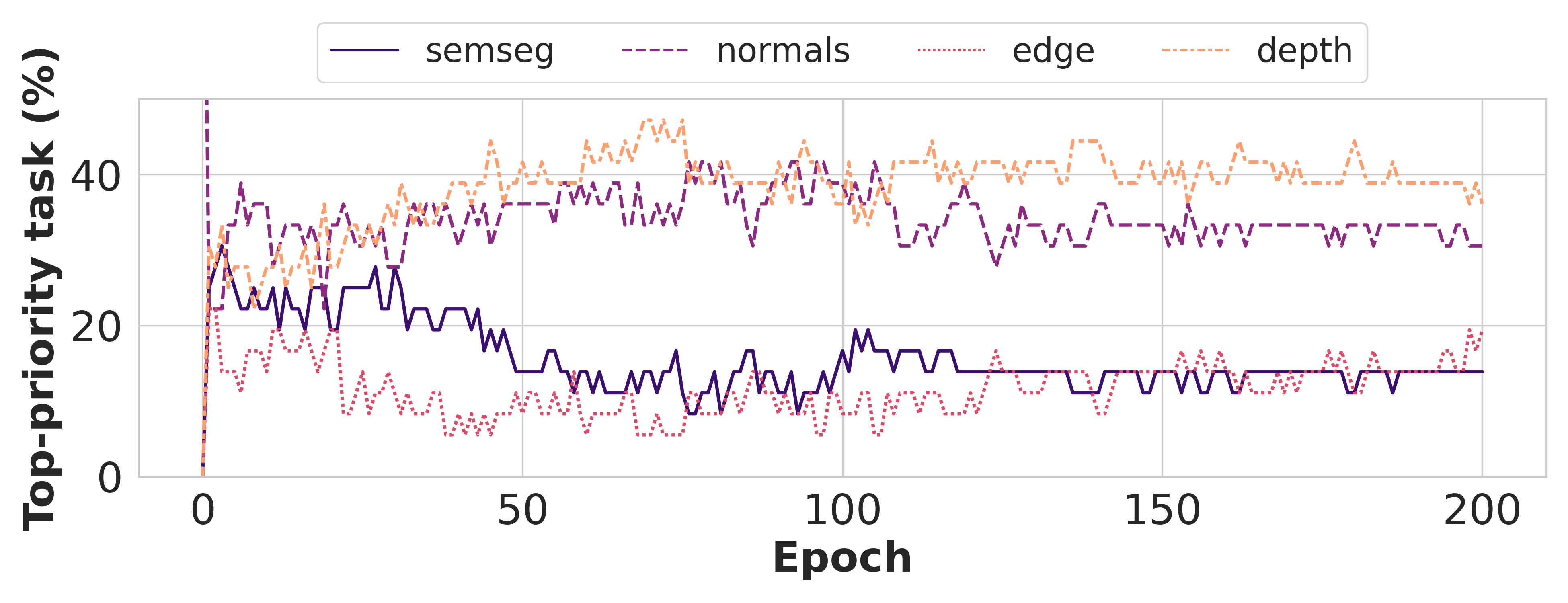}
    \caption{layer3-1-1-0}
\end{subfigure}
\begin{subfigure}{.45\textwidth}
    \includegraphics[width=.99\columnwidth]{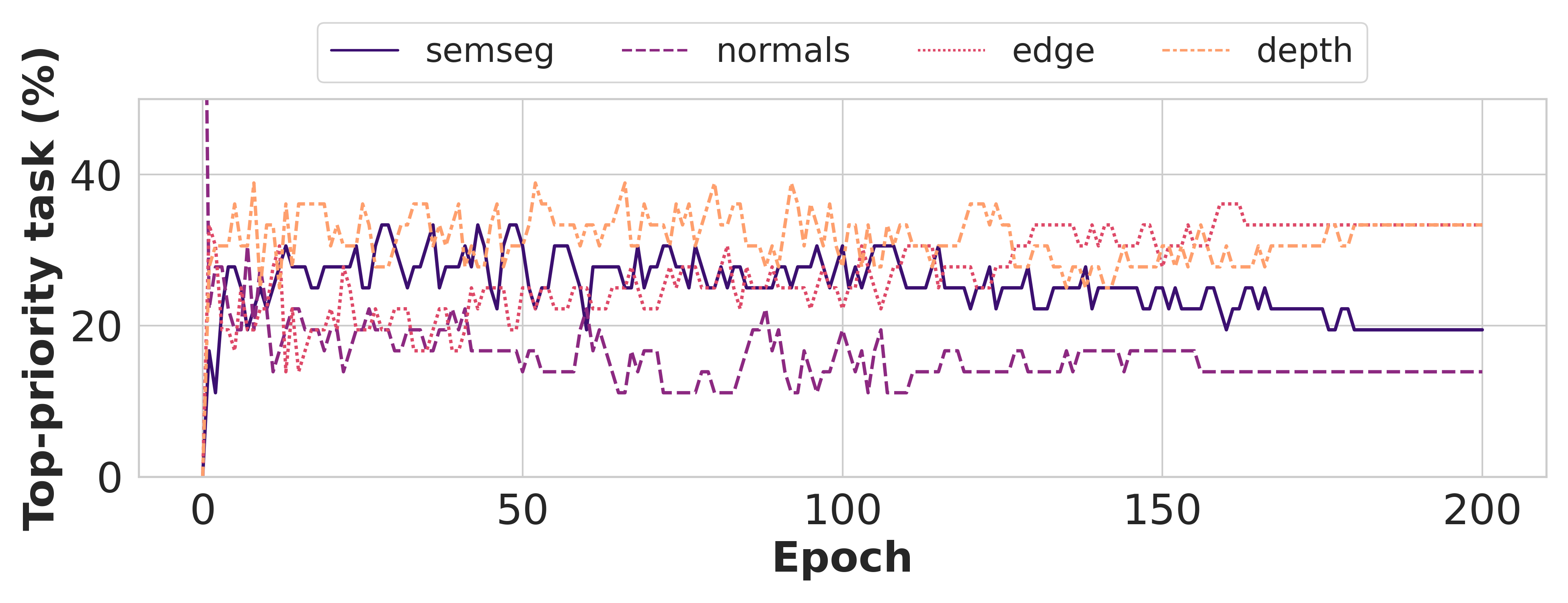}
    \caption{layer3-1-1-1}
\end{subfigure}

\begin{subfigure}{.45\textwidth}
    \includegraphics[width=.99\columnwidth]{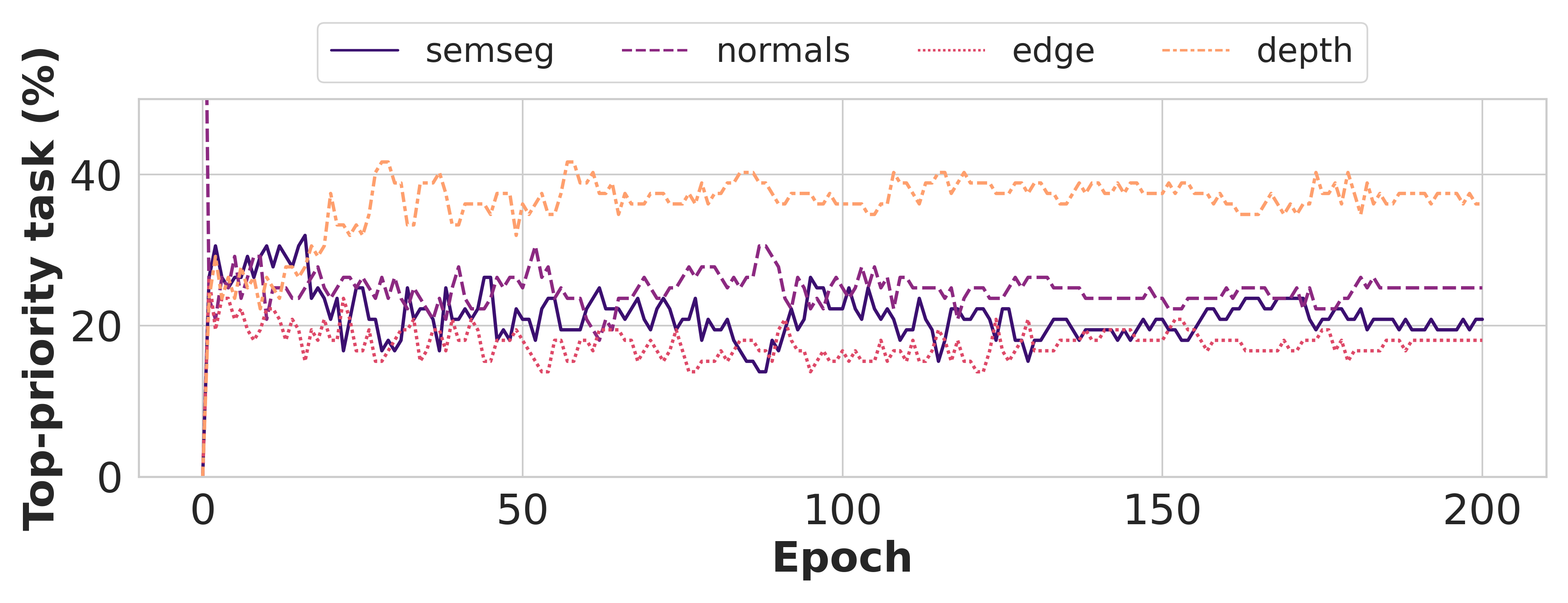}
    \caption{layer3-1-2-0}
\end{subfigure}
\begin{subfigure}{.45\textwidth}
    \includegraphics[width=.99\columnwidth]{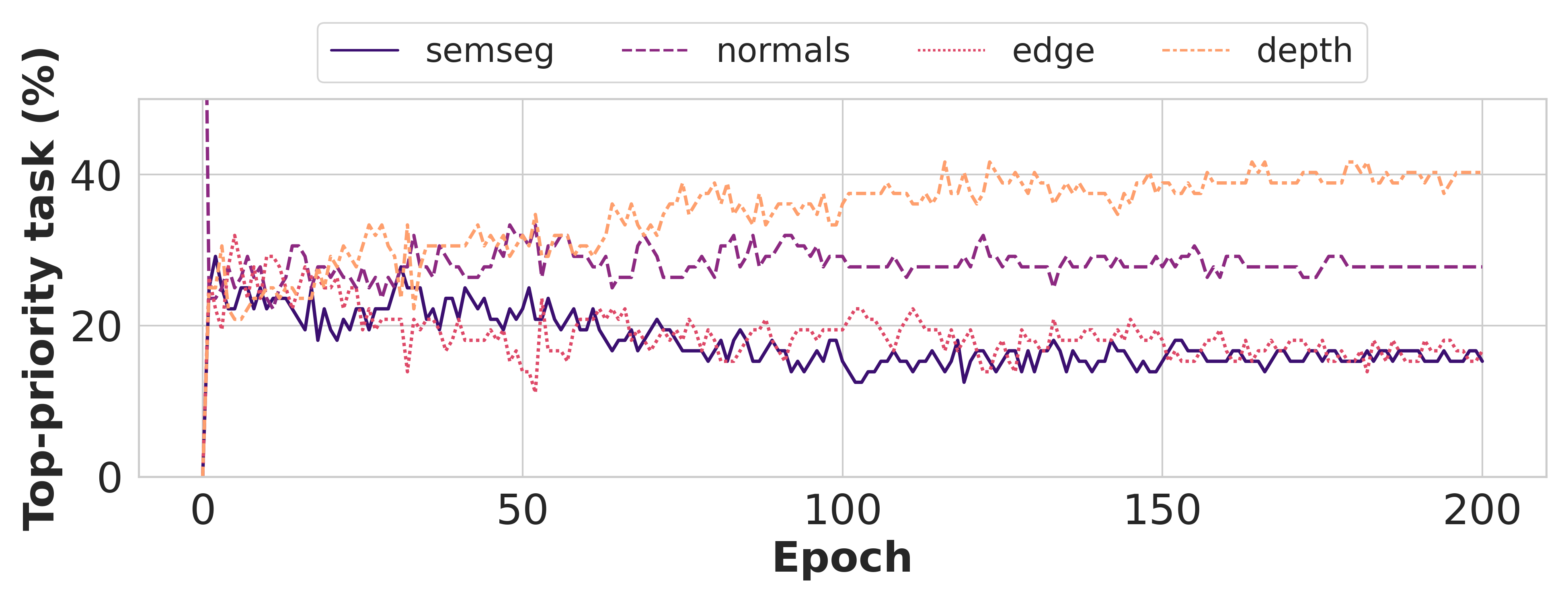}
    \caption{layer3-1-2-1}
\end{subfigure}

\begin{subfigure}{.45\textwidth}
    \includegraphics[width=.99\columnwidth]{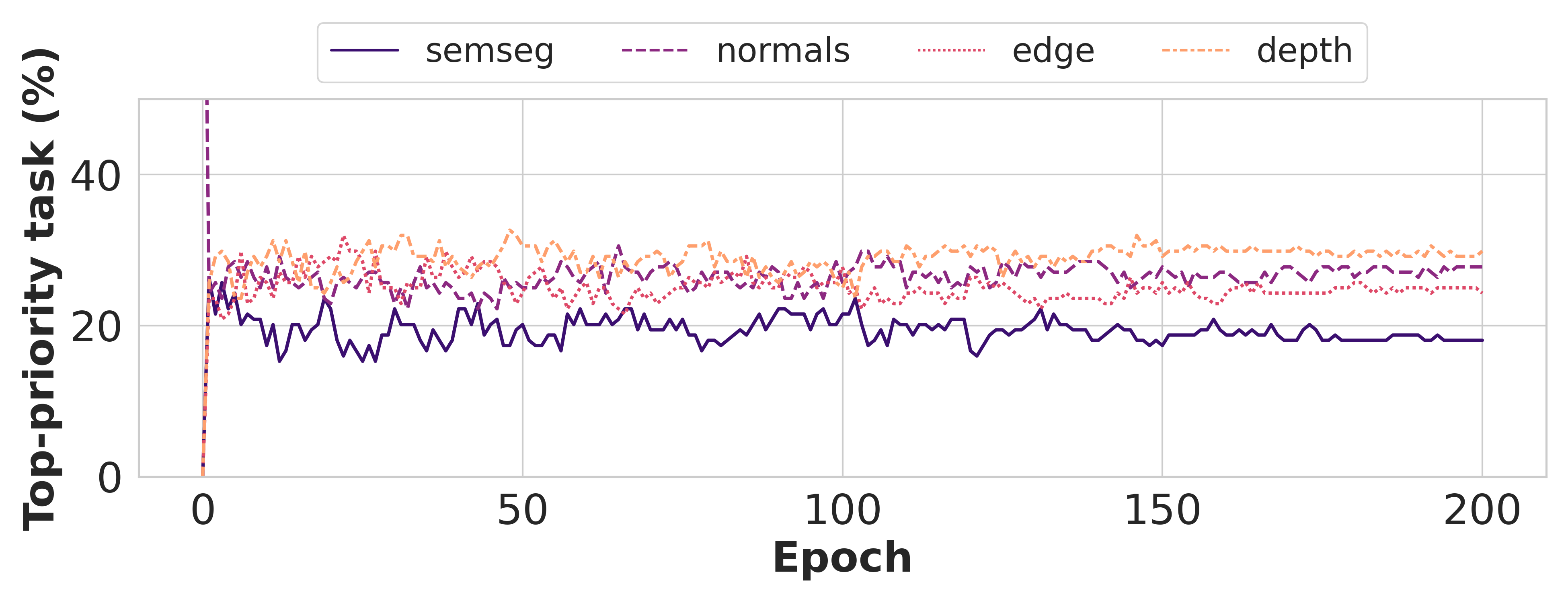}
    \caption{layer3-1-3-0}
\end{subfigure}
\begin{subfigure}{.45\textwidth}
    \includegraphics[width=.99\columnwidth]{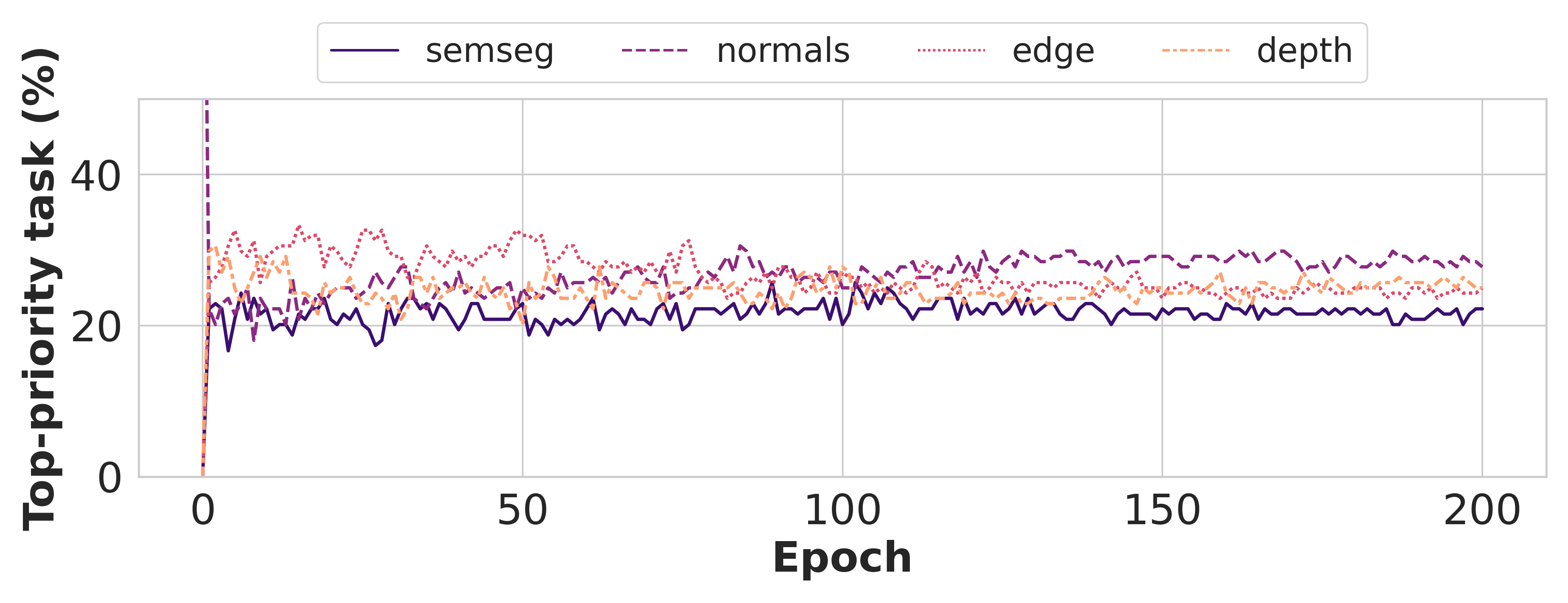}
    \caption{layer3-1-3-1}
\end{subfigure}

\caption{Visualization of the percentage of top-priority tasks over training epoch depending on the position in the network. We randomly selected several convolutional layers from the Network. The timing at which task priority stabilizes varies depending on the position of the convolutional layer.}
\end{figure}